\documentclass[12pt]{article}
\usepackage{amsmath}
\usepackage{graphicx}
\usepackage{enumerate}
\usepackage{natbib}
\usepackage{url} % not crucial - just used below for the URL 

\newcommand{\blind}{0}

\usepackage{amssymb}
\usepackage{amsfonts}
\usepackage{enumerate}
\usepackage{algorithm}
\usepackage{algorithmic}
\usepackage{lmodern}
\usepackage{setspace}
\usepackage{hyperref}
\usepackage{xcolor}
\usepackage{amsthm}
\usepackage{hyperref}
\DeclareUnicodeCharacter{2212}{\textminus}
\usepackage{booktabs}
\usepackage{enumitem}

\newtheorem{theorem}{Theorem}[section]

\newtheorem{lemma}[theorem]{Lemma}

\begin{document}

\def\spacingset#1{\renewcommand{\baselinestretch}%
{#1}\small\normalsize} \spacingset{1}

%%%%%%%%%%%%%%%%%%%%%%%%%%%%%%%%%%%%%%%%%%%%%%%%%%%%%%%%%%%%%%%%%%%%%%%%%%%%%%

\if1\blind
{
  \title{\bf HI}
  \author{Author 1\thanks{
    The authors gratefully acknowledge \textit{please remember to list all relevant funding sources in the unblinded version}}\hspace{.2cm}\\
    Department of YYY, University of XXX\\
    and \\
    Author 2 \\
    Department of ZZZ, University of WWW}
  \maketitle
} \fi

\if0\blind
{
  \bigskip
  \bigskip
  \bigskip
  \begin{center}
    {\Large\bf Bayesian Additive Regression Trees\\
    for functional ANOVA model}
\end{center}
  \medskip
} \fi
\begin{center}
Seokhun Park$^{1}$, Insung Kong$^{2}$ and Yongdai Kim$^{1,3}$\\
\end{center}
\begin{scriptsize}
\begin{center}
$^{1}$Department of Statistics, Seoul National University, \{shrdid,ydkim903\}@snu.ac.kr \\
$^{2}$Department of Applied Mathematics, University of Twente, insung.kong@utwente.nl \\
$^{3}$Corresponding author
\end{center}
\end{scriptsize}

\bigskip
\begin{abstract}
Bayesian Additive Regression Trees (BART) is a powerful statistical model that leverages the strengths of Bayesian inference and regression trees.
It has received significant attention for capturing complex non-linear relationships and interactions among predictors.
However, the accuracy of BART often comes at the cost of interpretability.
To address this limitation, we propose ANOVA Bayesian Additive Regression Trees (ANOVA-BART), a novel extension of BART based on the functional ANOVA decomposition, which is used to decompose the variability of a function into different interactions, each representing the contribution of a different set of covariates or factors.
Our proposed ANOVA-BART enhances interpretability, preserves and extends the theoretical guarantees of BART, and achieves comparable prediction performance.
Specifically, we establish that the posterior concentration rate of ANOVA-BART is nearly minimax optimal, and further provides the same convergence rates for each interaction that are not available for BART.
Moreover, comprehensive experiments confirm that ANOVA-BART is comparable to BART in both accuracy and uncertainty quantification, while also demonstrating its effectiveness in component selection.
These results suggest that ANOVA-BART offers a compelling alternative to BART by balancing predictive accuracy, interpretability, and theoretical consistency.
The official implementation of ANOVA-BART is publicly available at \href{https://github.com/ParkSeokhun/ANOVA-BART}{https://github.com/ParkSeokhun/ANOVA-BART}.

\end{abstract}

\noindent%

\spacingset{1.9} % DON'T change the spacing!

\section{Introduction}
\label{sec:intro}

Bayesian trees and their ensembles have demonstrated significant success in statistics and machine learning (\cite{BayesTree, Bayescart, BART, he2019xbart, lakshminarayanan2015particle, luo2021bast, luo2022bamdt}). 
In particular, Bayesian Additive Regression Trees (BART, \cite{BART}), which put
the prior mass on the space of ensembles of decision trees and obtain the posterior distribution,
have received much attention for their superior prediction performance in
various problems including causal inference (\cite{hill, Hahn}), variable selection (\cite{Lin}), survival analysis (\cite{spara}), interaction detection (\cite{due}), smooth function estimation (\cite{linero2018bayesian}), mean-variance function estimation (\cite{PM}), time series (\cite{taddy}), monotone function estimation (\cite{mbart}), to name just a few. 
In addition, theoretical properties of BART  have been actively studied (\cite{onBART, bforest, artbart, linero2018bayesian})

A limitation of BART, however, is that it is a black-box approach in the sense that the relation between covariates and response variable is hard to be explained. This is because a linear combination of decision trees is not easily interpretable
even though each decision tree is interpretable.
Recently, interpretability is an important issue in statistics, machine learning and Artificial Intelligence (AI), and constructing interpretable models without hampering prediction performance becomes a key challenge.

Various methods to improve interpretability can be roughly categorized into two approaches - (1) post-processing approach and (2) interpretable model approach.  The post-processing approach tries to interpret a given black-box model. Representative examples are partial dependency plots (\cite{friedman2001greedy}), LIME (\cite{ribeiro2016should}) and SHAP (\cite{lundberg2017unified}). In contrast, the interpretable model approach uses easily interpretable prediction models such as the linear model, generalized additive model (\cite{gam}) and more generally functional ANOVA model (\cite{func_first}). In particular, the functional ANOVA model has a long history in statistics and received much attention recently in machine learning and AI (\cite{func_first, func_cosso, func_diag, func_gaussian, func_hazard, func_proj, MARS, SSANOVA, func_hooker, func_puri, func_hyper, func_vae, park2025tensor}). 

The aim of this paper is to develop a Bayesian inference procedure for the functional ANOVA model, 
which we call ANOVA-BART. The core idea of ANOVA-BART, which is an interpretable model approach, is to approximate each interaction in the functional ANOVA model by a linear combination of decision trees.
A technical difficulty is that each component in the functional ANOVA model is not identifiable and 
thus direct adoption of BART to the functional ANOVA model would result
in unstable inference of each component.
Similar phenomena are observed in neural networks for the functional 
ANOVA model such as Neural Additive Models (NAM, \cite{nam}) and Neural Basis Models (NBM, \cite{nbm}).

To resolve this problem, we impose the identifiability condition (See Section \ref{sec:function_anova}
 for details) on each component. Then, we devise special decision trees such that
(1) they satisfy the identifiability condition, (2) a linear combination of them approximates a smooth function well
to guarantees a (nearly) minimax optimal posterior concentration rate, 
and (3) an efficient MCMC algorithm can be developed without much difficulty.

Special neural networks to ensure the identifiability of each component
in the functional ANOVA model have been developed by \cite{park2025tensor} and \cite{park2025bayesian}.
As usual, however, neural networks are highly vulnerable to outliers in covariates and thus 
special preprocessing of covariates are indispensable for stable inference.
In contrast, decision trees are robust to outliers and so yield
stabler and more accurate prediction models without requiring data preprocessing. 
  
This paper is organized as follows. 
Section \ref{sec:preliminary} reviews the functional ANOVA model and BART.  
ANOVA-BART is described in Section \ref{section3}, and an MCMC algorithm for the posterior sampling of ANOVA-BART is developed in Section \ref{sec:posterior}. 
Section \ref{5} provides the posterior concentration rate of ANOVA-BART which is nearly minimax optimal and adaptive to the interaction structure (i.e., the highest order of the signal components) as well as the smoothness of the true model. 
Section \ref{sec:experiment} demonstrates the competitiveness of ANOVA-BART against baseline models through analyses of simulation data and multiple benchmark datasets.
   
\section{Preliminaries}
\label{sec:preliminary}
We consider a standard nonparametric regression model given as
$$
Y|\mathbf{x} \sim \mathbb{P}_{f(\mathbf{x}),\psi}
$$ 
where $\mathbf{x}=(x_{1},...,x_{p})^{\top}\in \mathcal{X}\subset \mathbb{R}^p$ and $Y\in \mathbb{R}$ are covariate vector and response variable, respectively,
$f:\mathcal{X} \rightarrow \mathbb{R}$ is a regression function and $\psi \in \mathbb{R}$ is a nuisance parameter. 

We assume that the distribution $\mathbb{P}_{f(\mathbf{x}),\psi}$ belongs to the exponential family and admits a probability density function $p_{f(\mathbf{x}),\psi}$, defined as
\begin{align}
p_{f(\mathbf{x}),\psi}(y) = \exp\bigg( { f(\mathbf{x})y - A(f(\mathbf{x})) \over \psi} + S(y,\psi) \bigg), 
\label{eq:expontial_eta}
\end{align}
where $A(\cdot)$ is the log-partition function, ensuring that the density integrates to one.

In this section, we review the functional ANOVA model and BART since our proposed method is closely related to these two models.

\subsection{Notation}
For a real-valued function $f : \mathcal{X} \xrightarrow{}  \mathbb{R}$ and $1 \leq p < \infty$, we denote $\Vert f \Vert_{p,n}:= (\sum_{i=1}^{n}f(\textbf{x}_{i})^{p}/ n)^{1/ p},$ 
where $\mathbf{x}_{i}=(x_{i,1},...,x_{i,p})^{\top}, i=1,...,n$, are given covariate vectors. We write $\Vert f \Vert_{p,\mu} := (\int_{\mathbf{x} \in \mathcal{X}}f(\bold{x})^{p}\mu(d\mathbf{x}))^{1/ p},$ where $\mu$ is a probability measure defined on $\mathcal{X}.$
We define $[p]=\{1,...,p\}$ to represent the set of integers from $1$ to $p$.
For two sequences $a_{n}$ and $b_{n}$, we write 
$a_{n} = O(b_{n})$ if there exist constants $c > 0$ and $n_{0} \in \mathbb{N}$ 
such that $|a_{n}| \leq c |b_{n}|$ for all $n \geq n_{0}$.
For $A=\{a_1,\ldots,a_m\}$ with natural numbers $a_1<...<a_m$, we write $(x_j, j \in A) := (x_{a_1},...,x_{a_m})^{\top}$.
Let $\text{power}([p],d)$ be the collection of all subsets of $[p]$ whose cardinality is $d$.

\subsection{Functional ANOVA model}
\label{sec:function_anova}

We assume that $\mathcal{X}=\prod_{j=1}^p \mathcal{X}_j$, where $\mathcal{X}_j$ is a subset of $\mathbb{R}.$ 
For $\bold{x} \in \mathcal{X}$ and $S \subseteq [p]$, 
we write $\bold{x}_{S}:=(x_j, j\in S)$ and $\mathcal{X}_{S} := \prod_{j \in S}\mathcal{X}_{j}$. 
For a given function $f_{S}$ defined on $\mathcal{X}_{S}$ and a probability measure $\mu$ on $\mathcal{X}$, we say that $f_{S}$ satisfies the $\mu$-identifiability condition (\cite{func_hooker, func_diag}) if the following holds:
\begin{align}
\begin{split}
\forall  \: j \in S \mbox{ and }  \forall \: \mathbf{x}_{S\backslash\{j\}}\in \mathcal{X}_{S\backslash\{j\}},
\: \int_{\mathcal{X}_{j}}f_{S}(\bold{x}_{S})\mu_{j}(dx_j)=0,  \label{identfiable cond}
\end{split}
\end{align}
where $\mu_{j}$ is the marginal probability measure of $\mu$ on $\mathcal{X}_{j}.$ 

By the following theorem, any real-valued multivariate function $f$ can be uniquely decomposed into a functional ANOVA model whose components satisfy the $\mu$-identifiability condition.
Let $\mu^{\text{ind}}=\prod_{j=1}^{p}\mu_{j}$.

\begin{theorem}[\cite{func_diag, mcbook}]
\label{thm:fANOVA_decomp}
Any real-valued function $f$ defined on $\mathbb{R}^{p}$ can be uniquely decomposed as 
\begin{align}
f(\mathbf{x}) = \sum_{S \subseteq [p]}f_{S}(\mathbf{x}_{S}), \label{eq:f-ANOVA}
\end{align}
almost everywhere with respect to $\mu^{\text{ind}}$,
under the constraint that each interaction $f_{S}$ satisfies the $\mu$-identifiability condition.
\end{theorem} 

Here, the low dimensional functions $f_{S}$s are called the $|S|$th-order interactions or components of $f.$
Note that the functional ANOVA decomposition (\ref{eq:f-ANOVA}) of $f$ depends on the choice of the measures $\mu$ used in the identifiability condition (\ref{identfiable cond}).  
In this paper, we use the empirical distribution $\mu_{n}$ of given covariate vectors $\mathbf{x}_{1},\ldots,\mathbf{x}_{n}$ for $\mu$ and
we write simply `the identifiability condition' whenever $\mu=\mu_{n}$.
If $\mu$ is different from $\mu_{n}$, we will explicitly specify it.

The functional ANOVA model with only the main effects is the generalized additive model (GAM, \cite{gam}) which is widely used in practice. For machine learning applications,   \cite{func_hooker, func_puri, park2025tensor} emphasize
the usefulness of the functional ANOVA model as an interpretable AI tool. 
There are various algorithms to estimate the interactions in the functional ANOVA model including
MARS (\cite{MARS}), SSANOVA (\cite{SSANOVA}), COSSO (\cite{func_cosso}), NAM (\cite{nam}), ANOVA-TPNN (\cite{park2025tensor}) and so on.

\subsection{Bayesian Additive Regression Trees (BART)}
BART (\cite{BART}) assumes a Gaussian regression model, i.e., $Y|\mathbf{x} \sim N(\cdot|f(\mathbf{x}),\sigma^{2})$, where the regression function $f$ is modeled as the sum of random decision trees.
That is, BART assumes a priori that $f(\cdot)=\sum_{t=1}^{T} \mathbb{T}_t(\cdot),$ where $\mathbb{T}_t$s are random decision trees and obtains the posterior distribution of $\mathbb{T}_t$s.
Each decision tree is represented by the three parameters at each node - (i) the binary indicator whether a node is internal or terminal, (ii) the pair of (split-variable, split-criterion) if the node is internal and (iii) the height (i.e the prediction value) if the node is terminal.
BART puts a prior mass on these three parameters to generate independent random decision trees $\mathbb{T}_t$s
while $T$ is fixed in advance and not inferred.
Then, BART samples from the posterior distribution $p(\mathbb{T}_1,\ldots,\mathbb{T}_{T},\sigma^{2}|{\rm data})$ using the backfitting MCMC algorithm (\cite{BART}).

Building upon BART, several extensions have been proposed.
Specifically, \cite{linero2018dart} proposed Dirichlet Additive Regression Trees (DART), which equip BART with a Dirichlet prior on the splitting probabilities over covariates, thereby encouraging sparse and adaptive variable selection.
\cite{linero2018bayesian} proposed Soft BART (SBART) by replacing the indicator functions in BART decision trees with sigmoid functions, and proves that the posterior concentration rate of SBART is minimax optimal over Hölder function spaces with smoothness parameter $\alpha \in (0,\infty)$.
\cite{deshpande2020vcbart} introduced Varying Coefficient Bayesian Additive Regression Trees (VCBART), which use BART to estimate coefficient functions in varying coefficient models.
Recently, \cite{GBART} proposed a reversible jump Markov chain Monte Carlo (RJMCMC) algorithm for Generalized BART (GBART) when the data distribution belongs to the exponential family.

\section{ANOVA-BART : BART on the functional ANOVA model} \label{section3}

In this section, we propose a version of BART for the functional ANOVA model which we call ANOVA-BART. 
The main idea of ANOVA-BART is to approximate each component 
of the functional ANOVA model by the ensemble of random decision trees.
That is, we set
\begin{align*}
f_S(\mathbf{x}_S) = \sum_{t=1}^{T_S} \mathbb{T}_t^{S}(\mathbf{x}_S),
%\label{eq:each_com}
\end{align*}
where $\mathbb{T}_{t}^{S}$s are specially designed decision trees defined on $\mathcal{X}_{S}$.
Let $\mathcal{T}^S$ be the class of decision trees where $\mathbb{T}_t^S$s belong. 
In the following subsections, we propose $\mathcal{T}^S$ so that $f_{S}(\cdot)$ always satisfies the identifiability condition and devise priors on $\mathcal{T}^S$ for $S\subseteq [p]$ and the other parameters including the numbers of decision trees $T_{S}$ as well as the nuisance parameter $\psi$.

%(\ref{eq:f-ANOVA})
%using a sum of decision trees that satisfy the identifiability condition.
%That is, we set
%\begin{align}
%f_S(\mathbf{x}_S) = \sum_{t=1}^{T_S} \mathbb{T}_t^{S}(\mathbf{x}_S),
%\label{eq:each_com}
%\end{align}
%where $\mathbb{T}_{t}^{S}$s are specially designed decision trees defined on $\mathcal{X}_{S}$.
%Let $\mathcal{T}^S$ be the class of decision trees where $\mathbb{T}_t^S$s belong. 
%In the following subsections, we propose $\mathcal{T}^S$ so that $f_{S}(\cdot)$ always satisfies the identifiability condition and devise a prior on $\mathcal{T}^S$ for $S\subseteq [p]$ and
%the other parameters including the numbers of decision trees $T_S$ for each $S$ as well as the variance $\sigma^2$ of noises.
%Note that ANOVA-BART treats the number of decision trees as random and also infers it
%while BART fixes it in advance, 
%which makes ANOVA-BART differs theoretically as well as algorithmically from BART.
 
\subsection{Choice of $\mathcal{T}^S$: Identifiable binary-product trees} \label{3.1}

A technical difficulty in choosing $\mathcal{T}^S$ 
is that $f_{S}$ should satisfy the identifiability condition in (\ref{identfiable cond}).
To ensure this restriction,
we make each decision tree in $\mathcal{T}^S$ satisfy the identifiability condition.
Let $\mu_{n,S}$ be an empirical distribution of given $\mathbf{x}_{1,S},...,\mathbf{x}_{n,S}$, where $\mathbf{x}_{i,S} = (x_{i,j}, j \in S)$ for $i=1,...,n$.
When $S=\{\ell\}$, we write $\mu_{n,\ell}$ instead of $\mu_{n,\{\ell\}}$ for notational simplicity.
We let $\mathcal{T}^S$ consist of decision trees having the following form
\begin{align}
\mathbb{T}^{S}(\mathbf{x}_{S}) = \beta \prod_{j \in S}\big(\mathbb{I}(x_{j} \leq s_{j}) +  \mathfrak{a}_{j}\mathbb{I}(x_{j} > s_{j}) \big),
\end{align}
where $\mathfrak{a}_{j} = -\mu_{n,j}\{X_{j} \leq s_{j} \} / \mu_{n,j}\{X_{j} > s_{j} \}$.
Note that $\mathbb{T}^{S}$ satisfies the identifiability condition in (\ref{identfiable cond}).
We refer to such a decision tree as an {\it identifiable binary-product} tree.

\begin{figure}[t]
\centering
\includegraphics[scale=0.4]{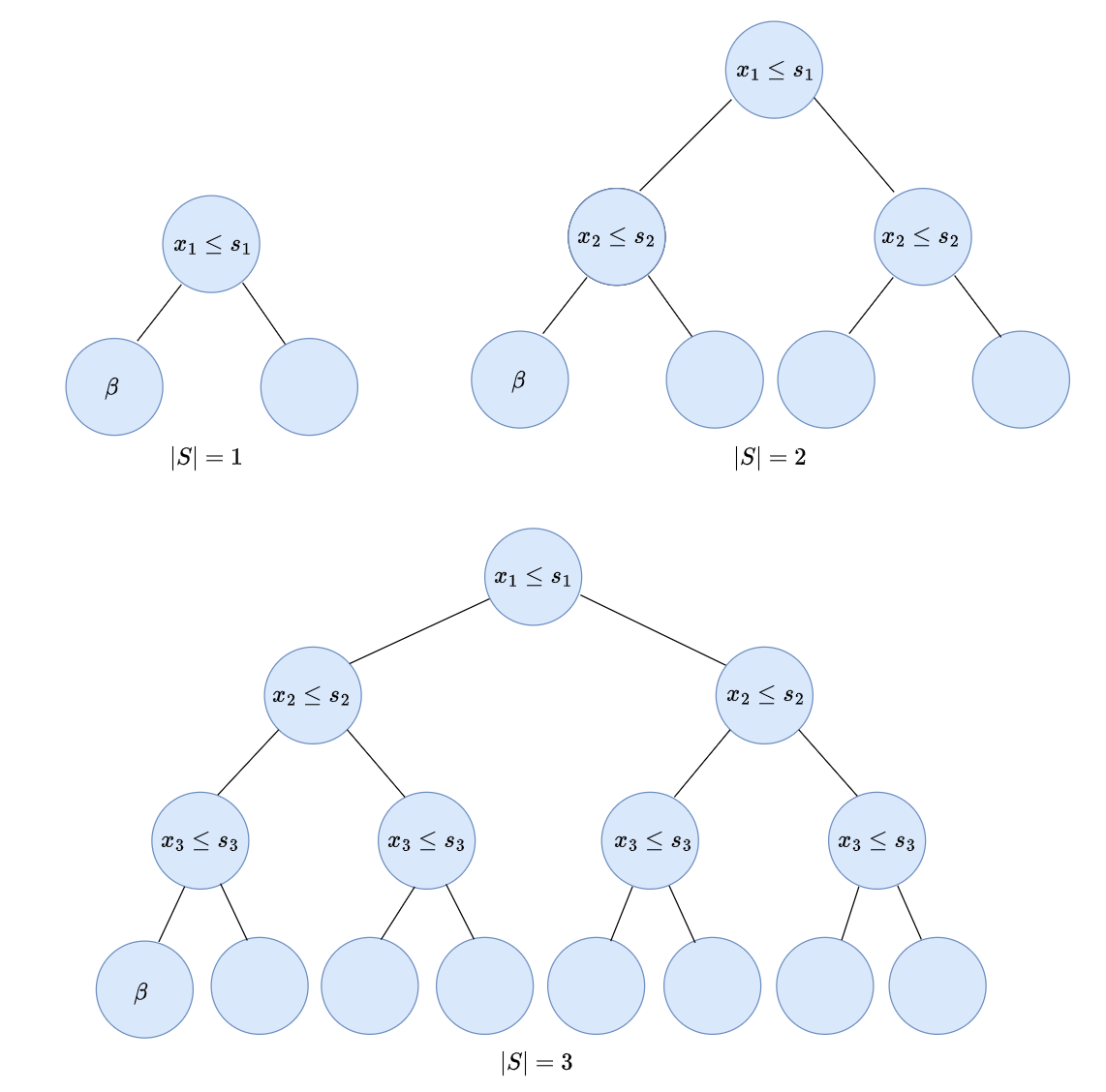}
\caption{Binary-product trees for $|S|$ being 1, 2, and 3, respectively.
Nodes at the same depth share the same split rule, and an observation is assigned to the left child node whenever the rule is satisfied (otherwise, it goes to the right child)} \label{tree fig}
\end{figure}

Figure \ref{tree fig} shows three indentifiable binary-product trees of the orders 1,2 and 3.
The decision trees have a special structure: all terminal (leaf) nodes have the same depth and all nodes at the same level share an identical splitting rule. 
The heights of each terminal node are automatically decided when the height $\beta$ of the left-most 
terminal node is selected. 
To sum up, any identifiable binary-product tree is parameterized by  
the componet set $S,$ the vector of the split values $\mathbf{s}:=(s_j, j\in S)$ 
corresponding to each covariate in $S$ and the height $\beta.$
Whenever needed, we write a given identifiable binary-product tree $\mathbb{T}^S(\mathbf{x}_S)$
as $\mathbb{T}(\mathbf{x}: S, \mathbf{s}, \beta)$.

For categorical covariate $x_{\ell},$ we proprocess all categorical covariates using the one-hot encoding. 
That is, each categorical covariate is transformed to a binary vector.
Therefore, splits on categorical covariates can be handled similarly to continuous covariates, and
we only consider continuous covariates for the remainder of the paper.

\subsection{Prior distribution} \label{prior sec}

ANOVA-BART assumes 
\begin{align} \label{eq_anova_bart}
    f(\mathbf{x})=\sum_{t=1}^T \mathbb{T}(\mathbf{x}:S_t, \mathbf{s}_t, \beta_{t})
\end{align}
for $S_t\subseteq [p],$ $\mathbf{s}_t \in \prod_{i\in S_{t}} \mathcal{A}_i$ and $\beta_t\in \mathbb{R}.$
% where $q_{\max}$ is the predefined maximum order of interactions.
The parameters to be inferred are $T$ as well as 
$(S_1,\mathbf{s}_1,\beta_{1}),\ldots,(S_T,\mathbf{s}_T,\beta_{T})$ and $\psi.$
In the following, we specify the prior distribution $T,$ 
for $(S_t,\mathbf{s}_t, \beta_t), t=1,\ldots,T$ and $\psi$.
First, we assume that a priori $(S_t,\mathbf{s}_t,\beta_{t}), t=1,2,\ldots$ are independent and identically distributed.
$\newline$
\noindent{\bf Prior for $S$.}
We let $S$ follow the mixture distribution $\sum_{d=1}^{p} \omega_d \cdot \mbox{Uniform} \left\{{\rm power}([p],d)\right\} $
for $\omega_d\ge 0$ and $\sum_{d=1}^{p} \omega_d=1$.
The weights $\omega_d$ are defined recursively as in \cite{BART}, as follows. 
For a given positive integer $d,$
we let $p_{\text{split}}(d)=\alpha_{\text{split}} (1+d)^{-\gamma_{\text{split}}}$
for $\alpha_{\text{split}}\in (0,1)$ and $\gamma_{\text{split}}>0.$ Then, we set
\begin{align}
\omega_d \propto (1-p_{\text{split}}(d)) \prod_{l < d} p_{\text{split}}(l). \label{split_prior} 
\end{align}
Basically, $\omega_d$ is decreasing in $d.$  The hyperparameters $\alpha_{\text{split}}$ and $\gamma_{\text{split}}$ 
control how fast $\omega_{d}$ decreases with $d$.
That is, the probability $\text{Pr}(|S|>d)$ increases  as $\alpha_{\text{split}}$ increases but decreases as $\gamma_{\text{split}}$ increases.
$\newline$
$\newline$
\noindent{\bf Prior for $\mathbf{s} | S$.} 
Conditional on $S$, $\{s_j, j\in S\}$ are independent and uniformly distributed on $\mathcal{A}_j, j\in S.$
In this paper, we set 
\begin{align}
\mathcal{A}_j=\{ (x_{(i),j}+x_{(i-1),j})/2: i=2,\ldots,n\}, 
\label{eq:split_set}
\end{align}
where $x_{(i),j}$s are the order statistics of given data $\{x_{i,j}, i=1,\ldots,n\}.$
$\newline$
$\newline$
\noindent{\bf Prior for $\beta$.}
We use a diffuse Gaussian prior
\begin{align*}
\beta \sim N(0,\sigma_{\beta}^{2})
\end{align*}
for $\sigma_{\beta}^2>0.$  
$\newline$
$\newline$
\noindent{\bf Prior for $T$.}
We use the following distribution for the prior of $T:$ 
\begin{align*}
&\pi\{T=t\}\propto e^{-C_{*} t \log n}, \:\: \text{for}\:\: t=0,1,\ldots, T_{\max},
\end{align*}
where $C_*>0$ and $T_{\max} \in \mathbb{N}_+$ are hyperparameters.
$\newline$
$\newline$
\noindent{\bf Prior for the nuisance parameter $\psi$.}
For a model having the nuisance parameter, we only consider the Gaussian regression model: 
$p_{f(\mathbf{x}),\psi}(y) \propto \exp(-(y-f(\mathbf{x}))^2/2\sigma^2),$
where $\psi=\sigma^2,$ and
we let
\begin{align*}
\sigma^{2} \sim IG\left({v\over2},{v\lambda \over 2}\right)   
\end{align*}
for $v>0$ and $\lambda>0$ a priori,
where $IG(a,b)$ is the inverse gamma distribution with the shape parameter $a$ and scale parameter $b.$

\section{Posterior Sampling}
\label{sec:posterior}

In this section, we develop an MCMC algorithm for posterior sampling of ANOVA-BART by modifying
the MCMC algorithms of BART (\cite{BART}, \cite{bartmachine}, \cite{linero2018bayesian}).
Note that, in contrast to Generalized BART (\citep{GBART}), the MCMC algorithm for ANOVA-BART does not require reversible jump MCMC when updating the tree structure, since the dimension of the height parameter is fixed at one and remains unchanged across different tree structures.

Basically, we use a Gibbs sampling algorithm to generate the parameters $T,
\{(S_{t},\mathbf{s}_t,\beta_{t}), t\in [T]\}$ and nuisance parameter $\psi$ from their conditional posterior distributions.
That is, our MCMC algorithm proceeds with the following three steps:
\begin{enumerate}
    \item Update $T$.
    \item Update $(S_{t}, \mathbf{s}_{t}, \beta_{t})$ for $t = 1, \ldots, T.$ 
    \item Update the nuisance parameter $\psi$.
\end{enumerate}
Let $\Theta = (T, S_{1}, \mathbf{s}_{1}, \beta_{1}, \ldots, S_{T}, \mathbf{s}_{T}, \beta_{T}, \psi)$. 
Define $P_{\pmb{\mathfrak{w}}}$ as a multinomial distribution with one trial on $[p]$ and probability vector 
$\pmb{\mathfrak{w}}=(\mathfrak{w}_1,\ldots,\mathfrak{w}_p)^\top$, where 
$\mathfrak{w}_j \propto \sum_{t=1}^T \mathbb{I}(j \in S_t)$ for $j \in [p]$.

%\textcolor{red}{
%We introduce a modified form of the MCMC algorithm in Bayesian-TPNN, specifically designed to more effectively explore higher-order components in high-dimensional datasets.
%}
%\textcolor{red}{
%Consider a probability mass function $P_{\pmb{\Xi}}(\cdot)$ based on $\pmb{\Xi} := (\Xi_{1}, \ldots, \Xi_{p})^{\top} \in \mathbb{R}_{+}^{p}$ for $[p]$.
%It represents the importance of each covariate.
%}

\subsection{Updating $T$}

We employ the following proposal distribution in the MH algorithm. 
Given $T$, we propose $T^{\text{new}} = T-1$ with probability $T/T_{\max}$, and $T^{\text{new}} = T+1$ with probability $1 - T/T_{\max}$.

\paragraph{Case of $T^{\text{new}} = T-1$.}
We obtain $\Theta^{\text{new}}$ by deleting one element randomly from $\{(S_{t}, \mathbf{s}_{t}, \beta_{t}), t=1, \ldots, T\}.$ 

\paragraph{Case of $T^{\text{new}} = T+1$.}
A central step is to construct a suitable proposal for the new tree, i.e, $(S_{T+1}^{\text{new}}, \mathbf{s}_{T+1}^{\text{new}}, \beta_{T+1}^{\text{new}})$.
For this purpose, we first propose $S_{T+1}^{\text{new}}$, and then 
$\mathbf{s}_{T+1}^{\text{new}}$ and $\beta_{T+1}^{\text{new}}$ are drawn from their prior distributions given 
$S_{T+1}^{\text{new}}.$
We devise a proposal distribution of $S_{T+1}^{\text{new}}$ carefully to ensure that
the MCMC algorithm can explore higher order interactions efficiently.

Let $M$ be a given positive integer.
First of all, the proposal distribution for $S_{T+1}^{\text{new}}$ consists of the following two alternatives:
\begin{itemize}
    \item Random : With probability $M/(M+T),$ generate $S_{T+1}^{\text{new}}$ from the prior distribution,
   \item Stepwise: With probability $T/(M+T),$ set $S_{T+1}^{\text{new}}=S_{t^*}\cup \{j_{t^*}\},$
   where $t^*\sim \text{Uniform}[T]$ and $j_{t^*} \sim P_{\pmb{\mathfrak{w}}}(\cdot|\cdot\in S_{t^{*}}^c).$ 
\end{itemize}

% The Stepwise move, which makes the order of a newly added tree be always higher by 1 than one of the existing trees in the current ensemble, is employed to encourage the MCMC algorithm to explore higher order interactions more frequently.
The stepwise move is designed to ensure that the order of a newly added tree is always one higher than that of an existing tree in the current ensemble, thereby encouraging the MCMC algorithm to explore higher-order interactions more frequently.
In addition, we use $P_{\pmb{\mathfrak{w}}}$ to increase the acceptance rate.
% As illustrated in the numerical studies, the Stepwise move enhances the ability of detecting higher order interactions significantly.

\subsection{Updating $(S_{t},\mathbf{s}_{t},\beta_{t})$ for $t=1,...,T$}
\label{sec:tree_sampling}
Let $\{(\mathbf{x}_1,y_1),\ldots,(\mathbf{x}_n,y_n)\}$
be given data which consist of $n$ pairs of observed covariate vector and response variable.
For the likelihood, we assume that
$y_i$s are realizations of $Y_i|\mathbf{x}_{i} \sim \mathbb{P}_{f(\mathbf{x}_{i}),\psi}$ conditional on $f$ and $\psi$.
Let $\mathbf{y}^{(n)} := (y_{1},...,y_{n})$.
\paragraph{Updating $(S_{t},\mathbf{s}_{t})$.}
The conditional posterior distribution $\pi\{S_{t},\mathbf{s}_{t}|\text{others}\}$ depends on 
solely through $\lambda_{t}, \mathbf{y}^{(n)}, \beta_{t}$ and $\psi$, where 
$$\lambda_{t}=\left\{ \sum_{k\neq t} \mathbb{T}(\bold{x}_i:S_k,\mathbf{s}_{k},\beta_{k}),\: i=1,\ldots,n\right\}$$
is the set of partial predictive values derived from the current model that does not include the $t$th decision tree. 
% Therefore, the sampling $(S_t,\mathbf{s}_{t})$ from the conditional posterior distribution is equal to sampling from the distribution $\pi\{S_t,\mathbf{s}_{t}|\lambda_{t},\mathbf{y}^{(n)},\beta_{t},\psi\}.$ 
Therefore, sampling $(S_t, \mathbf{s}_t)$ from the conditional posterior distribution is equivalent to sampling from the distribution $\pi\{S_t, \mathbf{s}_t \mid \lambda_t, \mathbf{y}^{(n)}, \beta_t, \psi\}$.

For updating $(S_t,\mathbf{s}_t),$ we use a modification of the MH algorithm used in BART (\cite{BART}).
We use a proposal distribution similar to one used in \cite{bartmachine, Lin}. 
A key difference
is that we consider only identifiable binary-product trees defined in Section \ref{3.1}. 
As a proposal distribution in the MH algorithm, we consider the following three possible alterations
of $(S_t,\mathbf{s}_t)$:
\begin{itemize}
\item \textbf{GROW} : adding an element $j^{\text{new}}$ selected randomly based on $P_{\pmb{\mathfrak{w}}}(\cdot|\cdot\in S_{t}^{c})$ into $S_t$  
and choosing a split value for the newly selected element $j^{\text{new}}$ by selecting it randomly from $\mathcal{A}_{j^{\text{new}}}.$
\item \textbf{PRUNE} : deleting an element from $S_t$ and the corresponding split value from $\mathbf{s}_t.$
\item \textbf{CHANGE} : changing an element in $S_t$ by one randomly selected from $S_t^c$ based on $P_{\pmb{\mathfrak{w}}}(\cdot|\cdot\in S_{t}^{c})$ and changing the split value accordingly. 
\end{itemize}
The MH algorithm proposes $(S_t^{\text{new}},\mathbf{s}_t^{\text{new}})$ using one of GROW, PRUNE, and CHANGE with probability 0.28, 0.28 and 0.44, respectively (\cite{bartmachine}), and then accepts/rejects $(S_t^{\text{new}},\mathbf{s}_t^{\text{new}})$
 according to the acceptance probability given in Section \ref{B.1} of Supplementary Material.

\paragraph{Updating $\beta_{t}$.}
We update $\beta_{t}$ via the MH algorithm with the Langevin dynamics (\cite{rossky1978brownian}).
The details are given in Section \ref{sec:beta_sample} of Supplementary Material.

\subsection{Updating nuisance parameter $\psi$}
Recall that $\psi=\sigma^2,$ where $\sigma^2$ is the variance of the noise in the Gaussian regression model.
Generating $\sigma^2$ can be done easily since the conditional posterior distribution of $\sigma^2$ is an inverse Gamma distribution whose details are given in Section \ref{sec:sigma_post} of Supplementary Material.

\subsection{MCMC algorithm for ANOVA-BART}
Algorithm \ref{alg:algorithm} summarizes the proposed MCMC algorithm for ANOVA-BART.
Note that the sampling of $T$ is not needed for BART since the number of decision trees $T$ is fixed.

\begin{algorithm}[h]
\caption{MCMC algorithm for ANOVA-BART} 
\label{alg:algorithm}
\hskip 0.7cm \textbf{Input}: $T$ : initial number of decision trees, $\mathbb{M}$ : number of MCMC iterations\\
\begin{algorithmic}[1] %[1] enables line numbers
\FOR{i : 1 to $\mathbb{M}$}
\STATE $T \sim \text{MH}_{T}(\{S_{t},\mathbf{s}_{t},\beta_{t}\}_{t=1}^{T},\psi)$
\FOR{t : 1 to $T$}
\STATE $(S_{t},\mathbf{s}_{t}) \sim \text{MH}_{S_{t},\mathbf{s}_{t}}(\lambda_{t},\mathbf{y}^{(n)},\beta_{t},\psi)$
\STATE $\beta_{t} \sim \text{MH}_{\beta_{t}}(\lambda_{t},\mathbf{y}^{(n)},S_{t},\psi)$
\ENDFOR \\
\STATE $\psi \sim \pi\{\psi|\text{others}\}$
% \STATE Update weight vector $\pmb{\Xi}$
\RETURN $\{\mathbb{T}(\cdot:S_{t},\mathbf{s}_{t},\beta_{t})\}_{t=1}^{T},\psi$
\ENDFOR
\end{algorithmic}
\end{algorithm}

% \textcolor{red}{
% \begin{remark}
% In BART, the likelihood of a tree structure is obtained by integrating over the height parameters. However, when extending to generalized regression, this integration becomes intractable.
% Consequently, in generalized BART, both the tree structure and the heights are generated jointly. Since the dimension of the heights changes with the tree structure, Reversible Jump MCMC (RJMCMC) is employed.
% However, in ANOVA-BART, the dimension of the heights to be estimated do not change with the tree structure, and thus RJMCMC is not required, making it straightforward to apply to generalized regression.
% \end{remark}
% }

\section{Posterior concentration rate} \label{5}

In this section, we study theoretical properties of ANOVA-BART. In particular, we drive the posterior concentration rate of ANOVA-BART when the true model is smooth (i.e., H\"{o}lder smooth).
Our posterior concentration rate is
the same as the minimax optimal rate (up to a logarithm term) and adaptive to the smoothness and the maximum order of signal interactions for the true function.
Moreover, we derive the posterior concentration rate of each component, 
which can be utilized to screen out unnecessary components after posterior computation.

For technical simplicity, we consider the random-$X$ design and
fix the nuisance parameter $\psi$ as 1. 
That is, the input-output pairs $(\bold{x}_1,y_1),\ldots,
(\bold{x}_n,y_n)$ are assumed to be a realization of $(\mathbf{X}_1,Y_1),\ldots, (\mathbf{X}_n,Y_n)$ that
are independent copies of $(\bold{X},Y)$ whose distribution $\mathbb{P}_0$ is given as
\begin{align*}
\bold{X}\sim \mathbb{P}_{\bold{X}}\quad \text{and} \quad Y|\bold{X}=\bold{x} \sim \mathbb{P}_{f_0(\bold{x}),1},  
\end{align*}
where $f_0$ is the true regression function.
Let $\mathbf{X}^{(n)} = (\mathbf{X}_{1},...,\mathbf{X}_{n})^{\top}$ and $Y^{(n)}=(Y_{1},...,Y_{n})^{\top}$.
The main reason of considering the random-$X$ design rather than the fixed-$X$ design is to apply the populational identifiability condition
($\mathbb{P}_{\mathbf{X}}$-identifiability condition) to the functional ANOVA decomposition of $f_0.$ 
Discussions about the fixed-$X$ design are given in Section \ref{sec:fixed-x} of Supplementary Material.

To derive the posterior concentration rate of ANOVA-BART, as usual, we will check the sufficient conditions given by \cite{nonp1}. 
The most technically difficult part is to study the approximation property
of the sum of identifiable binary-product trees with bounded heights.  
Theorem \ref{lemma5} in Section \ref{app:A_{n}} of Supplementary Material is the main approximation theorem for this purpose.

For technical reasons, we consider the truncated prior
$\pi_\xi$ for a large positive constant $\xi$ defined as
$\pi_\xi\{\cdot\}\propto \pi\{\cdot\} \mathbb{I}(\|f\|_\infty \le \xi),$
where $\pi\{\cdot\}$ is the prior introduced in Section \ref{prior sec}. 
For posterior sampling with this prior,
we generate samples of $f(\cdot)$ by the MCMC algorithm developed in Section \ref{sec:posterior}
and only accept samples satisfying $\|f\|_\infty \le \xi$.
We denote $\pi_\xi\{\cdot|\mathbf{X}^{(n)},Y^{(n)}\}$ the corresponding posterior.

\subsection{Posterior Concentration Rate}

For technical simplicity, we let $\mathcal{X}_j=[0,1]$ for all $j\in [p].$
We consider the populationally identifiable sparse ANOVA decomposition:
\begin{align*}
f_{0}(\mathbf{x}) = \sum_{S \in \mathbb{S} }f_{0,S}(\mathbf{x}_{S}),
\end{align*}
where $f_{0,S}, S\in \mathbb{S} $ satisfy the populational identifiability condition ($\mathbb{P}_{\mathbf{X}}$-identifiability condition) and $\mathbb{S} \subseteq \text{power}([p],d_{\max})$ is the index set of signal components. Let $d_{\max}=\max_{S\in\mathbb{S}}|S|$ that is the maximum order of signal components in the true function.
We assume the following regularity conditions.
\begin{enumerate}[label=(J.\arabic*)]
 \item The density $p_{\mathbf{X}}$ of
 $\mathbb{P}_{\bold{X}}$ with respect to the Lebesgue measure on $\mathbb{R}^p$ exists and satisfies
 $0< \inf_{\bold{x}\in \mathcal{X}}p_{\mathbf{X}}(\bold{x}) \le \sup_{\bold{x}\in \mathcal{X}}p_{\mathbf{X}}(\bold{x})<\infty. $
 \label{eq:Assumption_1}
 \item Each $f_{0,S}$ is a H\"{o}lder smooth function with smoothness $\alpha \in (0,1]$, i.e.,
 \begin{align*}
\Vert f_{0,S} \Vert_{\mathcal{H}^{\alpha}} := \sup_{\mathbf{x},\mathbf{x}' \in [0,1]^{|S|}}{|f_{0,S}(\mathbf{x}) - f_{0,S}(\mathbf{x}') | \over \Vert \mathbf{x} - \mathbf{x}' \Vert_{2}^{\alpha} 
} < \infty.
 \end{align*}
Additionally, we assume that $\Vert f_{0,S} \Vert_{\infty} \leq F$ for some positive constant $F$. 
For these assumptions, we write $f_{0,S}\in \mathcal{H}^{\alpha}_{F}.$
Furthermore, we write $f_0 \in \mathcal{H}^{\alpha}_{0,F}$ if $f_{0,S} \in \mathcal{H}^{\alpha}_{F}$ for all $S \in \mathbb{S}$.
 \label{eq:Assumption_2}
 \item The log-partition function $A(\cdot)$ is differentiable with a bounded second derivative over $[-2^{p}F,2^{p}F]$, i.e., there exists a positive constant $C_{A}$ such that
 \begin{align*}
     1/C_{A} \leq \ddot{A}(x) \leq C_{A}
 \end{align*}
 for all $x \in [-2^{p}F,2^{p}F].$ 
\label{eq:Assumption_5}
\item $T_{\max}=O(n)$.
\label{eq:Assumption_4}
\end{enumerate}

\begin{theorem}[Posterior concentration rate of ANOVA-BART] \label{theorem_SIBART} 
Assume that \ref{eq:Assumption_1}, \ref{eq:Assumption_2},\ref{eq:Assumption_5} and \ref{eq:Assumption_4} hold.
Then, for a given $\xi> 2^{p}F,$ we have
\begin{equation*}
\pi_\xi\big\{f : \Vert f_{0}-f \Vert_{2,n} > B_{n}\epsilon_{n} \big| \mathbf{X}^{(n)},Y^{(n)}\big\} \xrightarrow{} 0,
\end{equation*}
for any $B_{n} \xrightarrow{} \infty$ in $\mathbb{P}_0^{n}$ as $n \xrightarrow{} \infty,$
where 
$\epsilon_{n}=n^{-\frac{\alpha}{2\alpha+d_{\max}}} (\log n)^{1\over 2}.$ 
%and $\mathbb{S}_{q} = \cup_{d=1}^{q}\text{power}(p,d).$
\end{theorem}

Note that our prior is assigned to satisfy the $\mu_{n}$-identifiability condition, while the true function satisfies the $\mathbb{P}_{\mathbf{X}}$-identifiability condition.
This discrepancy constitutes one of the technically delicate points of the proof of Theorem \ref{theorem_SIBART}.

\subsection{Comparison with Bayesian Tree Ensemble Models}

\cite{onBART, artbart} derived the posterior concentration rates for BART, which are minimax optimal. 
Moreover, \cite{bforest} developed Bayesian forest, whose posterior concentration rate is adaptive to $d_{\max}$ as ANOVA-BART is, but its practical implementation would be computationally demanding.

ANOVA-BART adopts a fundamentally different strategy than BART.
%Specifically, the way of approximation $f_{0}$ is much different from BART. 
In particular, the number of decision trees increases as the sample size 
increases in ANOVA-BART while the sizes of each decision tree increase in BART.
ANOVA-BART approximates $f_{0}$ using a sum of many small decision trees, whereas BART approximates $f_{0}$ using a finite number of large decision trees.
That is, we have to prove that 
the sum of many small identifiable binary-product trees can approximate a smooth function well.
For this purpose, we first
approximate a smooth given function by a large decision tree as is done by \cite{bforest}.
Then, we prove that
this large decision tree can be represented by the sum of many small identifiable binary-product trees, which
is technically quite involved.
See Section \ref{app:approximation} of Supplementary Material.

An important practical advantage of ANOVA-BART compared to BART
is Theorem \ref{theorem_component} below, which derives the posterior concentration of each component.
An obvious applications of Theorem \ref{theorem_component} is to screen out unnecessary components a posteriori.

\begin{theorem}[Posterior concentration rate of each component]
\label{theorem_component}
Let $p^{\rm ind}_{\bold{X}}(\mathbf{x})= \prod_{j=1}^p p_{X_j}(x_j),$ where
$p_{X_j}$ is the density of $X_j$.
In addition to Assumptions \ref{eq:Assumption_1}-\ref{eq:Assumption_4}, we further assume that 
\begin{equation*}
0< \inf_{\bold{x}\in \mathcal{X}}\frac{p_{\mathbf{X}}(\bold{x})}{p^{\rm ind}_{\mathbf{X}}(\bold{x})}\le \sup_{\bold{x}\in \mathcal{X}}\frac{p_{\mathbf{X}}(\bold{x})}{p^{\rm ind}_{\mathbf{X}}(\bold{x})}<\infty.
\end{equation*}
Then, for $S \subseteq [p]$ and $\xi> 2^{p}F,$ we have
\begin{equation*}
\pi_{\xi}\big\{f  : \Vert f_{0,S}-f_S \Vert_{2,n} > B_{n}\epsilon_{n} \big| \mathbf{X}^{(n)},Y^{(n)} \big\} \xrightarrow{} 0,
\end{equation*}
for any $B_{n} \xrightarrow{} \infty$ in $\mathbb{P}_{0}^{n}$, as $n \xrightarrow{} \infty.$
\end{theorem}

The result of Theorem \ref{theorem_component} can be fruitfully utilized to screen out unnecessary components after obtaining the posterior distribution. That is, we can delete $f_S$ from the regression function when $\|f_S\|_{2,n}$ is small a posteriori.
To be more specific, we delete $f_S$ if 
$$\pi\{\|f_S\|_{2,n} > \epsilon_n \log n |\mathbf{X}^{(n)},Y^{(n)} \} <\delta$$
for a given positive real number $\delta$.
Theorem \ref{theorem_component} implies that this post-hoc component selection procedure is selection-consistent (i.e. deleting all unnecessary components).

\section{Experiments}
\label{sec:experiment}
This section presents the results of numerical experiments of ANOVA-BART.
In Section \ref{sec:synthetic}, we conduct an analysis of synthetic data, while we focus on real data analysis
in Section \ref{sec:real_data}.
In Section \ref{sec:stable_experiment}, we investigate the stability of ANOVA-BART in estimating the components.

We consider BART (\cite{BART}), SSANOVA (\cite{SSANOVA}), 
MARS (\cite{MARS}) and NAM (\cite{nam})
as the baseline methods.
We use the official code in `BayesTree' R package (\cite{chipman2016package}), 
`gss' R package (\cite{gu2014smoothing}), `earth' R package (\cite{milborrow2017earth})
and the official code in \href{https://github.com/AmrMKayid/nam}{https://github.com/AmrMKayid/nam} for 
implementing BART, SSANOVA, MARS and NAM, respectively.
For NAM, we extend the code to implement NAM to include the second order interactions.
The detailed descriptions of the selection of hyperparameters in ANOVA-BART and baseline methods are presented in Section \ref{app:experiment} of Supplementary Material.
Note that while SSANOVA and NAM require specifying the maximum order of the components to be estimated, ANOVA-BART does not require the predefined maximum order. We set the maximum order of interaction at 2 for SSANOVA and NAM 
because a larger value than 2 would not be possible due to memory deficiency.

\subsection{Analysis of synthetic data}
\label{sec:synthetic}

We evaluate the performance of ANOVA-BART in view of prediction and component selection by analyzing synthetic data.
To do this, we generate data using the Friedman's test function (\cite{MARS,BART,linero2018bayesian}) defined as
\begin{align*}
y=10\text{sin}(\pi x_{1}x_{2})+20(x_{3}-0.5)^{2}+10x_{4}+5x_{5}+\epsilon,
\end{align*}
where $\epsilon \sim N(0,\sigma_{\epsilon}^{2})$.
The error variance $\sigma_{\epsilon}^{2}$ is set to make the signal-to-noise ratio be 5.
For training data, we generate $\mathbf{x}_{i}=(x_{i,1},...,x_{i,p})^\top$, $i=1,...,1000$ from the uniform distribution on $(0,1)^p$ and then generate the response variable using only the first 5 covairates.
This means that all covariates except the first five are noises.

\subsubsection{Prediction performance}
\label{sec:pred_syn}
For all methods, hyperparameters are selected via 5-fold cross-validation on the training data. Independently generated test data of size 10,000 are used to compute the Root Mean Square Error (RMSE) as a measure of prediction performance. For ANOVA-BART and BART, we use the Bayes estimators.

Table \ref{table:synthetic_result} presents the averages and standard errors of the RMSEs for ANOVA-BART and baseline methods based on 5 repetitions of the simulation.
The results amply support that ANOVA-BART is favorably compared to the baseline methods in prediction performance.

\begin{table}[h]
\centering
\footnotesize
\caption{Averages and standard errors of RMSEs on the synthetic data.}
\label{table:synthetic_result}
\begin{tabular}{cccccc}
\toprule
 & ANOVA-BART  & BART & MARS & SSANOVA & NAM \\ \midrule
p=10 & 1.188 (0.01)  & 1.174 (0.02)  & 1.196 (0.02) & 1.207 (0.05)  & 1.207 (0.03)  \\ 
p=50 & 1.224 (0.01)  & 1.256 (0.03) &  1.310 (0.04) & 1.550 (0.05) & 1.713 (0.04)   \\ 
p=100 & 1.275 (0.02)  & 1.304 (0.03) & 1.407 (0.04) & 1.702 (0.14) & 1.779 (0.04) \\ \bottomrule
\end{tabular}%
\end{table}

\subsubsection{Component selection}
\label{sec:comp_select}
To examine how effectively ANOVA-BART selects the true signal components, we conduct a simulation.
We use the $l_{2}$ norm of each estimated component (i.e., $\Vert f_{S} \Vert_{2,n}$) as the important score, i.e., if $\Vert f_{S} \Vert_{2,n}$ large, we consider $f_{S}$ to be important.
We select the top 10 components based on the normalized important scores
(the important scores divided by the maximum important score), which are presented 
 in Figure \ref{fig:component_importance} for various values of $p.$
The red bars correspond to the signal components, while the blue bars are noisy ones.
The results strongly indicate that ANOVA-BART can detect signal components very well.

\begin{figure}[h]
    \centering
    \includegraphics[width=0.8\linewidth]{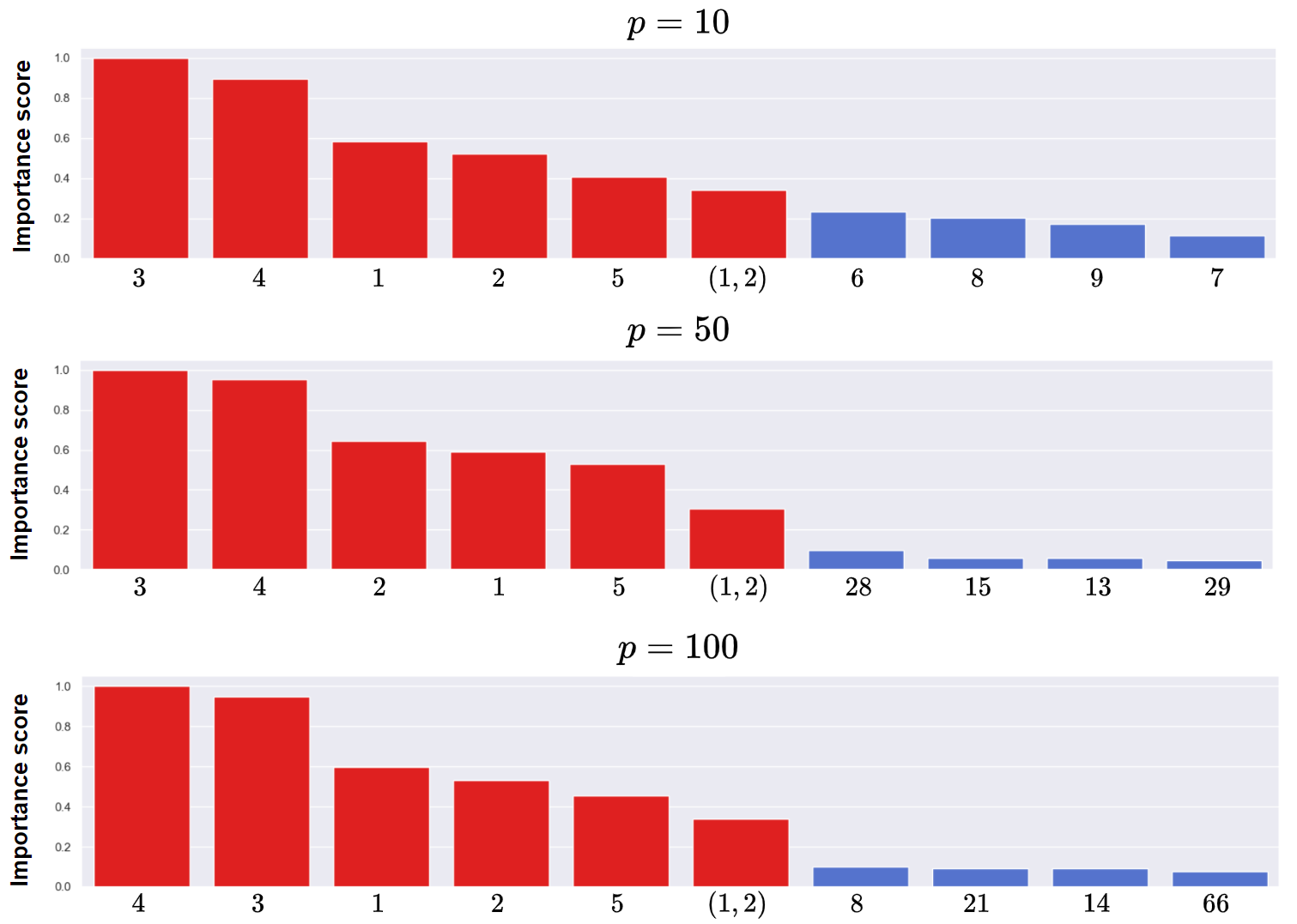}
    \caption{\footnotesize Importance scores of the estimated components by ANOVA-BART for $p=10,50$ and $100$.
    The importance scores are normalized by dividing each score by the maximum importance score.}
    \label{fig:component_importance}
\end{figure}

\subsection{Real data analysis}
\label{sec:real_data}

\begin{table}[h]
\footnotesize
\begin{center}
\caption{Summaries of real data sets}
\vskip 0.1cm
\label{table: Dataset}
\begin{tabular}{cccccr}
\toprule
Real data      & Size & Dimension of covariates & Task \\
\midrule
\textsc{Boston} (\cite{perera_boston_housing_kaggle})   & 506 & 13 & Regression \\
\textsc{Abalone} (\cite{abalone}) & 4,177 &  8 & Regression \\
\textsc{Servo} (\cite{servo_87}) & 167 & 4 & Regression \\
\textsc{Mpg} (\cite{auto_mpg_9}) & 392  & 7 & Regression \\ \midrule
\textsc{Breast} (\cite{breast_cancer_wisconsin_(diagnostic)_17}) & 569 & 30 & Classification\\
\textsc{Churn} (\cite{churn}) & 7,043 & 20 & Classification\\
\textsc{Madelon} (\cite{madelon_171}) & 2,600 & 500 & Classification\\
\bottomrule
\end{tabular}
%\end{sc}
\end{center}
\end{table}

We analyze 7 real data sets including \textsc{Boston}, \textsc{Abalone}, \textsc{Servo}, \textsc{Mpg}, \textsc{Breast}, \textsc{Churn}, and \textsc{Madelon} data sets. 
The four data sets \textsc{Boston}, \textsc{Abalone}, \textsc{Servo} and \textsc{Mpg} are analyzed in \cite{BART} while the remaining three data sets \textsc{Breast}, \textsc{Churn} and \textsc{Madelon} are examined in \cite{park2025bayesian}.
Table \ref{table: Dataset} summarizes the descriptions of the 5 data sets.
We split each data into 80\% training data and 20\% test data, and repeat this random split 5 times to
obtain 5 prediction measures.

\subsubsection{Prediction performance}

Table \ref{table:real_pred} presents the averages and standard errors of the RMSE and AUROC values for ANOVA-BART and baseline methods on 7 real data sets.
Similar to the case with synthetic data, ANOVA-BART shows comparable prediction performance on real data as well. 
In particular, on the \textsc{Madelon} and \textsc{Servo} datasets, ANOVA-BART as well as BART and MARS
outperform NAM and SSANOVA significantly. Note that the former three methods are able to detect higher order interactions while
the later two methods are limited to estimating the second-order interactions.
Table \ref{table:high_component_score} presents the top 5 most important components according to the normalized importance scores defined in Section \ref{sec:comp_select}, which clearly shows that higher order signal interactions exist in the  \textsc{Madelon} and \textsc{Servo} datasets. 
That is, NAM and SSANOVA fail to detect signal higher order interactions which result in inferior prediction performances. In contrast, ANOVA-BART detects such higher order interactions 
successfully. BART and MARS can detect higher-order interactions but BART is not easily interpretable and MARS is not good at
uncertainty quantification.

\begin{table}[h]
\centering
\scriptsize
\caption{Averages and standard errors of RMSEs on real data sets.}
\label{table:real_pred}
\begin{tabular}{ccccccc}
\toprule
Real data & Measure &ANOVA-BART  & BART & MARS & SSANOVA & NAM \\ \midrule
\textsc{Boston} & RMSE $\downarrow$ & 3.448 (0.59)  & 4.073 (0.67) & 4.788 (0.67) & 4.460 (0.65) &  3.832 (0.67) \\ 
\textsc{Abalone} & RMSE $\downarrow$& 2.112 (0.24)  & 2.197 (0.26) & 2.137 (0.25) & 2.137 (0.24) & 2.062 (0.23)  \\ 
\textsc{Servo} &RMSE $\downarrow$ & 0.316 (0.02)  & 0.342 (0.04) & 0.441 (0.06) & 0.820 (0.03) & 0.802 (0.04)  \\ 
\textsc{Mpg} &RMSE $\downarrow$ & 2.486 (0.32) & 2.699 (0.43)  &  3.019 (0.52)  & 3.091 (0.47) & 2.755 (0.41) \\ \midrule
\textsc{Breast} & AUROC $\uparrow$& 0.998 (0.001)  & 0.993 (0.002) & 0.993 (0.003) & 0.984 (0.004) & 0.988 (0.001)\\
\textsc{Churn} & AUROC $\uparrow$ & 0.852 (0.007)  & 0.849 (0.006) & 0.847 (0.007) & 0.848 (0.007) &  0.848 (0.008)\\
\textsc{Madelon} & AUROC $\uparrow$ & 0.866 (0.005) & 0.765 (0.005) & 0.863 (0.016) &  0.553 (0.009) & 0.644 (0.001)\\
\bottomrule
\end{tabular}%
\end{table}

\begin{table}[h]
\caption{Top 5 important components.}
\label{table:high_component_score}
\scriptsize
\scalebox{0.9}{
\begin{tabular}{c cc cc cc cc cc}
\toprule
 & 
\multicolumn{2}{c}{Rank 1} & \multicolumn{2}{c}{Rank 2} & 
\multicolumn{2}{c}{Rank 3} & \multicolumn{2}{c}{Rank 4} & 
\multicolumn{2}{c}{Rank 5} \\
\midrule
Dataset & Component & Score & Component & Score & Component & Score & Component & Score & Component & Score \\
\midrule
\textsc{Servo} & \{1\} & 1.000 & \{1,7,8\} & 0.461 & \{8\}& 0.434 & \{1,8\} & 0.324 & \{6\} & 0.298  \\
\textsc{Madelon} & \{49, 319, 339\} & 1.000 & \{337\} & 0.951 & \{319, 454\} & 0.852 & \{106\} & 0.725 & \{49\} & 0.528 \\
\bottomrule
\end{tabular}
}
\end{table}

\subsubsection{Uncertainty quantification}

An important advantage of Bayesian methods compared to frequentist's counterparts is superior performance
of uncertainty quantification.
We compare ANOVA-BART and BART in view of uncertainty quantification.
As a measure of uncertainty quantification, for regression task, we consider
Continuous Ranked Probability Score (CRPS, \cite{gneiting2007strictly}). 
For a given test sample  $(\bold{x},y),$ CRPS is defined as 
\begin{align*}
\text{CRPS}(F_{\bold{x}},y) := \int_{-\infty}^{\infty}(F_{\mathbf{x}}(z) - \mathbb{I}(y\leq z))^{2}dz,
\end{align*}
where $F_{\bold{x}}$ is the predictive cumulative distribution of $Y$ given $\bold{X}=\bold{x}.$
For classification task, we consider Expected Calibration Error (ECE, \cite{kumar2019verified}) as an uncertainty quantification measure.
For evaluation, we typically report the average CRPS or ECE on the test data, where smaller values indicate better uncertainty quantification.

Table \ref{table:real_crps} presents the averages and standard errors of 5 CRPS and 5 ECE values obtained by 5 random splits of training and test data.
Not surprisingly, the two Bayesian methods, BART and ANOVA-BART, show better performance than the frequentist method, MARS.

\begin{table}[h]
\centering
\caption{Averages and standard errors of CRPS and ECE on real data sets.}
\label{table:real_crps}
\vskip 0.1cm
\begin{tabular}{ccccc}
\toprule
Real data & Measure &ANOVA-BART & BART & MARS\\ \midrule
\textsc{Boston} & CRPS $\downarrow$ & 2.368 (0.29)  & 2.623 (0.25) & 3.580 (0.49)  \\ 
\textsc{abalone} & CRPS $\downarrow$&1.351 (0.16)  & 1.384 (0.18) & 1.575 (0.16)  \\ 
\textsc{servo} & CRPS $\downarrow$& 0.200 (0.02) & 0.202 (0.02) & 0.268 (0.03)  \\ 
\textsc{mpg} &CRPS $\downarrow$ &1.564 (0.17) & 1.553 (0.27) & 2.367 (0.47)\\ \midrule
\textsc{Breast} & ECE $\downarrow$ &0.066 (0.011) & 0.071 (0.013) & 0.072 (0.028)\\
\textsc{Churn} & ECE $\downarrow$ & 0.030 (0.002) & 0.033 (0.003) & 0.045 (0.001)\\
\textsc{Madelon} & ECE $\downarrow$ & 0.074 (0.004) & 0.084 (0.008) & 0.106 (0.005)\\
\bottomrule
\end{tabular}%
\end{table}

\subsection{Stable interpretation of ANOVA-BART}
\label{sec:stable_experiment}
%\subsubsection{Interpretation of ANOVA-BART using estimated components}

\begin{figure}[t]
    \centering
    \includegraphics[width=1.0\linewidth]{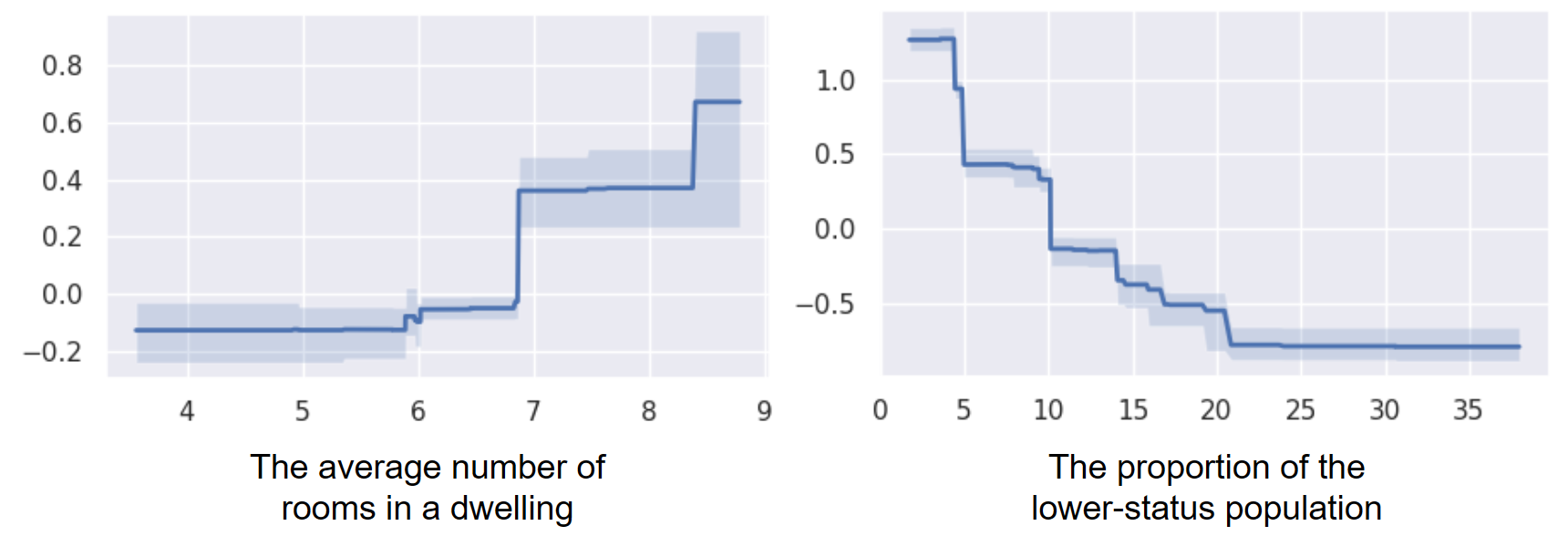}
    \caption{The functional relations of the two estimated components by ANOVA-BART on \textsc{Boston} data.}
    \label{fig_component_plot}
\end{figure}

Relations between the covariates and response variable can be interpreted by
looking at the functional relations of each estimated component in ANOVA-BART.
For illustration, Figure \ref{fig_component_plot} presents the plots of the two estimated components of ANOVA-BART on \textsc{Boston} data. Specifically, each plot displays the Bayes estimate along with the 95\% pointwise credible interval for `The average number of rooms in a dwelling' and `The proportion of the lower-status population'.
These plots suggest that
an increase in `The average number of rooms in a dwelling' is associated with a positive contribution to housing prices, whereas an increase in `The proportion of the lower-status population' is associated with a negative contribution to housing prices.

For this interpretation to be useful, the estimated components should be stable in the sense that
the result is robust to a small perturbation of traning data. 
In fact, ANOVA-BART employs the sum-to-zero condition to ensure that each component is identifiable and so can be
estimated stably. In contrast, neural network based methods such as NAM (\citep{nam}) and NBM (\citep{nbm})
are unstable in estimation of the component mainly due to the non-identifiability of each component.

Here, we conduct an experiment to compare the stability of ANOVA-BART in component estimation with that of NAM.
For this experiment, we estimate the components based on a bootstrap dataset.
This procedure is repeated five times to have five estimators for each component. 
We then compute the stability score using these estimators, where the stability score is defined as
\begin{align*}
\mathcal{SC}^{d}(f) = \sum_{S \subseteq [p], |S| \leq d}\mathcal{SC}(f_{S}).
\end{align*}
Here, $\mathcal{SC}(f_{S})$ is defined as
\begin{align*}
    \mathcal{SC}(f_S) := \frac{1}{|\mathcal{C}_S|}\sum_{\mathbf{x}\in \mathcal{C}_S} \frac{\sum_{j=1}^{5}(\hat{f}_S^j(\mathbf{x})-\bar{f}_S(\mathbf{x}))^2}{\sum_{j=1}^5 (\hat{f}_S^j(\mathbf{x}))^2},
\end{align*}
where $\hat{f}^{j}_{S}$ denotes the estimate of component $f_{S}$ obtained from the $j$th bootstrap data, $\bar{f} = {1\over 5}\sum_{j=1}^{5}\hat{f}_{S}^{j}(\tilde{\mathbf{x}}_{i}),$ and 
$\mathcal{C}_S=\prod_{j\in S}\{ x_{(0.09),j},x_{(0.19),j},\ldots, x_{(0.99),j}\}$ with
$x_{j,(\alpha)}$ being the $\alpha\times 100\%$ quantile of the empirical distribution of $\{x_{1,j},\ldots,x_{n,j}\}.$
A lower stability score indicates a more stable estimation of the component.

Table \ref{table:stab_score} compares the stability scores $\mathcal{SC}^{2}(f)$ of ANOVA-BART and NAM for multiple benchmark datasets.
These results confirm that ANOVA-BART outperforms the NAM in terms of the stability of component estimation, which suggests that enforcing the identifiability condition is important for stable interpretation.

\begin{table}[h]
\centering
\caption{Stability scores of ANOVA-BART and NAM.}
\label{table:stab_score}
\vskip 0.1cm
\begin{tabular}{cccc}
\toprule
Real data &ANOVA-BART & NAM  \\ \midrule
\textsc{Boston} & 0.435 & 0.705   \\ 
\textsc{abalone} & 0.452 & 0.770  \\ 
\textsc{servo}  & 0.478  & 0.665  \\ 
\textsc{mpg}  & 0.515 & 0.560 \\ \midrule
\textsc{Breast}  & 0.133  & 0.730  \\
\textsc{Churn} & 0.083 &  0.730 \\
\bottomrule
\end{tabular}%
\end{table}

\section{Conclusion}

ANOVA-BART can be considered as an interpretable modification of BART
and shows comparable performance to BART as well as other competitors on multiple benchmark data sets.
Theoretically, ANOVA-BART achieves not only a (near) minimax posterior concentration rate and but also selection consistency. 
Moreover, the posterior concentration rate is adaptive to the interaction structure of the true function
(i.e. the maximum order of signal components).

There are several possible future works.
First, theoretical results are obtained for a fixed dimension
$p$ of covariates. It would be interesting to modify ANOVA-BART for high-dimensional cases where $p$
diverges also.
Componentwise sparse priors would be needed for this purpose. 
Second, it would be useful to explore a new MCMC algorithm for improving the scalability of ANOVA-BART.

\bigskip
\appendix
\newpage
\begin{center}
{\large\bf Supplementary Material for\\ `Bayesian Additive Regression Trees\\ for functional ANOVA Model'}
\end{center}

\renewcommand{\thesection}{\Alph{section}}
%\numberwithin{equation}{section}
\numberwithin{equation}{section}

\section{Posterior Sampling}\label{B}
\renewcommand{\theequation}{A.\arabic{equation}}

In this section, we provide details of the conditional posteriors and the acceptance probabilities of the proposed MH algorithm within the MCMC algorithm in Section \ref{sec:posterior}.

\subsection{Acceptance probability for $T$}

For a given $\Theta=\{T,S_{1},\mathbf{s}_{1},\beta_{1},...,S_{T},\mathbf{s}_{T},\beta_{T},\psi\}$, we define a likelihood as
\begin{align*}
\mathcal{L}(\Theta) := \prod_{i=1}^{n}p_{f(\mathbf{x}_{i}),\psi}(y_{i}),
\end{align*}
where $f(\mathbf{x}_{i})=\sum_{t=1}^{T}\mathbb{T}(\mathbf{x}_{i}:S_{t},\mathbf{s}_{t},\beta_{t})$.

\subsubsection{$T^{\text{new}}=T+1$}

\paragraph{Proposal for $T$.}

We propose the new state 
$$
\Theta^{\text{new}} = \{T+1,(S_{1},\mathbf{s}_{1},\beta_{1}),...,(S_{T},\mathbf{s}_{T},\beta_{T}),(S_{T+1}^{\text{new}},\mathbf{s}_{T+1}^{\text{new}},\beta_{T+1}^{\text{new}}),\psi\}
$$
using Random or Stepwise move.

\paragraph{Transition proability.}
$q(\Theta|\Theta^{\text{new}})$ is given as 
\begin{align*}
&q(\Theta|\Theta^{\text{new}}) \\
&= \text{Pr}\{\text{Decrease the number of trees}\}\text{Pr}\{\text{Select a single tree from}\:T^{\text{new}}\:\text{trees to delete} \}  \\
&= {T^{\text{new}}\over T_{\max}}\cdot{1\over T^{\text{new}}}.
\end{align*}
For Random move, we have
\begin{align*}
q(\Theta^{\text{new}}|\Theta) &= \text{Pr}\{\text{Increase the number of trees}\}\pi\{S_{T+1}\}\pi\{\mathbf{s}_{T+1}\}\pi\{\beta_{T+1}\}
\end{align*}
and for Stepwise move, we have
\begin{align*}
&q(\Theta^{\text{new}}|\Theta) \\
&= \text{Pr}\{\text{Increase the number of trees}\}
\text{Pr}\{\text{Choose}\: S_{T+1}\:\text{using Stepwise}\}\pi\{\mathbf{s}_{T+1}\}\pi\{\beta_{T+1}\}
\end{align*}
where
\begin{align*}
&\text{Pr}\{\text{Increase the number of trees}\} = 1- {T^{\text{new}} / T_{\max}}
\end{align*}
and
\begin{align*}
&\text{Pr}\{\text{Choose}\: S_{T+1}\:\text{using Stepwise}\} \\
&= \sum_{t=1}^{T}\text{Pr}\{\text{Choose}\:S_{t}\}\text{Pr}\{S_{T+1}^{\text{new}} = S_{t}\cup \{ j^{\text{new}}\}\:\text{for some}\:j^{\text{new}} \in S_{t}^{c}\} \\
&= \sum_{t=1}^{T}{1\over T}\mathbb{I}(\exists j^{\text{new}} \in S_{t}^{c} \:\: \text{such that} \:\: S_{t}\cup \{j^{\text{new}}\} = S_{T+1}^{\text{new}}){\mathfrak{w}_{j^{\text{new}}}\over \sum_{j \in S_{t}^{c}}\mathfrak{w}_{j}}.
\end{align*}

\paragraph{Posterior ratio.}
Since the prior ratio is given as
\begin{align*}
{\pi\{\Theta^{\text{new}}\} \over \pi\{\Theta\} } = {\pi\{T+1\} \pi\{S_{T+1}^{\text{new}}\} \pi\{\mathbf{s}_{T+1}^{\text{new}}\}\pi\{\beta_{T+1}^{\text{new}} \} \over \pi\{T\}},
\end{align*}
we have
\begin{align*}
{\pi\{\Theta^{\text{new}}|\text{others}\} \over \pi\{\Theta|\text{others} \} } = {\mathcal{L}(\Theta^{\text{new}}) \over \mathcal{L}(\Theta)}{\pi\{T+1\} \pi\{S_{T+1}^{\text{new}}\} \pi\{\mathbf{s}_{T+1}^{\text{new}}\}\pi\{\beta_{T+1}^{\text{new}} \} \over \pi\{T\}}
\end{align*}

\paragraph{Acceptance probability.}
Note that since the dimension of the continuous parameter varies with $T$, this MH algorithm is RJMCMC.
However, because the continuous parameter (height) is generated from the prior distribution, the Jacobian is equal to 1.
That is, we accept the new state $\Theta^{\text{new}}$ with probability $p_{\text{accept}}$ defined as
\begin{align*}
p_{\text{accept}} = \min \bigg \{ 1, {\pi\{\Theta^{\text{new}}|\text{others}\} \over \pi\{\Theta|\text{others} \} }{q(\Theta|\Theta^{\text{new}})\over q(\Theta^{\text{new}}|\Theta)} \bigg\}.
\end{align*}

\subsubsection{$T^{\text{new}} = T-1$}

As the acceptance probability for the case of $T^{\text{new}} = T - 1$ 
can be derived by simply reversing the procedure in the case of $T^{\text{new}}=T+1$, 
we omit the details.

\subsection{Acceptance probability for $(S_t,\bold{s}_t)$} \label{B.1}

% Without loss of generality, let $S_t=\{1,\ldots,d\}$ for $d \leq p$.
% For $\bold{v}_{t}\in \{-1,1\}^d,$ let $\mathcal{R}_{\bold{v}_{t}}=\{\bold{x}: \prod_{j=1}^d \mathbb{I}^{(v_{tj})}(x_j-s_j>0)\}.$

\subsubsection{Transition probability}

Note that the proposal distribution $q$ of $(S_{t}^{\rm new}, \bold{s}_{t}^{\rm new})$ is given as
\begin{align*}
&q(S_{t}^{\rm new},\bold{s}_{t}^{\rm new}|S_t,\bold{s}_t, \text{GROW})\\
&= \text{Pr}(\mbox{Selecting a new input variable $j^{\text{new}}$ from } S_{t}^{c}\:\text{based on}\:P_{\pmb{\mathfrak{w}}}(\cdot|S_{t}^{c}) \mbox{ and the corresponding split value})\\
&= \frac{\mathfrak{w}_{j^{\text{new}}}}{\sum_{j \in S_{t}^{c}}\mathfrak{w}_{j} } \frac{1}{|\mathcal{A}_{j^{\text{new}}}|},
\end{align*}
where $j^{\text{new}}$ is the index of a newly selected input variable,
\begin{align*}
&q(S_{t}^{\rm new},\bold{s}_{t}^{\rm new}|S_t,\bold{s}_t, \text{CHANGE})\\
&= \text{Pr}\{\mbox{Selecting $j^{\text{new}}$ from $S_{t}^{c}$}\:\text{based on}\: P_{\pmb{\mathfrak{w}}}(\cdot|S_{t}^{c}),\\
&\quad\quad\quad \mbox{and Deleting one from $S_t$ chosen uniformly at random, and choosing one from $\mathcal{A}_{j^{\text{new}}}$} \}\\
&= \frac{\mathfrak{w}_{j^{\text{new}}}}{\sum_{j \in S_{t}^{c}}\mathfrak{w}_{j}}\frac{1}{|S_{t}|} \frac{1}{|\mathcal{A}_{j^{\text{new}}}|},
\end{align*}
and 
\begin{align*}
 &q(S_{t}^{\rm new},\bold{s}_{t}^{\rm new}|S_t,\bold{s}_t, \text{PRUNE}) \\
=&
\text{Pr}\{\mbox{Selecting an input variable in $S_t$ to be deleted}\}  \\
=& \frac{1}{|S_t|}.   
\end{align*}

To sum up, we have
\begin{align*}
q(S_{t}^{\rm new},\bold{s}_{t}^{\rm new}|S_t,\bold{s}_t)
&= \frac{\mathfrak{w}_{j^{\text{new}}}}{\sum_{j \in S_{t}^{c}}\mathfrak{w}_{j}} \frac{1}{|\mathcal{A}_{j^{\text{new}}}|} \text{Pr}\{\text{GROW}\} \mathbb{I}(|S_{t}^{\rm new}|=|S_t|+1)\\
&+  \frac{\mathfrak{w}_{j^{\text{new}}}}{\sum_{j \in S_{t}^{c}}\mathfrak{w}_{j}}\frac{1}{|S_t|} \frac{1}{|\mathcal{A}_{j^{\text{new}}}|} \text{Pr}\{\text{CHANGE}\}  \mathbb{I}(|S_{t}^{\text{new}}|=|S_t|)\\
&+ \frac{1}{|S_t|} \text{Pr}\{\text{PRUNE}\} \mathbb{I}(|S_{t}^{\text{new}}|=|S_t|-1).
\end{align*}

\subsubsection{Posterior probability ratio}
\label{sec:post_ratio}

Let $\lambda_{t,i} = \sum_{k\neq t}\mathbb{T}(\mathbf{x}_{i}:S_{k},\mathbf{s}_{k},\beta_{k})$ for $i=1,...,n$.
The likelihood ratio is given as
\begin{align*}
\prod_{i=1}^{n}{p_{\mathbb{T}(\mathbf{x}_{i}:S_{t}^{\text{new}},\mathbf{s}_{t}^{\text{new}},\beta_{t})+\lambda_{t,i}}(y_{i}) \over p_{\mathbb{T}(\mathbf{x}_{i}:S_{t},\mathbf{s}_{t},\beta_{t})+\lambda_{t,i}}(y_{i})}.
\end{align*}
Moreover, the ratios of the priors for GROW, PRUNE and CHANGE are given as follows.
$\newline$
For GROW, we have
$${\pi(S_{t}^{\rm new},\bold{s}_{t}^{\rm new})\over \pi(S_t,\bold{s}_t)} = {\alpha_{\text{split}} (1- \alpha_{\text{split}}(2+d)^{-\gamma_{\text{split}}}) \over ( (1+d)^{\gamma_{\text{split}}} - \alpha_{\text{split}} ) \times (p-d)\eta},$$
where $d = |S_{t}|$, and $\eta=|\mathcal{A}_{j^{\text{new}}}|$, where $j^{\text{new}}$ is the index of a newly selected input variable. 
$\newline$
For PRUNE, we have
$${\pi(S_{t}^{\rm new},\bold{s}_{t}^{\rm new})\over \pi(S_t,\bold{s}_t)} = {(d^{\gamma_{\text{split}}} - \alpha_{\text{split}}) \times (p-d+1)\eta \over \alpha_{\text{split}} \times (1 - \alpha_{\text{split}} (1+d)^{-\gamma_{\text{split}}}) }$$
where $d=|S_{t}|$, and $\eta = |\mathcal{A}_{j^{\text{new}}}|$ where $j^{\text{new}}$ is the index of a deleted input variable.
$\newline$
For CHANGE, we have
\begin{align*}
{\pi\{S_{t}^{\rm new},\bold{s}_{t}^{\rm new}\}\over \pi\{S_t,\bold{s}_t\}} = {|\mathcal{A}_{j^{\text{deleting}}}| \over |\mathcal{A}_{j^{\text{new}}}|}.
\end{align*}

To sum up, the posterior probability is 
\begin{align*}
{\pi\{S_{t}^{\rm new},\bold{s}_{t}^{\rm new}|\text{others}\}\over \pi\{S_t,\bold{s}_t|\text{others}\}} = \bigg( \prod_{i=1}^{n}{p_{\mathbb{T}(\mathbf{x}_{i}:S_{t}^{\text{new}},\mathbf{s}_{t}^{\text{new}},\beta_{t})+\lambda_{t,i}}(y_{i}) \over p_{\mathbb{T}(\mathbf{x}_{i}:S_{t},\mathbf{s}_{t},\beta_{t})+\lambda_{t,i}}(y_{i})}\bigg) {\pi\{S_{t}^{\rm new},\bold{s}_{t}^{\rm new}\}\over \pi\{S_t,\bold{s}_t\}}
\end{align*}

\subsubsection{Acceptance probability}

In summary, we accept $(S_{t}^{\rm new},\bold{s}_{t}^{\rm new})$ with probability
$p_{\text{accept}},$ where
$$p_{\text{accept}}=\min \left\{ 1,{\pi\{S_{t}^{\rm new},\bold{s}_{t}^{\rm new}|\text{others}\}\over \pi\{S_t,\bold{s}_t|\text{others}\}}{q(S_t,\bold{s}_t|S_{t}^{\rm new},\bold{s}_{t}^{\rm new}) \over q(S_{t}^{\rm new},\bold{s}_{t}^{\rm new}|S_t,\bold{s}_t)}
\right\}.$$

\subsection{Acceptance probability for $\beta_{t}$ and $\mathbf{s}_{t}$} \label{sec:beta_sample}

\subsubsection{Proposal for $\beta_{t}$}

To propose new state $\beta_{t}^{\text{new}}$, we use Langevin Dynamics, i.e.,
\begin{align*}
\beta_{t}^{\text{new}} = \beta_{t} + {\varepsilon^{2}\over 2}{\partial \over \partial \beta_{t}}\log \pi\{\beta_{t}|\text{others}\} + \varepsilon m,
\end{align*}
where $m \sim N(0,1)$, $\varepsilon$ is a step size and ${\partial \over \partial \beta_{t}}\log \pi\{\beta_{t}|\text{others}\}$ is defined as
\begin{align*}
{\partial \over \partial \beta_{t}}\log \pi\{\beta_{t}|\text{others}\} &= {\partial \over \partial \beta_{t}}\mathcal{L}(T,\{S_{k},\mathbf{s}_{k},\beta_{k}\}_{k=1}^{T},\psi) - {\beta_{t}\over \sigma_{\beta}^{2}}.
\end{align*}

% For $\mathbf{s}_{t}^{\text{new}}$, we propose as follows.
% For $j \in S_{t}$, let $\mathbf{s}_{t,j,\text{down}}$ and $\mathbf{s}_{t,j,\text{up}}$ be the split values that are immediately smaller and immediately larger than $\mathbf{s}_{t,j}$, respectively, in terms of their rank among all candidate split values and let $\mathbf{s}_{t,j,\text{current}}=\mathbf{s}_{t,j}$.
% Then, $\mathbf{s}_{t,j}^{\text{new}}$ is proposed by sampling from $\mathbf{s}_{t,j,\text{down}}$, $\mathbf{s}_{t,j,\text{current}}$ and $\mathbf{s}_{t,j,\text{up}}$ with probabilities $p_{\mathbf{s}_{t,j,\text{down}}}$, $p_{\mathbf{s}_{t,j},\text{current}}$ and $p_{\mathbf{s}_{t,j,\text{up}}}$, respectively for $j \in S_{t}$.
% Here, each probability is defined as
% \begin{align*}
% p_{\mathbf{s}_{t,j,\text{down}}} \propto \pi\{\mathbf{s}_{t,j,\text{down}}|\text{others}\}, \: p_{\mathbf{s}_{t,j},\text{current}} \propto \pi\{\mathbf{s}_{t,j}|\text{others}\}\quad \text{and}\quad p_{\mathbf{s}_{t,j,\text{up}}} \propto \pi\{\mathbf{s}_{t,j,\text{up}}|\text{others}\}.
% \end{align*} 

\subsubsection{Transition probability ratio}

The transition probability ratio for the proposal distribution is given as 
\begin{align*}
{q(\beta_{t}|\beta_{t}^{\text{new}}) \over q(\beta_{t}^{\text{new}}| \beta_{t})} = \exp\bigg(-{1\over 2}\big( (m^{\text{new}})^{2} - m^{2} \big) \bigg),
\end{align*}
where 
$$
m^{\text{new}} = m + {\varepsilon\over 2}{\partial \over \partial \beta_{t}}\log \pi\{\beta_{t}|\text{others}\} + {\varepsilon\over 2}{\partial \over \partial \beta_{t}}\log \pi\{\beta_{t}^{\text{new}}|\text{others}\}
$$
and $\varepsilon$ is the step size.

\subsubsection{Posterior probability ratio}

Since 
\begin{align*}
{ \pi\{\beta_{t}^{\text{new}}\} \over \pi\{\beta_{t}\}} = \exp\bigg(-{1\over 2\sigma_{\beta}^{2}}(\beta_{t}^{\text{new}} - \beta_{t})^{2}\bigg),
\end{align*}
the posterior probability ratio is
\begin{align*}
{ \pi\{\beta_{t}^{\text{new}}|\text{others}\} \over \pi\{\beta_{t}|\text{others}\}} = {\mathcal{L}(T,\{S_{k},\mathbf{s}_{k},\beta_{k}\}_{k\neq t},S_{t},\beta_{t}^{\text{new}},\mathbf{s}_{t}^{\text{new}},\psi) \over \mathcal{L}(T,\{S_{k},\mathbf{s}_{k},\beta_{k}\}_{k=1}^{T},\psi)}\exp\bigg(-{1\over 2\sigma_{\beta}^{2}}(\beta_{t}^{\text{new}} - \beta_{t})^{2}\bigg).
\end{align*}

\subsubsection{Acceptance probability}

In summary, we accept $\beta_{t}^{\text{new}}$ with probability
$p_{\text{accept}},$ where
$$p_{\text{accept}}=\min \left\{ 1,{\pi\{\beta_{t}^{\text{new}}|\text{others}\} \over \pi\{\beta_{t}|\text{others}\}}{q(\beta_t|\beta_{t}^{\rm new} ) \over q(\beta_{t}^{\rm new}|\beta_t)}
\right\}.$$

\subsection{Sampling nuisance parameter $\psi$}
\label{sec:sigma_post}

In this section, we consider the gaussian regression model, i.e., $y_{i}$ is realization from
\begin{align*}
Y_{i}|\mathbf{x}_{i} \sim N(\cdot | f(\mathbf{x}_{i}),\sigma^{2})
\end{align*}
for $i=1,...,n$, where $\sigma^{2}$ is the nuisance parameter.
Since $\sigma^{2} \sim IG\left({v \over 2}, {v\lambda \over 2}\right)$,
it follows that
\begin{align*}
\sigma^{2} | \text{others} \sim IG\bigg({v + n \over 2}, {v\lambda + \sum_{i=1}^{n}(y_{i} - f(\mathbf{x}_{i}))^{2} \over 2}\bigg). 
\end{align*}

\newpage

\section{Proofs of Theorem \ref{theorem_SIBART} }\label{sec:concentration}
\renewcommand{\theequation}{B.\arabic{equation}}

\subsection{Additional notations}
For two positive sequences $\{a_n\}$ and $\{b_n\}$, we use the notation $a_n \lesssim b_n$ to indicate that there exists a positive constant $c>0$ such that $a_n \leq cb_n$ for all $n \in \mathbb{N}$. 
We use the little $o$ notation, writing $a_n = o(b_n)$ to mean that $\lim_{n \to \infty} a_n / b_n = 0$.
We denote $N(\varepsilon, \mathcal{F}, d)$ as the $\varepsilon$-covering number of $\mathcal{F}$ with respect to a semimetric $d$.
Let $\mathbb{P}_{\mathbf{X}}^{n}=\prod_{i=1}^{n}\mathbb{P}_{\mathbf{X}_{i}}$, where $\mathbb{P}_{\mathbf{X}_{i}}$ is the probability distribution of $\mathbf{X}_{i}$, for $i=1,...,n$. 
We denote $\Vert \cdot\Vert_{1}$ as a $\ell_{1}$ norm for a vector, i.e., for a given vector $\mathbf{e}=(e_{1},\ldots,e_{n})$, $\Vert\mathbf{e}\Vert_{1} := \sum_{i=1}^{n} |e_{i}|$.
For a real-valued function $f : \mathcal{X} \xrightarrow{} \mathbb{R}$, we denote $\Vert f \Vert_{\infty} := \sup_{\mathbf{x} \in \mathcal{X}}|f(\mathbf{x})|$.

\subsection{Overall strategy}
\label{app:overall_strategy}

For a given identifiable binary-product tree with the parameters $S,\mathbf{s},\beta,$ we 
define the corresponding binary-product partition $\mathcal{R}$ of $\mathcal{X}$ as
$$\mathcal{R}:=\bigg\{ \bigcap_{j\in S} \{\mathbf{x}: x_j \le s_j\}^{(v_j)}:  v_j\in \{-1,1\}, j\in S\bigg\}.$$
Here, we let $A^{(1)}=A$ and $A^{(-1)}=A^c$ for a given set $A$.
%Note that $\mathcal{R}$ is a partition of $\prod_{j\in S} \mathcal{X}_j.$
For a given binary-product partition $\mathcal{R},$ let $\text{var}(\mathcal{R})$
and $\text{sval}(\mathcal{R})$ be the set of split variables and the set of split values used in constructing $\mathcal{R},$
respectively. Since there is a one-to-one relation between $(S,\mathbf{s})$ and $\mathcal{R},$ we parameterize a given identifiable binary-product tree by
$\mathcal{R},\beta$ instead of $S,\mathbf{s},\beta.$ We will use these two parameterizations interchangeably unless there is any confusion.

Let $f$ be an ensemble of $T$ many identifiable binary-product trees.
Then, it can be parameterized by
$T, \mathcal{E}=(\mathcal{R}_1,\ldots,\mathcal{R}_T)$ and
$\mathcal{B}=(\beta_1,\ldots,\beta_T),$ where $(\mathcal{R}_t, \beta_t)$
is the parameter of the $t$th identifiable binary-product tree used in $f.$
Whenever we want to emphasize the parameters of a given ensemble $f,$ we write $f_{T,\mathcal{E},\mathcal{B}}$
or $f_{\mathcal{E},\mathcal{B}}.$
We refer to $\mathcal{E}=(\mathcal{R}_1,\ldots,\mathcal{R}_T)$ as an ensemble partition of length $T$. 
For a given $T,$  let $\mathcal{E}(T)$ be the set of all possible ensemble partitions $\mathcal{E}$ of length $T$. 
Finally, we let
$$\mathcal{F}:=\big\{f_{T,\mathcal{E},\mathcal{B}}: T\in [T_{\max}], 
\mathcal{E}\in \mathcal{E}(T), \mathcal{B} \in \mathbb{R}^{T} \big\},$$
which is the support of the prior $\pi$ (before truncation).

Note that our goal is to show that for any $\delta > 0$, 
\begin{align}
\lim_{n \to \infty}\mathbb{P}_{0}^{n}\bigg \{ \pi_{\xi} \{ \Vert f - f_{0} \Vert_{2,n}   > B_{n}\epsilon_{n} | \mathbf{X}^{(n)},Y^{(n)} \} > \delta \bigg \} = 0.
\label{eq:final_object}
\end{align}
We prove (\ref{eq:final_object}) as follows.
We first specify a subset $A_n$ of $\mathcal{X}^n$
such that $\mathbb{P}_{\mathbf{X}}^n\{A_n\}\rightarrow 1$ as $n\rightarrow \infty,$
and for any $\mathbf{x}^{(n)}\in A_n$ we have
\begin{align}
\mathbb{P}_{Y^{(n)}}\bigg \{ \{ \pi_{\xi} \{ \Vert f - f_{0} \Vert_{2,n} > B_{n}\epsilon_{n} | \mathbf{X}^{(n)},Y^{(n)} \} > \delta \} \bigg | \mathbf{X}^{(n)} = \mathbf{x}^{(n)} \bigg \} \rightarrow 0 
\label{eq:cond_conv_rate}
\end{align}
as $n\rightarrow \infty,$ 
where $\mathbb{P}_{Y^{(n)}}\{\cdot | \mathbf{X}^{(n)}\}$ is the conditional distribution of $Y^{(n)}$ given $\mathbf{X}^{(n)}$.
Then, we complete the proof easily since for any $\delta>0,$
\begin{align*}
&\mathbb{P}_{0}^{n}\bigg \{ \{ \pi_{\xi} \{ \Vert f - f_{0} \Vert_{2,n}  > B_{n}\epsilon_{n} | \mathbf{X}^{(n)},Y^{(n)} \} > \delta \} \bigg \} \\
&\leq \mathbb{P}_{0}^{n}\bigg \{ \{ \pi_{\xi} \{ \Vert f - f_{0} \Vert_{2,n}  > B_{n}\epsilon_{n} | \mathbf{X}^{(n)},Y^{(n)} \} > \delta \} \cap A_{n} \bigg \} + \mathbb{P}_{\mathbf{X}}^{n}(A_{n}^{c}) \\
&\rightarrow 0
\end{align*}
as $n\rightarrow \infty.$

Let $C_{\mathbb{B}}$ be a positive constant to be determined later (Section \ref{app:final_condition_check} of Supplementary Material).
To prove (\ref{eq:cond_conv_rate}), we verify the following three conditions for $\mathbf{x}^{(n)}\in A_n:$
there exists $\mathcal{F}^n\subseteq \mathcal{F}$ depending on $\mathbf{x}^{(n)}$ such that
\begin{align}
& \log N\left({\epsilon_{n}\over 36},\mathcal{F}^{n},\Vert \cdot \Vert_{\infty} \right) \leq n\epsilon_{n}^{2} \label{eq1}\\
&\pi\big\{f \in \mathcal{F} : \Vert f - f_{0} \Vert_{\infty} \leq C_{\mathbb{B}}\epsilon_{n} \big\} \geq e^{-d_{1}n\epsilon_{n}^{2}} \label{eq2}
\\
&\pi\{\mathcal{F}\backslash \mathcal{F}^{n}\} \leq e^{-(2d_{1} + 2)n\epsilon_{n}^{2}} \label{eq3}
\end{align}
for some constant $d_{1} > 0$.
In turn, we show that these three conditions imply the posterior convergence rate in (\ref{eq:cond_conv_rate}) using Theorem 4 of \cite{nonp1}.
\subsection{Construction of $A_n$}
\label{app:A_{n}}
For constructing $A_n,$ we need the following theorem which shows that
any populationally identifiable smooth function can be approximated by a sum of identifiable binary-product trees
whose heights are bounded. This theorem is technically involved, and its proof is given in Section \ref{app:approximation} of Supplementary Material.

\begin{theorem} \label{lemma5}
For  $\alpha \in (0,1]$, suppose $g_{S}$ be a populationally identifiable $S$-component function in $\mathcal{H}^{\alpha}_{F}$.
Then, there exist positive constant $C$ that does not depend on the function $g_{S}$ such that 
for any $T_{S}\in \mathbb{N}_{+},$ there exist $T_{S}$ many binary-product partitions $\mathcal{R}_1,\ldots,\mathcal{R}_{T_{S}}$ with $var(\mathcal{R}_t)=S$ for all $t\in [T_{S}]$
and $T_{S}$ many real numbers $\beta_1,\ldots,\beta_{T_{S}}$ depending on $\mathbf{x}^{(n)}$ such that
$\sup_{j} |\beta_j|\le C$ and
\begin{align}
&\left\Vert g_{S}(\cdot) - \sum_{t=1}^{T_{S}} \mathbb{T}(\cdot: \mathcal{R}_t,\beta_t) \right\Vert_{\infty} \leq \Vert g
_{S}\Vert_{\mathcal{H}^{\alpha}}{C_{S}^{\alpha}\over 
(T_{S}^{1\over |S|} +1)^{\alpha}} + \Phi_{n,S}(\mathbf{x}^{(n)}, T_S),
\label{eq:approx-lemma}
\end{align}
where
\begin{align*}
&\mathbb{P}^n_{\mathbf{X}}\left\{ \Phi_{n,S}(\mathbf{X}^{(n)}, T_S)
\le (24F+1)2^{|S|}\sqrt{ (T_{S}^{1\over |S|} +1)^{|S|}}\sqrt{\log n \over n} \right\}\\
&\ge 1-2^{|S|+1}|S|(T_{S}^{1\over|S|}+1)^{|S|^{2}}\exp\left( - {(T_{S}^{1\over |S|}+1)^{|S|} \log n \over 4F^{2}} \right) 
\end{align*}
for all sufficiently large $n$ and $C_{S}=\sqrt{|S|}/p_{L}$.
\end{theorem}
$\newline$
For $S \in \mathbb{S}$, let $t_{n,S}$ be the smallest positive integer satisfying
\begin{align} \label{SP_app}
\Vert f_{0,S}\Vert_{\mathcal{H}^{\alpha}}{C_{S}^{\alpha}\over 
(t_{n,S}^{1\over |S|} +1)^{\alpha} } + (24F+1)2^{|S|}\sqrt{(t_{n,S}^{1 \over |S|}+1)^{|S|}} \sqrt{\log n \over n} \leq {C_{\mathbb{B}}\epsilon_{n}\over 2^{p+1}}.
\end{align} 
Note that $t_{n,S} \lesssim n\epsilon_{n}^{2} / \log n$.
Now, we let
\begin{equation}
\label{eq:An}
A_n := \left\{\mathbf{x}^{(n)} : \Phi_{n,S}(\mathbf{x}^{(n)}, t_{n,S})
\le (24F+1)2^{|S|}\sqrt{ (t_{n,S}^{1\over |S|} +1)^{|S|}}\sqrt{\log n \over n}
\mbox{ for all } S \in \mathbb{S} \right\}.
\end{equation}
By Theorem \ref{lemma5}, $\mathbb{P}_{\mathbf{X}}^{n}(A_n) \rightarrow 1$
as $n\rightarrow \infty.$

\subsection{Proof of Condition (\ref{eq1})} \label{proof_eq1}

We consider the sieve 
$$\mathcal{F}^{n}:=\big\{ f_{T,\mathcal{E},\mathcal{B}}: T\in [M_n], 
\mathcal{E}\in \mathcal{E}(T), \mathcal{B} \in [-n,n]^{T} \big\}.$$
for some positive integer $M_{n} \le T_{\max},$
and will choose $M_n$ so that $\mathcal{F}^n$ satisfies Condition
(\ref{eq1}) with
$\epsilon_{n}^{2} = n^{-{2\alpha\over 2\alpha +d_{\max}}}\log n$.
For this purpose, we need the following lemma.

\begin{lemma} \label{lemma:covering}
For given $T\in \mathbb{N}_+$ and $\mathcal{E}\in \mathcal{E}(T),$
let 
$$\mathcal{F}^n(T,\mathcal{E}):=\{ f_{T,\mathcal{E},\mathcal{B}}: \mathcal{B}\in [-n,n]^T\}.$$
Then, we have
$N(\epsilon,\mathcal{F}^n(T,\mathcal{E}),\Vert \cdot \Vert_{\infty} ) \leq \left ( 1 + {2n^{p+1} T \over \epsilon} \right )^{T}$ 
\end{lemma}
\begin{proof}
For the $T$-dimensional hypercube $[-n,n]^{T}$, We have 
\begin{align}
N(\epsilon_{1},[-n,n]^T, \Vert \cdot \Vert_{1}) \leq \bigg(1 + {2nT\over \epsilon_{1}}\bigg)^{T} \label{eq:ball_covering}
\end{align}
for any $\epsilon > 0$.
Let $\{\mathcal{B}^{1},...,\mathcal{B}^{N(\epsilon_{1},[-n,n],\Vert \cdot \Vert_{1})}\}$
 be an $\epsilon_{1}$-cover of $[-n,n]^T,$ and for given $\mathcal{B}\in [-n,n]^T,$
 let $\tilde{\mathcal{B}}$ be an element in the $\epsilon_{1}$-cover such that $\|\mathcal{B}-\tilde{\mathcal{B}}\|_1 \le \epsilon_{1}.$ 
%For a given ensemble $f_{T,\bold{S},\mathcal{E},\mathcal{B}}$  and $S\in \mathbb{S}_q,$
%let $f_{\mathcal{E}_S,\mathcal{B}_S}(\cdot)=\sum_{j=1}^T \mathbb{T}_j(\cdot:S_j,\mathcal{T}_j,\beta_j) \mathbb{I}(S_j=S).$
%Suppose that $\mathcal{B}\in [-n,n]^T.$
Then, for any $f_{\mathcal{E},\mathcal{B}} \in \mathcal{F}^n(T,\mathcal{E}),$
we have
\begin{align*}
\sup_{\mathbf{x}}|f_{\mathcal{E},\mathcal{B}}(\bold{x}) -  f_{\mathcal{E},\tilde{\mathcal{B}}} (\bold{x}) |  &\leq \sup_{\mathbf{x}} \sum_{t=1}^T 
|\mathbb{T}(\bold{x}:\mathcal{R}_t,\beta_t)  - \mathbb{T}(\bold{x}:\mathcal{R}_t,\tilde{\beta}_t) | \\ 
&\leq \sum_{t=1}^T  (1+n)^{|S_t|}|\beta_{t}-\tilde{\beta}_{t} | \\
&\leq (1+n)^{p} \epsilon_{1},
\end{align*}
where we use the inequality that 
\begin{align}
\begin{split}
\prod_{j \in S_{t}}\big( \mathbb{I}(x_{j} \leq s_{j}) + \mathfrak{a}_{t,j}\mathbb{I}(x_{j} > s_{j}) \big) &\leq \prod_{j \in S_{t}}|1+\mathfrak{a}_{t,j}| \\
&\leq (1+n)^{|S_t|}.  
\end{split} \label{eq:id_constan_bdd}
\end{align}
By letting $\epsilon_{1} = \epsilon / (1+n)^{p}$ in (\ref{eq:ball_covering}),
we have 
\begin{align}
N(\epsilon,\mathcal{F}^n(T,\mathcal{E}),\Vert \cdot \Vert_{\infty}) \leq \bigg ( 1 + {2n(1+n)^{p}T \over \epsilon} \bigg )^{T},
\end{align}
which completes the proof.
\end{proof}
$\newline$
An upper bound of the covering number of $\mathcal{F}^{n}$ is given as
\begin{align}
N\bigg({\epsilon_{n}\over 36},\mathcal{F}^{n},\Vert \cdot \Vert_{\infty} \bigg) \nonumber &\leq \sum_{T=1}^{M_{n}} \sum_{\mathcal{E} \in \mathcal{E}(T)} N\bigg( {\epsilon_{n}\over 36}, \mathcal{F}^n(T,\mathcal{E}), \Vert \cdot \Vert_{\infty} \bigg) %\label{eq:number_of_valid} 
\nonumber\\
& \leq \sum_{T=1}^{M_{n}} \bigg( 1+ { 72n(1+n)^{p}T \over \epsilon_{n} } \bigg)^{T}
\sum_{\mathcal{E} \in \mathcal{E}(T)} 1 \nonumber\\
&\lesssim \sum_{T=1}^{M_{n}} \bigg( 1+ { 72n(1+n)^{p}T \over \epsilon_{n} } \bigg)^{T} n^{Tp}\label{eq:yd1} \\
& \leq M_{n} n^{M_{n}p} \bigg( 1+ { 72n(1+n)^{p}M_{n} \over \epsilon_{n} } \bigg)^{M_{n}}, \nonumber
\end{align}
where the inequality in (\ref{eq:yd1}) is due to
\begin{align*}
\sum_{\mathcal{E} \in \mathcal{E}(T)} 1 \lesssim  n^{Tp}.
\end{align*}
Therefore, we have the following upper bound of the log covering number:
\begin{align*}
\log N\bigg({\epsilon_{n} \over 36},\mathcal{F}^{n},\Vert \cdot \Vert_{\infty} \bigg)\lesssim  \log M_{n} + M_{n}p\log n + M_{n}\log \bigg(1 + {72n(1+n)^{p}M_{n} \over \epsilon_{n}} \bigg )
\end{align*}
With $M_{n}= \lfloor C_{1}{n\epsilon_{n}^{2}\over \log n}  \rfloor$ for some large enough constant $C_{1}>0$ (will be determined later), Condition (\ref{eq1}) is satisfied.

\qed

\subsection{Proof of Condition (\ref{eq2})}\label{sec:eq2}
\label{app:eq2_proof} 
For given $\mathbf{x}^{(n)}\in A_n,$ let $\hat{f}_S$ be an ensemble of $t_{n,S}$ many identifiable binary-product trees satisfying
$\|f_{0,S}-\hat{f}_S\|_{\infty} \le C_{\mathbb{B}}\epsilon_{n}/2^{p+1},$ whose existence is guaranteed by 
Theorem \ref{lemma5}, and let  $\hat{f}=\sum_{S \in \mathbb{S} } \hat{f}_S.$
Let $\hat{\mathcal{E}}_S$ and $\hat{\mathcal{B}}_S$ be
the ensemble partitions and height vector for $\hat{f}_S$ and let
$\hat{\mathcal{E}}$ and $\hat{\mathcal{B}}$ be the ensemble partition and height vector for $\hat{f}$. 
Let $t_{n}=\sum_{S \in \mathbb{S} }t_{n,S}$.
\medskip
\newline
Note that for any $\mathcal{B}\in \mathbb{R}^{t_n},$
\begin{align*}
    \| f_{t_n, \hat{\mathcal{E}}, \mathcal{B}}- f_{t_n, \hat{\mathcal{E}}, \hat{\mathcal{B}}}\|_{\infty}
   &\le \sum_{t=1}^{t_n} \|\mathbb{T}(\cdot: \hat{\mathcal{R}}_{t},\beta_{t}) -   \mathbb{T}(\cdot: \hat{\mathcal{R}}_{t},\hat{\beta}_{t})\|_{\infty}\\
   &\le \sum_{t=1}^{t_n} (1+n)^{p} |\beta_{t}-\hat{\beta}_{t}|\\
   &\le (1+n)^p \sqrt{t_n} \|\mathcal{B}-\hat{\mathcal{B}}\|_2,
\end{align*}
where the second inequality follows from (\ref{eq:id_constan_bdd}).
Hence, if $\|\mathcal{B}-\hat{\mathcal{B}}\|_2 \le ((1+n)^{p} \sqrt{t_n})^{-1} C_{\mathbb{B}}\epsilon_n/2,$ we have
$$\|f_0-f_{t_n, \hat{\mathcal{E}}, \mathcal{B}}\|_{\infty} 
    \le   \| f_{0}-f_{t_n, \hat{\mathcal{E}}, \hat{\mathcal{B}}}\|_{\infty}
      +\| f_{t_n, \hat{\mathcal{E}}, \mathcal{B}}- f_{t_n, \hat{\mathcal{E}}, \hat{\mathcal{B}}}\|_{\infty}
      \le C_{\mathbb{B}}\epsilon_n.
$$
Thus, to prove Condition (\ref{eq2}), it suffices to
show that 
\begin{equation}
\label{eq:prob-c2-1}
\pi\left\{T=t_n, \mathcal{E}=\hat{\mathcal{E}}, \|\mathcal{B}-\hat{\mathcal{B}}\|_2 \le ((1+n)^{p} \sqrt{t_n})^{-1} C_{\mathbb{B}}\epsilon_n/2 \right\}
\end{equation}
is sufficiently large.
We decompose (\ref{eq:prob-c2-1}) as
\begin{align*}
&\pi\left\{T=t_n, \mathcal{E}=\hat{\mathcal{E}}, \|\mathcal{B}-\hat{\mathcal{B}}\|_2 \le ((1+n)^{p} \sqrt{t_n})^{-1} C_{\mathbb{B}}\epsilon_n/2\right\}\\
&= \pi\{T=t_n\}\pi\{ \mathcal{E}=\hat{\mathcal{E}} |T=t_n\}\pi\{\|\mathcal{B}-\hat{\mathcal{B}}\|_2 \le ((1+n)^{p} \sqrt{t_n})^{-1} C_{\mathbb{B}}\epsilon_n/2 | 
\mathcal{E}=\hat{\mathcal{E}}, T=t_n\}.
\end{align*}
We will show that these three prior probabilities on the right hand side of the above equality are sufficiently large.
\medskip
\newline
\textbf{B-(a).} A lower bound of $\pi\{T = t_{n}\}:$ 
Since $t_n\lesssim n\epsilon_{n}^{2} / \log n$, there exists a constant $d_{2} > 0$ such that
\begin{align*}
\pi\{T = t_{n}\} &= {\exp(-C_{*}t_{n}\log n) \over \sum_{t=0}^{T_{\max}}\exp(-C_{*}t\log n)  } \\
&\geq (1-n^{-C_{*}})\exp(-C_{*}t_{n}\log n)  \\
& \geq \exp(-d_{2}n\epsilon_{n}^{2})
\end{align*}
for all sufficiently large $n$.
\medskip
\newline
\textbf{B-(b).} A lower bound of $\pi\{\mathcal{E}=\hat{\mathcal{E}}|T=t_n\}:$
Let $\mathcal{E}=(\mathcal{R}_t, t=1,\ldots,t_n)$
and $\hat{\mathcal{E}}=(\hat{\mathcal{R}}_t, t=1,\ldots,t_n).$
Note that conditional on $T=t_n,$ $\mathcal{R}_t, t\in [t_n]$ are independent a priori and thus we have
$$\pi\{\mathcal{E}=\hat{\mathcal{E}}|T=t_n\} =\prod_{t=1}^{t_n} \pi\{\mathcal{R}_t=\hat{\mathcal{R}}_t\}.$$
In turn,
\begin{align*}
\pi\{\mathcal{R}_t=\hat{\mathcal{R}}_t \} = & \pi\left\{|\text{var}(\mathcal{R}_t)|=|\text{var}(\hat{\mathcal{R}}_t)|\right\}\\
&\times \pi\left\{\text{var}(\mathcal{R}_t)=\text{var}(\hat{\mathcal{R}}_t)\big||\text{var}(\mathcal{R}_t)|=|\text{var}(\hat{\mathcal{R}}_t)|\right\}\\
&\times \pi\left\{\text{sval}(\mathcal{R}_t)=\text{sval}(\hat{\mathcal{R}}_t)\big|\text{var}(\mathcal{R}_{t}) = \text{var}(\hat{\mathcal{R}}_{t})\right\},    
\end{align*}
where 
\begin{align*}
&\pi\left\{|\text{var}(\mathcal{R}_t)|=|\text{var}(\hat{\mathcal{R}}_t)|\right\} = {\omega_{|\text{var}(\hat{\mathcal{R}}_t)|} \over \sum_{d=0}^{p}\omega_{d} }\ge {\omega_{p}\over 2},\\    
&\pi\left\{\text{var}(\mathcal{R}_t)=\text{var}(\hat{\mathcal{R}}_t)\big||\text{var}(\mathcal{R}_t)|=|\text{var}(\hat{\mathcal{R}}_t)|\right\} = 1\big/{p \choose |\text{var}(\hat{\mathcal{R}}_t)|}
\end{align*}
and 
$$\pi\left\{\text{sval}(\mathcal{R}_t)=\text{sval}(\hat{\mathcal{R}}_t)\big|\text{var}(\mathcal{R}_{t}) = \text{var}(\hat{\mathcal{R}}_{t})\right\}= \prod_{j\in \text{var}(\hat{\mathcal{R}}_t)} {1 \over |\mathcal{A}_j|} \ge n^{-p}.$$
Thus we have 
$$\pi\big\{\mathcal{E}=\hat{\mathcal{E}}|T=t_n\big\} \ge \left(\frac{\omega_{p}}{2{p \choose |\text{var}(\hat{\mathcal{R}}_{t})| }} n^{-p}\right)^{t_n}
\ge\exp(-d_3 n\epsilon_n^2)$$
for a certain constant $d_3>0.$
$\newline$
\textbf{B-(c).} A lower bound of $\pi\{\|\mathcal{B}-\hat{\mathcal{B}}\|_2 \le ((1+n)^{p} \sqrt{t_n})^{-1} C_{\mathbb{B}}\epsilon_n/2 | 
\mathcal{E}=\hat{\mathcal{E}}, T=t_n\}:$
\begin{align}
& \pi\left\{\mathcal{B}\in \mathbb{R}^{t_{n}} :  \Vert \hat{\mathcal{B}} - \mathcal{B} \Vert_{2} \leq  {C_{\mathbb{B}}\over (1+n)^{p}\sqrt{t_{n}} } \times {\epsilon_{n}\over 2} \right\}  \nonumber\\
&\geq 2^{-t_{n}}\left({t_{n}\over 2}\right)^{-{t_{n}\over 2}-1}\exp\left(-{\Vert \hat{\mathcal{B}}\Vert_{2}^{2}\over \sigma_{\beta}^{2}} - {(C_{\mathbb{B}}\epsilon_{n})^{2} \over 8(1+n)^{2p }t_{n}\sigma^{2}_{\beta}}\right)\left({(C_{\mathbb{B}}\epsilon_{n})^{2} \over 4(1+n)^{2p}t_{n}\sigma_{\beta}^{2} }\right)^{t_{n}\over 2}\label{rock_bdd} \\
&\geq 2^{-t_{n}}\left({t_{n}\over 2}\right)^{-{t_{n}\over 2}-1}\exp\left(-{t_{n}C^{2}\over \sigma_{\beta}^{2}} - {(C_{\mathbb{B}}\epsilon_{n})^{2} \over 8(1+n)^{2p }t_{n}\sigma_{\beta}^{2}}\right)\left({(C_{\mathbb{B}}\epsilon_{n})^{2} \over 4(1+n)^{2p}t_{n}\sigma_{\beta}^{2}}\right)^{t_{n}\over 2},\label{lower bound1}
\end{align}
where (\ref{rock_bdd}) is derived from equation (8.9) in (\cite{bforest}) and (\ref{lower bound1}) is derived from Theorem \ref{lemma5}. 
Finally, the three terms in (\ref{lower bound1}) are bounded below by
$$
2^{-t_{n}}\left({t_{n}\over 2}\right)^{-{t_{n}\over 2}-1} = 2^{-t_{n}}\exp\left(-\left(1+{t_{n}\over 2}\right)\log {t_{n}\over 2}\right) \\
\geq \exp(-d_{4}n\epsilon_{n}^{2})
$$
for some constant $d_{4} >0,$ 
\begin{align}
\exp\left(-{t_{n} C^{2}\over \sigma_{\beta}^{2}} - {(C_{\mathbb{B}}\epsilon_{n})^{2} \over 8(1+n)^{2p}t_{n}\sigma_{\beta}^{2}}\right) &\geq \exp(-d_{5}n\epsilon_{n}^{2})
\end{align}
for some constant $d_{5} >0$ and
$$
\left({(C_{\mathbb{B}}\epsilon_{n})^{2} \over 4(1+n)^{2p}t_{n}\sigma_{\beta}^{2}}\right)^{t_{n}\over 2} 
\geq \exp(-d_{6}n\epsilon_{n}^{2})
$$
for some constant $d_{6}>0$.
The proof is done by letting $d_{1} =\sum_{i=2}^{6}d_{i}$.
\qed

\subsection{Proof of Condition (\ref{eq3})}
\label{eq:conditio_3}
We will check Condition (\ref{eq3}) with the choice $C_{1}>(2d_{1}+2)/C_{*}$, where 
$d_{1}$ is the constant satisfying Condition (\ref{eq2}).
Note that we have
\begin{align*}
\mathcal{F} \backslash \mathcal{F}^{n} =\bigg\{ T > M_{n} \bigg\} \bigcup \bigg\{  \{T \leq M_{n} \} \cap \{ \exists t \in [T] \:\text{s.t}\: |\beta_{t} | > n \} \bigg\}.     
\end{align*}
Therefore, $\pi \{ \mathcal{F} \backslash \mathcal{F}^{n} \}$ is upper bounded by
\begin{align*}
\pi \{ \mathcal{F} \backslash \mathcal{F}^{n} \} &\leq \pi \{ T > M_{n} \} + \pi\{T \leq M_{n}\} \pi \{ \exists t \in [T]\:\text{s.t}\: |\beta_{t} | > n | T\leq M_{n}\} \\
&\leq \pi \{ T > M_{n} \} + \pi \{ \exists t \in [T]\:\text{s.t}\: |\beta_{t} | > n | T\leq M_{n}\}.
\end{align*}
{\bf Case 1. Upper bound of $\pi \{ T > M_{n} \}$.}
We will show that $\pi\{T>M_{n}\} e^{(2d_{1}+2)n\epsilon_{n}^{2}} \rightarrow 0$ 
as $n\rightarrow \infty,$
where $M_{n}=\lfloor C_{1}{n\epsilon_{n}^{2}\over \log n}  \rfloor,$ which holds because
\begin{align*}
\pi\{T>M_{n}\} &= { \sum_{t=M_{n}+1}^{T_{\max}}\exp(-C_{*}t\log n)  \over \sum_{t=0}^{T_{\max}}\exp(-C_{*}t\log n)} \\
&\leq {{1\over n^{(M_{n}+1)C_{*}}}\big(1-{1\over n^{(T_{\max}+1)C_{*}}} \big) \over 1 - {1\over n^{C_{*}}}}\times{ 1 -{1\over n^{C_{*}}} \over \big(1-{1\over n^{(T_{\max}+1)C_{*}}} \big) } \\
&= \exp(-(M_{n}+1)C_{*}\log n).
\end{align*}
Hence, $\pi\{T>M_{n}\} e^{(2d_{1}+2)n\epsilon_{n}^{2}} \rightarrow 0$ as 
$n\rightarrow \infty$. 
$\newline$
$\newline$
{\bf Case 2. Upper bound of $\pi \{ \exists t \in [T]\:\text{s.t}\: |\beta_{t} | > n | T\leq M_{n}\}$.}
We have
\begin{align*}
\pi \{ \exists i \in [T]\:\text{s.t}\: |\beta_{t} | > n | T\leq M_{n}\} &\leq M_{n}\pi\{ 
|\beta_{1}| > n \}  \\
&\leq 2M_{n}\exp\bigg(-{n^{2}\over 2\sigma_{\beta}^{2}}\bigg), 
\end{align*}
where $\beta_{1} \sim N(0,\sigma_{\beta}^{2})$ and $\sigma_{\beta}^{2} >0$ is a constant.
Hence, $\pi \{ \exists t \in [T]\:\text{s.t}\: |\beta_{t} | > n | T\leq M_{n}\}e^{(2d_{1}+2)n\epsilon_{n}^{2}} \xrightarrow{} 0$ as $n \xrightarrow{} \infty$.

\qed

\subsection{Verification of the conditions in Theorem 4 of \cite{nonp1}}
\label{app:final_condition_check}

For a given $f\in \mathcal{F}$,
consider the probability model $Y|\mathbf{X} \sim \mathbb{P}_{f(\mathbf{X}),1}$
and $\mathbf{X}\sim \mathbb{P}_{\mathbf{X}}$.
For given data $\mathbf{x}^{(n)}$ and $f$, let $p_{f(\mathbf{x}_{i})}$ be the density function of $\mathbb{P}_{f(\mathbf{x}_{i}),1}$ with respect to Lebesgue measure for $i=1,...,n$.
For given two densities $p_{f_1},p_{f_2}$, let $K(p_{f_1},p_{f_2})$ be a Kullback-Leibler (KL) divergence defined as $K(p_{f_1},p_{f_{2}})=\int \log(p_{f_{1}}(\bold{v}) / p_{f_{2}}(\bold{v}))p_{f_{1}}(\bold{v})d\bold{v},$ where $\bold{v}\in \mathcal{X}\times \mathbb{R}.$
Let $V(p_{f_{1}},p_{f_{2}})=\int | \log (p_{f_{1}}(\bold{v})/p_{f_{2}}(\bold{v})) - K(p_{f_{1}},p_{f_{2}}) |^{2} p_{f_{1}}(\bold{v})d\bold{v}.$
For $i=1,...,n$, let $p_{f_{1}(\mathbf{x}_{i})}$, $p_{f_{2}(\mathbf{x}_{i})}$ be densities and let $\rho_{n}^{2}(\mathbb{P}_{f_{1},1},\mathbb{P}_{f_{2},1}) = {1\over n} \sum_{i=1}^{n}\int (\sqrt{p_{f_{1}(\mathbf{x}_{i})}(y_{i})} - \sqrt{p_{f_{2}(\mathbf{x}_{i})}(y_{i})})^{2}dy_{i}$.
Let $\mathcal{F}_\xi=\{f\in \mathcal{F}: \|f\|_\infty \le \xi\}$ and $\mathcal{F}_{\xi}^{n}=\{f\in \mathcal{F}^{n}: \|f\|_\infty \le \xi\}$.
$\newline$
$\newline$
Note that by Theorem 7.6 in page 81 of \cite{ibragimov2013statistical}, for all sufficiently large $n$ 
and any $f \in \mathcal{F}_{\xi}$,
\begin{align}
C_{\rho,1}\Vert f_{0} - f \Vert_{2,n} 
\leq \rho_{n}(\mathbb{P}_{f_{0},1},\mathbb{P}_{f,1}) 
\leq C_{\rho,2}\Vert f_{0} - f \Vert_{2,n}.
\label{eq:hellinger_inquality}
\end{align}
holds,
where $C_{\rho,1}$ and $C_{\rho,2}$ are positive constants.
Therefore, to show the posterior concentration rate in (\ref{eq:cond_conv_rate}), following Theorem 4 of \cite{nonp1} and (\ref{eq:hellinger_inquality}), it suffices to verify that for all $f_{0}\in \mathcal{H}^{\alpha}_{0,F}$ and $\mathbf{x}^{(n)} \in A_{n}$, there exists a sieve $\mathcal{F}_{\xi}^{n}$ and a constant $d_{*}>0$ such that the following three conditions are satisfied.
\begin{align}
& \log N\left({\epsilon_{n} \over 36},\mathcal{F}^{n}_\xi ,\Vert \cdot \Vert_{2,n} \right) \leq n\epsilon_{n}^{2} \label{eq1-1}\\
&\pi_{\xi}\{\mathbb{B}_{n} \}\geq e^{-d_{*}n\epsilon_{n}^2} \label{eq2-1}
\\
&\pi_{\xi}\{\mathcal{F}_\xi\backslash \mathcal{F}^{n}_\xi\} \leq e^{-(d_{*}+2)n\epsilon_{n}^{2}} \label{eq3-1},
\end{align}
where 
$$
\mathbb{B}_{n} := \bigg\{ f \in \mathcal{F}_{\xi} : {1\over n }\sum_{i=1}^{n}K(p_{f_{0}(\mathbf{x}_{i})},p_{f(\mathbf{x}_{i})}) \leq \epsilon_{n}^{2} , {1\over n}\sum_{i=1}^{n}V(p_{f_{0}(\mathbf{x}_{i})},p_{f(\mathbf{x}_{i})}) \leq \epsilon_{n}^{2}  \bigg\}. 
$$

We will show that the above three conditions are satisfied to complete the proof of Theorem \ref{theorem_SIBART}.
$\newline$
$\newline$
{\bf Verification of Condition  (\ref{eq1-1}).}
Condition (\ref{eq1}) implies 
$$\log N\left({\epsilon_{n} \over 36},\mathcal{F}^{n}_\xi ,\Vert \cdot \Vert_{2,n} \right) \lesssim n\epsilon_{n}^{2}.$$
$\newline$
{\bf Verification of Condition  (\ref{eq2-1}).}
Direct calculation yields, 
\begin{align}
K(p_{f_{0}(\mathbf{x}_{i})},p_{f(\mathbf{x}_{i})}) &= \int \Big( (f_{0}(\mathbf{x}_{i}) - f(\mathbf{x}_{i}))y - A(f_{0}(\mathbf{x}_{i})) + A(f(\mathbf{x}_{i})) \Big)p_{f_{0}(\mathbf{x}_{i})}(y)dy \\
&= \bigg( (f_{0}(\mathbf{x}_{i}) - f(\mathbf{x}_{i}))\mathbb{E}[Y_{i}|\mathbf{x}_{i}] - A(f_{0}(\mathbf{x}_{i})) + A(f(\mathbf{x}_{i})) \bigg) \\
&= \bigg((f_{0}(\mathbf{x}_{i}) - f(\mathbf{x}_{i}))\dot{A}(f_{0}(\mathbf{x}_{i})) - A(f_{0}(\mathbf{x}_{i})) + A(f(\mathbf{x}_{i})) \bigg)
\end{align}
for $i=1,...,n$, and
\begin{align}
V(p_{f_{0}(\mathbf{x}_{i})},p_{f(\mathbf{x}_{i})}) &= \int  (f_{0}(\mathbf{x}_{i}) - f(\mathbf{x}_{i}))^{2} ( y - \dot{A}(f_{0}(\mathbf{x}_{i})) )^{2} p_{f_{0}(\mathbf{x}_{i})}(y)dy \\
&= (f_{0}(\mathbf{x}_{i}) -  f(\mathbf{x}_{i}))^{2}Var(Y_{i}|\mathbf{x}_{i}) \\
&= (f_{0}(\mathbf{x}_{i}) - f(\mathbf{x}_{i}))^{2}\ddot{A}(f_{0}(\mathbf{x}_{i})).
\end{align}
In turn, using Talyor expansion, we have
\begin{align*}
K(p_{f_{0}(\mathbf{x}_{i})},p_{f(\mathbf{x}_{i})}) = {1\over 2}\ddot{A}(\Tilde{x})(f_{0}(\mathbf{x}_{i}) - f(\mathbf{x}_{i}))^{2},
\end{align*}
where $\Tilde{x} \in [-2^{p}F,2^{p}F]$.

That is, we have
\begin{align*}
&\max\bigg\{{1\over n}\sum_{i=1}^{n}K(p_{f_{0}(\mathbf{x}_{i})},p_{f(\mathbf{x}_{i})}),{1\over n}\sum_{i=1}^{n}V(p_{f_{0}(\mathbf{x}_{i})},p_{f(\mathbf{x}_{i})})\bigg\}\leq C_{A}\Vert f_{0}-f \Vert_{2,n}^{2}.
\end{align*}
With a positive constant $C_{\mathbb{B}} = \sqrt{C_{A}}$, we have
\begin{align*}
\mathbb{B}_{n} \supseteq \{f \in \mathcal{F}_{\xi} : \Vert f - f_{0} \Vert_{2,n} \leq C_{\mathbb{B}}\epsilon_{n} \}
\end{align*}
and
\begin{align*}
&\pi_\xi\{f \in \mathcal{F}_{\xi} : \Vert f - f_{0} \Vert_{2,n} \leq C_{\mathbb{B}}\epsilon_{n} \} \geq \pi_\xi\{f \in \mathcal{F}_{\xi}:\|f-f_{0}\|_{\infty} \le C_{\mathbb{B}}\epsilon_n  \}.
\end{align*}
In turn, since $\xi \geq 2^{p}F+C_{\mathbb{B}}\epsilon_n$ for sufficiently large $n$, we have
\begin{align*}
\pi_\xi\{f \in \mathcal{F}_{\xi}:\|f-f_{0}\|_{\infty} \le C_{\mathbb{B}}\epsilon_n \} \geq \pi\{f \in \mathcal{F}:\|f-f_{0}\|_{\infty} \le C_{\mathbb{B}}\epsilon_n \}.
\end{align*}
Thus, Condition (\ref{eq2}) implies
$$\pi_\xi\bigg\{ f \in \mathcal{F}_{\xi} : {1\over n }\sum_{i=1}^{n}K(p_{f_{0}(\mathbf{x}_{i})},p_{f(\mathbf{x}_{i})}) \leq \epsilon_{n}^{2} , {1\over n}\sum_{i=1}^{n}V(p_{f_{0}(\mathbf{x}_{i})},p_{f(\mathbf{x}_{i})}) \leq \epsilon_{n}^{2}  \bigg\}   \geq e^{-d_{1} n \epsilon_n^2}.$$
$\newline$
{\bf Verification of Condition  (\ref{eq3-1}).}
Note that
$\pi_\xi\{\mathcal{F}_\xi\backslash \mathcal{F}^{n}_\xi\} \le
 \pi_\xi\{\mathcal{F}\backslash \mathcal{F}^n\}.$ 
 In turn,
 $$\pi_\xi\{\mathcal{F}\backslash \mathcal{F}^n\} \le \frac{\pi\{\mathcal{F}\backslash \mathcal{F}^n\}}{\pi\{\Vert f\Vert_{\infty} \le \xi\}}
 \le  \frac{\pi\{\mathcal{F}\backslash \mathcal{F}^n\}}{\pi\{\|f-f_0\|_\infty \le \epsilon_n\}},
 $$ 
 which is less than $e^{-(d_{1}+2) n\epsilon_n^2}$ by Condition (\ref{eq2}) and Condition (\ref{eq3}) whenever $2^{p}F+C_{\mathbb{B}}\epsilon_n\le \xi$.
The proof is completed by letting $d_{*}=d_{1}$.
\qed

\newpage

\section{Bridging the empirical and populational identifiabilities for multinary-product trees}
\label{app:multi_binary_approximation}
\renewcommand{\theequation}{C.\arabic{equation}}
\renewcommand{\thefigure}{C.\arabic{figure}}

For any component $S \in \mathbb{S} $, we first define a multinary-product tree, which is an extension of a binary-product tree, as follows.
Recall that $\mathcal{X}_j=[0,1]$ for all $j\in [p].$
A partition $\mathcal{P}_j$ of $\mathcal{X}_j$ is called an interval partition if
any element in $\mathcal{P}_j$ is an interval. For example,
$\{[0,1/3),[1/3,2/3),[2/3,1]\}$ is an interval partition.
In the followings, we only consider interval partitions and we simply write them as `partitions' unless
there is any confusion.
Let $\mathcal{P}_j$ be a given partition of $\mathcal{X}_j$ for each $j\in S,$ where $\phi_{j} :=|\mathcal{P}_j| \ge 2.$
Let $\mathcal{P}=\prod_{j\in S} \mathcal{P}_j.$ 
Then, the multinary-product tree defined on
the product partition $\mathcal{P}$ with the height vector
$\pmb{\gamma}=(\gamma_{\pmb{\ell}}, \pmb{\ell}\in \prod_{j\in S} \{1,\ldots,\phi_j\})$ is defined as
$$f_{S,\mathcal{P},\gamma}(\mathbf{x})=\sum_{\pmb{\ell}} \gamma_{\pmb{\ell}} \mathbb{I}(\mathbf{x}_S\in I_{\pmb{\ell}}),$$
where $I_{\pmb{\ell}}=\prod_{j\in S} I_{j,\ell_j}, I_{j,\ell_j}\in \mathcal{P}_j$.
Note that a binary-product tree is a special case of a multinary-product tree with $\phi_{j}=2$ for all $j\in S.$
Figure \ref{fig:compare_bi_multi} compares binary-product tree and multinary-product tree.

\begin{figure}[h]
    \centering
    \includegraphics[width=0.8\linewidth]{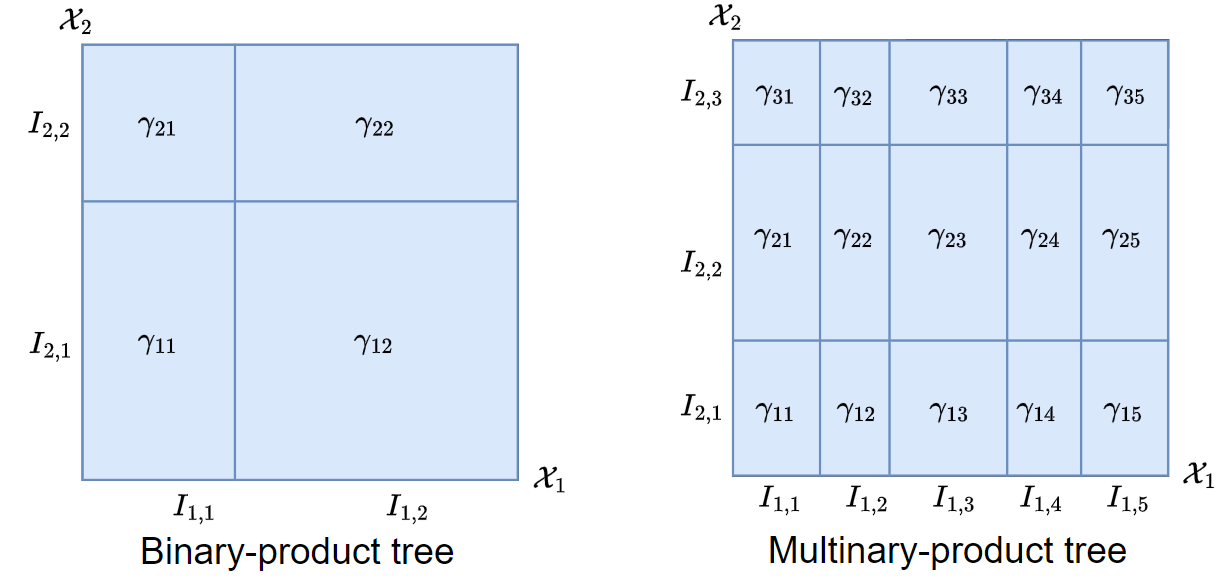}
    \caption{Examples of binary-product tree and multinary-product tree.
    The left panel is a binary-product tree, while the right panel illustrates the partition and cell heights of a multinary-product tree with $\phi_{1}=5$ and $\phi_{2}=3$.}
\label{fig:compare_bi_multi}
\end{figure}

\subsection{Identifiable transformation of multinary-product tree}
\label{app:transformation}

For a given probability measure $\nu$ on $[0,1]^{|S|},$
we explain how we transform any multinary-product tree
to a $\nu$-identifiable multinary-product tree. 
Let $f_{S}$ be a multinary-product tree $f_{S}$ defined as
\begin{align*}
f_{S}(\mathbf{x}_{S}) = \sum_{\pmb{\ell}}\gamma_{\pmb{\ell}}\mathbb{I}(\mathbf{x}_{S} \in I_{\pmb{\ell}}).
\end{align*}
We modify $\gamma_{\pmb{\ell}}$ so that the resulting function satisfies the $\nu$-identifiability condition.
That is, the resulting function is a $\nu$-identifiable multinary-product tree.
$\newline$
For $W\subseteq S,$ let $\pmb{\ell}_W=(\ell_j,j\in W)$ and let $W^c=S-W.$
For the modification of  $\gamma_{\pmb{\ell}},$ we let
$$\tilde{\gamma}_{\pmb{\ell}}= \gamma_{\pmb{\ell}}+
\sum_{k=1}^{|S|} (-1)^{k} \left( \sum_{W: |W|=k, W \subseteq S} \mathbb{E}_{\nu,W}(\gamma_{\pmb{\ell}})\right),
$$
where
\begin{align}
\mathbb{E}_{\nu,W}(\gamma_{\pmb{\ell}}) := \sum_{\pmb{\ell}_{W}\in \prod_{t \in W}[\phi_{t}]} 
\gamma_{(\pmb{\ell}_{W^c},\pmb{\ell}_{W})} \nu_{W}^{\text{ind}}\{I_{\pmb{\ell}_{W}}\}. \label{eq:marignal_expect_def}
\end{align}
Here, $\nu_{W}^{\text{ind}} = \prod_{j \in W}\nu_{j}$ and $I_{\pmb{\ell}_{W}} = \Pi_{j \in W}I_{j,\ell_{j}}$.
Note that $\nu_{j}$ is the marginal probability measure for $j\in S$.
When $W=\{j\},$ we write $\mathbb{E}_{\nu,j}$ instead of $\mathbb{E}_{\nu,\{j\}}$ for notational simplicity.
$\newline$
$\newline$
For given $j\in S,$ we can rewrite $\tilde{\gamma}_{\pmb{\ell}}$ as
\begin{align*}
  \tilde{\gamma}_{\pmb{\ell}} =& \gamma_{\pmb{\ell}}+
\sum_{k=1}^{|S|} (-1)^k \left( \sum_{W: |W|=k, j\in W , W \subseteq S} \mathbb{E}_{\nu.W}(\gamma_{\pmb{\ell}})\right) \\
& \quad + \sum_{k=1}^{|S|-1} (-1)^k \left( \sum_{W: |W|=k, j\not\in W , W \subseteq S} \mathbb{E}_{\nu,W}(\gamma_{\pmb{\ell}})\right)\\
=& \gamma_{\pmb{\ell}} - \mathbb{E}_{\nu,j}(\gamma_{\pmb{\ell}}) + \sum_{k=2}^{|S|} (-1)^k \left( \sum_{W: |W|=k, j\in W , W \subseteq S} \mathbb{E}_{\nu,W}(\gamma_{\pmb{\ell}})\right) \\
& \quad  + \sum_{k=1}^{|S|-1} (-1)^k \left( \sum_{W: |W|=k, j\not\in W , W \subseteq S} \mathbb{E}_{\nu,W}(\gamma_{\pmb{\ell}})\right)
\end{align*}
Thus, we have
\begin{align*}
\mathbb{E}_{\nu,j}(\tilde{\gamma}_{\pmb{\ell}})
&= \sum_{k=2}^{|S|} (-1)^k \left( \sum_{W: |W|=k, j\in W , W \subseteq S} \mathbb{E}_{\nu,W}(\gamma_{\pmb{\ell}})\right) \\
& \quad + \sum_{k=1}^{|S|-1} (-1)^k \left( \sum_{W: |W|=k , j \notin W, W \subseteq S} \mathbb{E}_{\nu,W\cup \{j \}}(\gamma_{\pmb{\ell}})\right) \\
&= \sum_{k=1}^{|S|-1} (-1)^k \left( \sum_{W: |W|=k, j \notin W , W \subseteq S} (-\mathbb{E}_{\nu,W\cup \{j\}}(\gamma_{\pmb{\ell}}) +\mathbb{E}_{\nu,W\cup \{j\}}(\gamma_{\pmb{\ell}})) \right) \\
&= 0
\end{align*}
for $j \in S$.
Therefore, $\tilde{\gamma}_{\pmb{\ell}}$ satisfies the $\nu$-identifiability condition.
$\newline$
We denote the resulting function by $f_{\nu,f_{S}}$, i.e.,
\begin{align}
f_{\nu,f_{S}}(\mathbf{x}_{S}) = \sum_{\pmb{\ell}}\tilde{\gamma}_{\pmb{\ell}}\mathbb{I}(\mathbf{x}_{S} \in I_{\pmb{\ell}}).
\label{eq:identifiable-trans}
\end{align}

\subsection{Approximation error between the empirically and populationally identifiable multinary-product trees}

We first investigate the approximation error between a given empirically identifiable multinary-product tree
and its populationally identifiable version obtained by the formula (\ref{eq:identifiable-trans}), whose result is given in the following theorem.
For $S \in \mathbb{S} $, let $\mathbb{P}_{S}$ be the distribution of $\mathbf{X}_S$ and $\mathbb{P}_{S}^{\text{ind}}=\prod_{j \in S}\mathbb{P}_{j}$, where $\mathbb{P}_{j}$ is the probability distribution of $\mathbf{X}_{j}$.

\begin{theorem}
\label{thm:multi_po_em_approx}
Let $\mathcal{F}^{E}_{S,K,\xi}$ be the set of all empirically identifiable $S$-component multinary-product trees $f_S$
with $\max_{j\in S}|\mathcal{P}_{j}|\leq K$ and $\Vert f_{S} \Vert_{\infty} \leq \xi.$
Then, we have
\begin{align*}
\mathbb{P}_{\mathbf{X}}^{n} \bigg\{ \sup_{f_{S} \in \mathcal{F}^{E}_{S,K,\xi}}
\Vert f_{S} - f_{\mathbb{P}_{S},f_{S}} \Vert_{\infty} \leq (24\xi+1)2^{|S|}\sqrt{K\log n \over n} 
\bigg \} \\
\geq 1- 2^{|S|+1}|S|K^{|S|}\exp\bigg(-{K\log n \over 4\xi^{2}}\bigg).
\end{align*}
\end{theorem}

\begin{proof}
Consider an identifiable multinary-product tree $f_{S}$ given as
\begin{align*}
f_{S}(\mathbf{x}_{S}) = \sum_{\pmb{\ell}}\gamma_{\pmb{\ell}}\mathbb{I}(\mathbf{x}_{S} \in I_{\pmb{\ell}}).
\end{align*}
Note that $\gamma_{\pmb{\ell}}$s depend on the data $\mathbf{X}^{(n)}$ because of the identifiability condition. 

We will show that
$\mathbb{E}_{\mathbb{P}_{S},W}(\gamma_{\pmb{\ell}})$ is small for all $W\subseteq S.$
Let $\phi_{j} = |\mathcal{P}_{j}|$ for $j \in S$.
Note that for all $W \subseteq S$, we have
\begin{align*}
\bigg | \mathbb{E}_{\mathbb{P}_{S},W}(\gamma_{\pmb{\ell}})  \bigg |  &= \bigg| \sum_{\pmb{\ell}_{W} \in \prod_{t \in W}[\phi_{t}]} \gamma_{(\pmb{\ell}_{W^{c}},\pmb{\ell}_{W})}\mathbb{P}_{W}^{\text{ind}}\{I_{\pmb{\ell}_{W}}\} \bigg| \\
&\leq  \sum_{\pmb{\ell}_{W\backslash\{j\}}\in \prod_{t \in W\backslash \{j\}}[\phi_{t}]}  \bigg |\sum_{\ell_{j} \in [\phi_{j}]}\gamma_{(\pmb{\ell}_{j^{c}},\ell_{j})}\prod_{i \in W}\mathbb{P}_{i}\{I_{i,\ell_{i}}\}  \bigg |\\
&= \sum_{\pmb{\ell}_{W\backslash\{j\}}\in \prod_{t \in W\backslash\{j\}}[\phi_{t}]}\bigg(\prod_{i \in W\backslash \{j\}}\mathbb{P}_{i}\{I_{i,\ell_{i}}\}\bigg) \bigg |\sum_{\ell_{j} \in [\phi_{j}]}\gamma_{(\pmb{\ell}_{j^{c}},\ell_{j})}\mathbb{P}_{j}\{I_{j,\ell_{j}}\}  \bigg |\\
&\leq \max_{ \pmb{\ell}_{j^{c}}\in \prod_{t \in S\backslash\{j\}}[\phi_{t}] } \bigg | \sum_{\ell_{j} \in [\phi_{j}] } \gamma_{(\pmb{\ell}_{j^{c}},\ell_{j})}\mathbb{P}_{j}\{I_{j,\ell_{j}}\} \bigg | \\
&=\max_{ \pmb{\ell}_{j^c}\in \prod_{t \in S\backslash\{j\}}[\phi_{t}] }|\mathbb{E}_{\mathbb{P}_{S},j}(\gamma_{\pmb{\ell}})|.
\end{align*}
In turn, since
$\mathbb{E}_{\mu_{n,S},j}(\gamma_{\pmb{\ell}})=0$
for any $j \in W,$
we have
\begin{align}
\mathbb{E}_{\mathbb{P}_{S},j}(\gamma_{\pmb{\ell}}) &= \mathbb{E}_{\mathbb{P}_{S},j}(\gamma_{\pmb{\ell}})  -  \mathbb{E}_{\mu_{n,S},j}(\gamma_{\pmb{\ell}}) \nonumber\\
&= \sum_{\ell_{j}\in [\phi_{j}]}\gamma_{(\pmb{\ell}_{j^{c}},\ell_{j})} (  \mathbb{P}_{j}\{I_{j,\ell_{j}}\} - \mu_{n,j}\{I_{j,\ell_{j}})\} \nonumber\\
&= \sum_{\ell_{j}\in [\phi_{j}]}\gamma_{(\pmb{\ell}_{j^{c}},\ell_{j})} \bigg(  \mathbb{P}_{j}\{I_{j,\ell_{j}}\} - {1\over n}\sum_{i=1}^{n}\mathbb{I}(X_{i,j} \in I_{j,\ell_{j}}) \bigg). \label{eq:diff_ej}
%&= \mathbb{E}_{j}[q(X_{j})] - {1\over n}\sum_{i=1}^{n}q(X_{i,j}),\nonumber
\end{align}
In the following, we will show that (\ref{eq:diff_ej}) is small.

For a given interval partition $\mathcal{P}=\{I_1,\ldots,I_{\phi_{j}}\}$ of $[0,1]$ with $\phi_{j}\le K$
and a vector $\pmb{\gamma}\in [-\xi,\xi]^{\phi_{j}},$ define a function $q_{\mathcal{P},\pmb{\gamma}}$ as
$q_{\mathcal{P},\pmb{\gamma}}(x)=\sum_{\ell=1}^{\phi_{j}} \gamma_{\ell} \mathbb{I}(x\in I_{\ell}).$
%$$g_{\mathcal{P},\mathbf{\gamma}}(x)=\sum_{l=1}^{|\mathcal{P}|} \gamma_l \bigg(  \mathbb{P}_{j}\{I_{l}\} - {1\over n}\sum_{i=1}^{n}\mathbb{I}(X_{i,j} \in I_{l}) \bigg).$$
The proof would be complete if we show 
\begin{align}
\mathbb{P}_{\mathbf{X}}^{n}\bigg\{ \sup_{q\in \mathcal{Q}_{K,\xi}}
\bigg | {1\over n}\sum_{i=1}^{n}q(X_{i,j}) -\mathbb{E}_{j}[q(X_{j})]  \bigg| > (24\xi+1)\sqrt{K\log n \over n}\bigg\} \leq 2\exp\bigg(-{K\log n \over 4\xi^{2}}\bigg), \label{eq:general_space}
\end{align}
where $\mathcal{Q}_{K,\xi}$ is the set of all $q_{\mathcal{P},\pmb{\gamma}}$
with $\phi_{j}\le K$ and $\pmb{\gamma}\in [-\xi,\xi]^{\phi_{j}},$ and $\mathbb{E}_{j}$ denotes the expectation under $\mathbb{P}_{j}$.
Here, the interval partitions $I_{\ell} \subseteq [0,1]$ for $q \in \mathcal{Q}_{K,\xi}$ are arbitrary subsets of $[0,1]$ that do not depend on $\mathbf{X}^{(n)}$.

For Rademacher random variables $\pmb{\vartheta}=\{\vartheta_{1},...,\vartheta_{n}\}$, we define the empirical and populational Rademacher complexities as
$$
\mathbf{R}_j(\mathbf{x}_{1},...,\mathbf{x}_{n}) = \mathbb{E}_{\pmb{\vartheta}}\bigg[\sup_{q\in \mathcal{Q}_{K,\xi}}\bigg|{1\over n}\sum_{i=1}^{n}\vartheta_{i}q(x_{i,j}) \bigg|\bigg]
\quad \text{and} \quad
\mathbf{R}_j(\mathcal{Q}_{K,\xi}) = \mathbb{E}^{n}_{\mathbf{X}}[\mathbf{R}_j(\mathbf{X}_{1},...,\mathbf{X}_{n})].
$$
For $k \in [n]$, let $\mathbf{X}^{(n),\text{new}}$ denote the data obtained by replacing the $k$th observation with a new one (independent with the data), i.e.,
\begin{align*}
\mathbf{X}^{(n),\text{new}} = (\mathbf{X}_{1},...,\mathbf{X}_{k-1},\mathbf{X}_{k}^{\text{new}},\mathbf{X}_{k+1},...,\mathbf{X}_{n}).
\end{align*}
Then, we have
\begin{align*}
\bigg |\sup_{q\in \mathcal{Q}_{K,\xi}} \triangle_{\mathbf{X}^{(n)}}(q) - \sup_{q \in \mathcal{Q}_{K,\xi}}\triangle_{\mathbf{X}^{(n),\text{new}}}(q) \bigg | 
&\leq \sup_{q \in \mathcal{Q}_{K,\xi}} \bigg | \triangle_{\mathbf{X}^{(n)}}(q) - \triangle_{\mathbf{X}^{(n),\text{new}}}(q) \bigg | \\
&\leq \sup_{q \in \mathcal{Q}_{K,\xi}} \bigg | {1\over n}\bigg(q(X_{k,j})-q(X_{k,j}^{\text{new}})\bigg) \bigg | \\
&\leq {2\xi \over n},
\end{align*}
where 
$$
\triangle_{\mathbf{X}^{(n)}}(q) = \bigg |{1\over n}\sum_{i=1}^{n}q(X_{i,j}) - \mathbb{E}_{j}[q(X_{j})] \bigg|.
$$ 
Therefore, using McDiarmid's inequality (\cite{10.5555/3153490}), we have
\begin{align*}
\mathbb{P}^{n}_{\mathbf{X}}\bigg\{ \bigg | \sup_{q\in \mathcal{Q}_{K,\xi}}\triangle_{\mathbf{X}^{(n)}}(q) - \mathbb{E}_{\mathbf{X}}^{n}\bigg[\sup_{q\in \mathcal{Q}_{K,\xi}}\triangle_{\mathbf{X}^{(n)}}(q)\bigg] \bigg | \geq \sqrt{K\log n \over n} \bigg \}\leq 2\exp\bigg(-{K\log n \over 4\xi^{2}}\bigg).
\end{align*}
Let $\mathcal{Q}_{K,\xi,\text{order}}$ denote the function class consisting of functions $q(\cdot)$ defined above, whose split values are determined using the order statistics of $\mathbf{X}^{(n)}$ as in (\ref{eq:split_set}).
Since $N(\epsilon,\mathcal{Q}_{K,\xi},\Vert \cdot \Vert_{2,n}) = N(\epsilon,\mathcal{Q}_{K,\xi,\text{order}},\Vert \cdot \Vert_{2,n})$, 
the covering number of $\mathcal{Q}_{K,\xi}$ is easily derived using a similar approach used in Section \ref{proof_eq1} of Supplementary Material: 
\begin{align}
N(\epsilon,\mathcal{Q}_{K,\xi},\Vert\cdot\Vert_{2,n}) &\leq n^{\phi_{j}}\bigg ( 1 + {2\phi_{j}\xi \over \epsilon}\bigg)^{\phi_{j}} \nonumber\\
&\leq n^{K}\bigg ( 1 + {2K\xi \over \epsilon}\bigg)^{K}, \label{eq:alpha_j}
\end{align}
where (\ref{eq:alpha_j}) is from $\phi_{j} \leq K$.
$\newline$
Therefore, using Dudley Theorem (Theorem 1.19 in \cite{wolf2018MSbook}), we have
\begin{align*}
\mathbf{R}(\mathbf{x}_{1},...,\mathbf{x}_{n}) &\leq \inf_{0\leq \epsilon \leq \xi/ 2}\bigg \{ 4\epsilon + {12\over \sqrt{n}}\int_{\epsilon}^{\xi}\sqrt{\log N(w,\mathcal{Q}_{K,\xi},\Vert \cdot \Vert_{2,n})}dw\bigg\} \\
&\lesssim \inf_{0\leq \epsilon \leq \xi /2}\bigg\{ 4\epsilon + 12(\xi-\epsilon)\sqrt{K\log n \over n} \bigg \} \\
&\leq 12\xi \sqrt{K\log n \over n}.
\end{align*}
That is, we have
\begin{align*}
\mathbf{R}(\mathcal{Q}_{K,\xi}) \leq 12\xi \sqrt{K\log n \over n}.
\end{align*}
Using Lemma \ref{eq: Rademacher_bound} in Section \ref{app:Rademacher bound} of Supplementary Material, we have
\begin{align*}
\mathbb{E}_{\mathbf{X}}^{n}\bigg[\sup_{q\in \mathcal{Q}_{K,\xi}}\triangle_{\mathbf{X}^{(n)}}(q)\bigg] &\leq 2\mathbf{R}(\mathcal{Q}_{K,\xi}) \\ 
&\leq 24\xi\sqrt{K\log n \over n}.
\end{align*}
Finally, we conclude that
\begin{align}
&\mathbb{P}_{\mathbf{X}}^{n}\bigg\{ \sup_{f_{S} \in \mathcal{F}_{S,K,\xi}^{E}}\Vert f_{S} - f_{\mathbb{P}_{S},f_{S}} \Vert_{\infty} \leq (24\xi+1)2^{|S|}\sqrt{K\log n\over n}\bigg\} \nonumber\\
&\geq \mathbb{P}_{\mathbf{X}}^{n}\bigg\{ \max_{ W\subseteq S }\max_{j \in W}\max_{\pmb{\ell}_{j^{c}} }\sup_{q\in \mathcal{Q}_{K,\xi}}  \triangle_{\mathbf{X}^{(n)}}(q) \leq (24\xi+1)\sqrt{K\log n\over n}\bigg\} \nonumber \\
&\geq 1 - \sum_{W \subseteq S}\sum_{j \in W}\sum_{\pmb{\ell}_{j^{c}}}\mathbb{P}_{\mathbf{X}}^{n}\bigg\{ \sup_{q\in \mathcal{Q}_{K,\xi}}  \triangle_{\mathbf{X}^{(n)}}(q) > (24\xi+1)\sqrt{K\log n\over n}\bigg\} \nonumber \\
&\geq 1- 2^{|S|+1}|S|K^{|S|}\exp\bigg(- {K\log n \over 4\xi^{2}} \bigg). \label{eq:hoff}
\end{align}
\end{proof}

The converse of  Theorem \ref{thm:multi_po_em_approx} is also true.
That is, Theorem \ref{thm:multi_po_em_approx2} proves that
the approximation error bettween a given populationally identifiable multinary-product tree
and its empirically identifiable version is the same as that of Theorem \ref{thm:multi_po_em_approx}.
The proof can be done by simply interchanging $\mu_{n,S}$ and $\mathbb{P}_{S}$ in the proof of Theorem \ref{thm:multi_po_em_approx} and so is omitted.

\begin{theorem}
\label{thm:multi_po_em_approx2}
Let $\mathcal{F}^{P}_{S,K,\xi}$ be the set of all populationally identifiable $S$-component multinary-product trees $f_S$ with $\max_{j\in S}|\mathcal{P}_{j}|\leq K$ and $\Vert f_{S} \Vert_{\infty} \leq \xi.$
Then, we have
\begin{align*}
\mathbb{P}_{\mathbf{X}}^{n}\bigg\{ \sup_{f_{S} \in \mathcal{F}^{P}_{S,K,\xi}}\Vert f_{S} - f_{\mu_{n,S},f_{S}} \Vert_{\infty} \leq (24\xi+1)2^{|S|}\sqrt{K\log n \over n} \bigg \} 
\geq 1- 2^{|S|+1}|S|K^{|S|}\exp\bigg(-{K\log n \over 4\xi^{2}}\bigg).
\end{align*}
\end{theorem}

\subsection{Rademacher complexity bound}
\label{app:Rademacher bound}
\begin{lemma}
\label{eq: Rademacher_bound}
For function class $\mathcal{Q}_{K,\xi}$, we have
$$
\mathbb{E}_{\mathbf{X}}^{n}\bigg[\sup_{q\in \mathcal{Q}_{K,\xi}}\triangle_{\mathbf{X}^{(n)}}(q)\bigg] \leq 2\mathbf{R}( \mathcal{Q}_{K,\xi} ).
$$
\end{lemma}
\begin{proof}
Let $\mathbf{X}'_{1},...,\mathbf{X}'_{n}$ be independent identical sample from the distribution $\mathbb{P}_{\mathbf{X}}$, where $\mathbf{X}'_{i} = (X'_{i,1},...,X'_{i,p})$.
Then, we have
\begin{align*}
&\mathbb{E}_{\mathbf{X}}^{n}\bigg[\sup_{q\in \mathcal{Q}_{K,\xi}}\triangle_{\mathbf{X}^{(n)}}(q)\bigg] \\
&=  \mathbb{E}_{\mathbf{X}}^{n}\bigg[\sup_{q\in \mathcal{Q}_{K,\xi}} \bigg |{1\over n}\sum_{i=1}^{n}q(X_{i,j}) - \mathbb{E}_{\mathbf{X}'_{1}}[q(X'_{1,j})] \bigg |\bigg] \\
&= \mathbb{E}_{\mathbf{X}}^{n}\bigg[\sup_{q\in \mathcal{Q}_{K,\xi}} \bigg |\mathbb{E}_{\mathbf{X}'}^{n}\bigg[{1\over n}\sum_{i=1}^{n}\bigg( q(X_{i,j}) - q(X'_{i,j})\bigg) \bigg] \bigg |\bigg] \\
&\leq \mathbb{E}_{\mathbf{X},\mathbf{X}'}^{n}\bigg[\sup_{q\in \mathcal{Q}_{K,\xi}} \bigg |{1\over n}\sum_{i=1}^{n}\bigg( q(X_{i,j}) - q(X'_{i,j})\bigg)  \bigg |\bigg]  
\end{align*}
Since the distribution of $( q(X_{i,j}) - q(X'_{i,j}))$ is identical to that of $\vartheta_{i}( q(X_{i,j}) - q(X'_{i,j}))$ for $i=1,...,n$, we have
\begin{align*}
&\mathbb{E}_{\mathbf{X},\mathbf{X}'}^{n}\bigg[\sup_{q\in \mathcal{Q}_{K,\xi}} \bigg |{1\over n}\sum_{i=1}^{n}\bigg( q(X_{i,j}) - q(X'_{i,j})\bigg)  \bigg |\bigg]  \\
&= \mathbb{E}_{\mathbf{X},\mathbf{X}',\pmb{\vartheta}}^{n}\bigg[\sup_{q\in \mathcal{Q}_{K,\xi}} \bigg |{1\over n}\sum_{i=1}^{n}\vartheta_{i}\bigg( q(X_{i,j}) - q(X'_{i,j})\bigg)  \bigg |\bigg]  \\
&\leq 2\mathbf{R}(\mathcal{Q}_{K,\xi}).
\end{align*}
\end{proof}

\newpage

\section{Proof of Theorem \ref{theorem_component}}
\renewcommand{\theequation}{D.\arabic{equation}}
\renewcommand{\thefigure}{D.\arabic{figure}}

The proof of Theorem \ref{theorem_component} consists of the following four steps.
$\newline$
\noindent {\bf (STEP 1).} We derive the posterior convergence rate with respect to the population $l_2$ norm. That is, we show that
\begin{align}
\label{eq:step1}
\pi_{\xi}\big\{ f \in \mathcal{F}_{\xi}^{n}: \Vert f- f_{0} \Vert_{2,\mathbb{P}_{\mathbf{X}}}  > B_{n}\epsilon_{n} \big| \mathbf{X}^{(n)},Y^{(n)}\big\} \xrightarrow{} 0,
\end{align}
for any $B_{n} \xrightarrow{} \infty$ in $\mathbb{P}_{0}^{n}$ as $n \xrightarrow{} \infty$.
$\newline$
$\newline$
\noindent {\bf (STEP 2).} From (\ref{eq:step1}), 
for any $S \in \mathbb{S}$ we establish that
\begin{align}
\label{eq:step2}
\pi_{\xi}\big\{ f \in \mathcal{F}_{\xi}^{n}: \Vert f_{S}-f_{0,S} \Vert_{2,\mathbb{P}_{\mathbf{X}}}  > B_{n}\epsilon_{n} \big| \mathbf{X}^{(n)},Y^{(n)}\big\} \xrightarrow{} 0,
\end{align}
for any $B_{n} \xrightarrow{} \infty$ in $\mathbb{P}_{0}^{n}$ as $n \xrightarrow{} \infty$.
$\newline$
$\newline$
\noindent {\bf (STEP 3).} We modify (\ref{eq:step2}) for the empirical $l_2$ norm. That is, we show that
\begin{align}
\label{eq:step3}
\pi_{\xi}\big\{ f \in \mathcal{F}_{\xi}^{n}: \Vert f_{S}-f_{0,S} \Vert_{2,n}  > B_{n}\epsilon_{n} \big| \mathbf{X}^{(n)},Y^{(n)}\big\} \xrightarrow{} 0,
\end{align}
for any $B_{n} \xrightarrow{} \infty$ in $\mathbb{P}_{0}^{n}$ as $n \xrightarrow{} \infty$.
$\newline$
$\newline$
\noindent {\bf (STEP 4).} Finally, we establish
\begin{align}
\pi_{\xi}\big\{ \mathcal{F}_{\xi} \backslash \mathcal{F}_{\xi}^{n} \big| \mathbf{X}^{(n)},Y^{(n)}\big\} \xrightarrow{} 0
\label{eq:step4}
\end{align}
as $n\rightarrow \infty.$

\subsection{Proof of (\ref{eq:step1})}

We rely on the following result (See Theorem 19.3 of \cite{gyorfi2006distribution} for its proof).

\begin{lemma}[Theorem 19.3 of \cite{gyorfi2006distribution}]\label{gyorfi2002distribution}
	Let $\boldsymbol{X}, \boldsymbol{X}_1, \dots, \boldsymbol{X}_n$ be independent and identically distributed random vectors with values in $\mathbb{R}^d$. Let $K_1, K_2 \geq 1$ be constants and let $\mathcal{G}$ be a class of functions $g : \mathbb{R}^d \to \mathbb{R}$ with
	\begin{align}
	|g(\boldsymbol{x})| \leq K_1, \quad \mathbb{E}[g(\boldsymbol{X})^2] \leq K_2 \mathbb{E}[g(\boldsymbol{X})].
    \label{eq:gyo_1}
	\end{align}
	Let $0<\kappa<1$ and $\zeta>0$. Assume that
	\begin{align*}
	\sqrt{n} \kappa \sqrt{1-\kappa} \sqrt{\zeta} \geq 288 \max \left\{2 K_{1}, \sqrt{2 K_{2}}\right\}
    %\label{eq:gyo_2}
	\end{align*}
	and that, for all $\mathbf{x}_1 , \dots , \mathbf{x}_n \in \mathbb{R}^d$ and for all $t \geq \frac{\zeta}{8}$,
	\begin{align}
	\frac{\sqrt{n} \kappa(1-\kappa) t}{96 \sqrt{2} \max \left\{K_{1}, 2 K_{2}\right\}} 
	\geq \int_{\frac{\kappa(1-\kappa)t}{16 \max \left\{K_{1}, 2 K_{2}\right\}}}^{\sqrt{t}}  
	\sqrt{\log N\left(u,\left\{g \in \mathcal{G}: \frac{1}{n} \sum_{i=1}^{n} g\left(\mathbf{x}_{i}\right)^{2} \leq 16 t\right\}, ||\cdot||_{1,n}\right)} d u.
	\label{eq:gyo_3}
    \end{align}
	Then,
	\begin{align*}
	\mathbb{P}^{n}_{\mathbf{X}}\left\{\sup _{g \in \mathcal{G}} \frac{\left|\mathbb{E}[g(\boldsymbol{X})]-\frac{1}{n} \sum_{i=1}^{n} g\left(\boldsymbol{X}_{i}\right)\right|}{\zeta +\mathbb{E}[g(\boldsymbol{X})]}>\kappa\right\} \nonumber 
	\leq 60 \exp \left(-\frac{n \zeta \kappa^{2}(1-\kappa)}{128 \cdot 2304 \max \left\{K_{1}^{2}, K_{2}\right\}}\right).
	\end{align*}
\end{lemma}

First, since $\mathcal{F}_{\xi}^{n}$ depends on the data $\mathbf{X}^{(n)}$, we cannot directly apply Lemma \ref{gyorfi2002distribution}.
To resolve this problem, as in (\ref{eq:general_space}), we consider a class of 
linear combinations of binary product trees not necessarily identifiable, with split values chosen as scalars in $[0,1]$ that do not depend on $\mathbf{X}^{(n)}$.
We denote this class by $\mathcal{F}_{\xi,\text{general}}^{n}$ and show that its covering number is of the same order as that of $\mathcal{F}_{\xi}^{n}$.

For given $S \subseteq [p],\: \mathbf{s}=(s_{j} \in [0,1], j \in S)$, and $\pmb{\beta}=(\beta_{1,j},\beta_{2,j}, j \in S)$, let $\mathbb{T}^{G}(\mathbf{x}_{S}:S,\mathbf{s},\pmb{\beta})$ be a binary-product tree given as 
\begin{align*}
\mathbb{T}^{G}(\mathbf{x}_{S}:S,\mathbf{s},\pmb{\beta}) = \prod_{j \in S}\bigg( \beta_{1,j}\mathbb{I}(x_{j}-s_{j} \leq 0) + \beta_{2,j}\mathbb{I}(x_{j}-s_{j} > 0) \bigg).
\end{align*}
Note that we do not impose the identifiability condition on $\pmb{\beta}.$
Now, we define $\mathcal{F}_{\text{general}}^{n}$ as 
$$
\mathcal{F}_{\text{general}}^{n}
= \Bigg\{ f : f(\mathbf{x}) = \sum_{t=1}^{T} 
   \mathbb{T}^{G}\big(\mathbf{x}_{S_{t}} : S_{t}, \mathbf{s}_{t}, \pmb{\beta}_{t}\big),
   \quad T \in [M_{n}], \: \mathbf{s}_{t} \in [0,1]^{|S_{t}|}, \max_t \Vert\pmb{\beta}_t\Vert_{\infty} \leq n \Bigg\},
$$
and 
$\mathcal{F}_{\xi,\text{general}}^{n} = \{f \in \mathcal{F}_{\text{general}}^{n}, \Vert f \Vert_{\infty} \leq \xi \}$.
Note that $\mathcal{F}^{n}_{\xi} \subseteq \mathcal{F}^{n}_{\xi,\text{general}}$.
By a similar approach used in Section \ref{proof_eq1} of Supplementary Material and (\ref{eq:alpha_j}), we can show
\begin{align}
N(\epsilon,\mathcal{F}_{\xi,\text{general}}^{n},\Vert \cdot \Vert_{1,n}) \lesssim M_{n}n^{M_{n}p}\bigg( 1+ {8np(n+1)^{p}M_{n} \over \epsilon}\bigg)^{4pM_{n}}. \label{eq:general_cov_num}
\end{align}
For $K_{1}=K_{2}=4\xi^{2}$, $\kappa: = {1\over 2}$, $\zeta = \epsilon_{n}^{2}$, and $\mathcal{G} = \{ g : g=(f_{0} -f)^{2} , f \in \mathcal{F}_{\xi,\text{general}}^{n} \} $, we first verify Condition (\ref{eq:gyo_1}) and Condition (\ref{eq:gyo_3}) in Lemma \ref{gyorfi2002distribution}.
$\newline$
$\newline$
{\bf Verifying Condition (\ref{eq:gyo_1}):}
Since $\Vert f \Vert_{\infty} \leq \xi$ for $f \in \mathcal{F}_{\xi,\text{general}}^{n}$, we have
$$\sup_{\bold{x}}g(\mathbf{x})=\sup_{\bold{x}} (f_{0}(\mathbf{x})-f(\mathbf{x}))^{2} \leq 4\xi^{2}$$
for all $g \in \mathcal{G}.$
$\newline$
$\newline$
{\bf Verifying Condition (\ref{eq:gyo_3}):}
Since
\begin{align*}
\forall f_{1},f_{2} \in \mathcal{F}_{\xi,\text{general}}^{n},\quad \Vert (f_{1}-f_{0})^{2} - (f_{2}-f_{0})^{2} \Vert_{1,n} \leq 4\xi\Vert f_{1}-f_{2}\Vert_{1,n},
\end{align*}
we have
\begin{align}
N(u,\mathcal{G},\Vert \cdot \Vert_{1,n}) 
&\leq N\bigg({u\over 4\xi},\mathcal{F}_{\xi,\text{general}}^{n},\Vert \cdot \Vert_{1,n}\bigg) \nonumber\\
&\lesssim {n\epsilon_{n}^{2}\over \log n}n^{n\epsilon_{n}^{2}p/\log n} \bigg( 1+ { 32\xi np(1+n)^{p}\epsilon_{n}^{2} \over u \log n} \bigg)^{n\epsilon_{n}^{2}/\log n}
\label{eq:covering_extension}
\end{align}
for any $u >0$, where (\ref{eq:covering_extension}) is derived from (\ref{eq:general_cov_num}).
Therefore, for $t \geq {\epsilon_{n}^{2} / 8},\: {t/4 \over 16\max(K_{1},2K_{2})} \leq u \leq \sqrt{t}$,
we have
\begin{align*}
\log N(u,\mathcal{G},\Vert \cdot \Vert_{1,n}) &\lesssim
\log {n\epsilon_{n}^{2}\over \log n} +
{n\epsilon_{n}^{2}\over \log n}p\log n + { n\epsilon_{n}^{2} \over \log n }\log \bigg ( 1 + {32\xi np(1+n)^{p}\epsilon_{n}^{2} \over u\log n }\bigg ) \\
&\lesssim n\epsilon_{n}^{2}. 
\end{align*}
Hence, for all $t \geq \epsilon_{n}^{2} / 8$, the following inequality holds:
\begin{align*}
\int_{{t/4 \over 16\max(K_{1},2K_{2})}}^{\sqrt{t}} \sqrt{\log N(u,\mathcal{G},\Vert \cdot \Vert_{1,n})} &\lesssim \sqrt{t}\sqrt{n}\epsilon_{n}  \\
&= o\left( \frac{\sqrt{n} t/4}{96 \sqrt{2} \max \left\{K_1, 2K_2 \right\}} \right).
\end{align*}
$\newline$
{\bf Proof of Condition (\ref{eq:step1}):}
From Lemma \ref{gyorfi2002distribution}, we obtain the following bound:
\begin{align*} 
&\mathbb{P}^{n}_{\mathbf{X}}\left\{\sup _{f \in \mathcal{F}_{\xi}^{n}} \frac{\left| ||f-f_{0}||_{2, \mathbb{P}_{\mathbf{X}}}^2  - ||f-f_{0}||_{2, n}^2 \right|}{\varepsilon_{n}^2+||f-f_{0}||_{2, \mathbb{P}_{\mathbf{X}}}^2} > \frac{1}{2}\right\} \\
&\leq \mathbb{P}^{n}_{\mathbf{X}}\left\{\sup _{f \in \mathcal{F}_{\xi,\text{general}}^{n}} \frac{\left| ||f-f_{0}||_{2, \mathbb{P}_{\mathbf{X}}}^2  - ||f-f_{0}||_{2, n}^2 \right|}{\varepsilon_{n}^2+||f-f_{0}||_{2, \mathbb{P}_{\mathbf{X}}}^2} > \frac{1}{2}\right\} \\
&\leq 60 \exp \left(-\frac{n \varepsilon_{n}^2 / 8}{128 \cdot 2304 \cdot 16\xi^4}\right).
\end{align*}
It implies that
$$\forall f \in \mathcal{F}_{\xi}^{n}, \quad 2\Vert f - f_{0} \Vert_{2,n}^{2} \geq \Vert f - f_{0} \Vert_{2,\mathbb{P}_{\mathbf{X}}}^{2} - \epsilon_{n}^{2}$$
with probability at least $1- 60 \exp \left(-\frac{n \varepsilon_{n}^2 / 8}{128 \cdot 2304 \cdot 16\xi^4}\right)$
under the probability distribution $\mathbb{P}^{n}_{\mathbf{X}}$.
We refer to this event as $\Omega_{n}^{*}$.
$\newline$
$\newline$
On the event $\Omega_{n}^{*}$, we have
\begin{align*}
&\pi_{\xi}\big\{ f \in \mathcal{F}_{\xi}^{n} : \Vert f - f_{0} \Vert_{2,\mathbb{P}_{\mathbf{X}}}  > B_{n}\epsilon_{n} \big|\mathbf{X}^{(n)},Y^{(n)}\big\} \\
&\leq \pi_{\xi}\big\{ f \in \mathcal{F}_{\xi}^{n} : \Vert f - f_{0} \Vert_{2,n}  > B_{n}\epsilon_{n}  \big| \mathbf{X}^{(n)},Y^{(n)}\big\} \\
&\xrightarrow{} 0 
\end{align*}
for any $B_{n} \xrightarrow{} \infty$ in $\mathbb{P}_{0}^{n}$ as $n \xrightarrow{} \infty$ by Theorem \ref{theorem_SIBART}.
Since $\mathbb{P}_{0}^{n}(\Omega_{n}^{*}) \xrightarrow{} 1$ as $n\xrightarrow{} \infty$, we have
$$ \pi_{\xi}\big\{ f \in \mathcal{F}_{\xi}^{n} : \Vert f - f_{0} \Vert_{2,\mathbb{P}_{\mathbf{X}}}  > B_{n}\epsilon_{n} \big| \mathbf{X}^{(n)},Y^{(n)}\big\} \xrightarrow{} 0$$
for any $B_{n} \xrightarrow{} \infty$ in $\mathbb{P}_{0}^{n}$ as $n \xrightarrow{} \infty,$
which completes the proof of (\ref{eq:step1}).
\qed

\subsection{Proof of (\ref{eq:step2})}
Let $C_{L}$ and $C_{U}$ be positive constants such that
\begin{align*}
C_{L} \leq \inf_{\bold{x}\in \mathcal{X}}\frac{p_{\mathbf{X}}(\bold{x})}{p^{\rm ind}_{\mathbf{X}}(\bold{x})}\le \sup_{\bold{x}\in \mathcal{X}}\frac{p_{\mathbf{X}}(\bold{x})}{p^{\rm ind}_{\mathbf{X}}(\bold{x})} \leq C_{U}.
\end{align*}

Any $f \in \mathcal{F}_{\xi}^{n}$ can be decomposed into the sum of identifiable multinary-product trees $f_{S}$s, i.e.,
\begin{align*}
f(\mathbf{x}) = \sum_{S \in \mathbb{S} }f_{S}(\mathbf{x}_{S}),
\end{align*}
where
$f_{S}$ is the ensemble of $T_{S}$ many identifiable binary-product trees.
Let $K_{S}= \prod_{j \in S}|\mathcal{P}_{j}|$, where 
$\mathcal{P}_{j}s$ are partitions in $f_{S}$.
Since $T_{S}\leq \sum_{S \in \mathbb{S} }T_{S} \leq M_{n}$, it follows that
\begin{align}
K_{S} &\leq (T_{S}^{1\over |S|}+1)^{|S|} \nonumber \leq {C_{1}n\epsilon_{n}^{2} / \log n}.
\end{align}
Therefore by Theorem \ref{thm:multi_po_em_approx}, we have
\begin{align*}
\mathbb{P}_{\mathbf{X}}^{n}\bigg\{ \sup_{f_{S} : f \in \mathcal{F}_{\xi}^{n}}\Vert f_{S} - f_{\mathbb{P}_{S},f_{S}} \Vert_{\infty} \lesssim \epsilon_{n} \bigg\} \geq 1 - 2^{|S|+1}|S|(C_{1}n\epsilon_{n}^{2}/\log n)^{|S|}\exp\bigg(-{C_{1}n\epsilon_{n}^{2}\over 4\xi^{2}}\bigg).
\end{align*}
For notational simplicity, we denote $f_{\mathbb{P}_{S},f_{S}}$ by $f_{S}^{P}$.
Let $f^{P}=\sum_{S\in \mathbb{S} }f_{S}^{P}$.
$\newline$
For $f^{P}$, we have
\begin{align}
\Vert f^{P} - f_{0} \Vert_{2,\mathbb{P}_{\mathbf{X}}}^{2}
&= \int_{\mathcal{X}}\bigg\{ \sum_{S \in \mathbb{S} } ( f_{S}^{P}(\mathbf{x}_{S}) - f_{0,S}(\mathbf{x}_{S})) \bigg\}^{2}\mathbb{P}_{\mathbf{X}}(d\mathbf{x})  \nonumber \\
&\geq {1\over C_{L}} \int_{\mathcal{X}}\bigg\{ \sum_{S \in \mathbb{S} } ( f_{S}^{P}(\mathbf{x}_{S}) - f_{0,S}(\mathbf{x}_{S})) \bigg\}^{2}\prod_{j=1}^{p}\mathbb{P}_{j}(dx_{j}) \nonumber \\
&= {1\over C_{L}} \sum_{S \in \mathbb{S} }\int_{\mathcal{X}}( f_{S}^{P}(\mathbf{x}_{S}) - f_{0,S}(\mathbf{x}_{S}))^{2}\prod_{j=1}^{p}\mathbb{P}_{j}(dx_{j}) \label{eq:id_cond_component}\\
&\geq {C_{U}\over C_{L}}\sum_{S \in \mathbb{S} }\Vert f_{S}^{P} - f_{0,S} \Vert_{2,\mathbb{P}_{\mathbf{X}}}^{2} \nonumber  \\
&\gtrsim \Vert f_{S}^{P} - f_{0,S} \Vert_{2,\mathbb{P}_{\mathbf{X}}}^{2} \nonumber ,
\end{align}
for all $S \in \mathbb{S}$, where the equality (\ref{eq:id_cond_component}) holds since $f_{S}^{P}$s satisfy the populational identifiability condition.
Thus, we obtain the following lower bound for $\Vert f - f_{0} \Vert_{2,\mathbb{P}_{\mathbf{X}}}$ : 
\begin{align*}
\Vert f - f_{0} \Vert_{2,\mathbb{P}_{\mathbf{X}}} &\geq \Vert f^{P}-f_{0}\Vert_{2,\mathbb{P}_{\mathbf{X}}} - \Vert f - f^{P}\Vert_{2,\mathbb{P}_{\mathbb{X}}}\\
& \gtrsim \Vert f^{P}-f_{0}\Vert_{2,\mathbb{P}_{\mathbf{X}}} - \epsilon_{n} \\
& \gtrsim \Vert f_{S}^{P} - f_{0,S} \Vert_{2,\mathbb{P}_{\mathbf{X}}}  - \epsilon_{n}\nonumber \\
& \gtrsim \Vert f_{S} - f_{0,S} \Vert_{2,\mathbb{P}_{\mathbf{X}}}  -2 \epsilon_{n}\nonumber
\end{align*}
To sum up, we conclude that
\begin{align*}
\pi_{\xi}\big\{ f \in \mathcal{F}_{\xi}^{n}: \Vert f_{S}-f_{0,S} \Vert_{2,\mathbb{P}_{\mathbf{X}}}  > B_{n}\epsilon_{n} \big| \mathbf{X}^{(n)},Y^{(n)}\big\} \xrightarrow{} 0,
\end{align*}
for any $B_{n} \xrightarrow{} \infty$ in $\mathbb{P}_{0}^{n}$ as $n \xrightarrow{} \infty$.
\qed

\subsection{Proof of (\ref{eq:step3})}
In the same manner as in (\ref{eq:step1}), using Lemma \ref{gyorfi2002distribution} with $\mathcal{G} = \{ g : g=(f_{0,S} -f_{S})^{2} , f \in \mathcal{F}_{\xi,\text{general}}^{n} \}$, we can obtain the following :
$$ \pi_{\xi}\big\{ f \in \mathcal{F}_{\xi}^{n} : \Vert f_{S} - f_{0,S} \Vert_{2,n}  > B_{n}\epsilon_{n} \big| \mathbf{X}^{(n)},Y^{(n)}\big\} \xrightarrow{} 0$$
for any $B_{n} \xrightarrow{} \infty$ in $\mathbb{P}_{0}^{n}$ as $n \xrightarrow{} \infty$.

\subsection{Proof of (\ref{eq:step4})}
For given $\mathbf{x}^{(n)} \in A_{n}$, where $A_{n}$ is defined in Section \ref{app:A_{n}} of Supplementary Material, we have
\begin{align*}
{\pi_{\xi}\{\mathcal{F}_{\xi} \backslash \mathcal{F}_{\xi}^{n}\} \over \pi_{\xi} \{\mathbb{B}_{n} \} } &\leq \exp \left (-2n\epsilon_{n}^{2}\right)
\end{align*}
since Condition (\ref{eq2-1}) and Condition (\ref{eq3-1}) hold.
Thus, Lemma 1 of \cite{nonp1} implies that for any $\mathbf{x}^{(n)} \in A_{n}$ and an arbitrary $\delta > 0$, 
\begin{align*}
\lim_{n \to \infty}\mathbb{P}_{Y^{(n)}}\bigg \{  \pi_{\xi}\{ \mathcal{F}_{\xi} \backslash \mathcal{F}_{\xi}^{n} | \mathbf{X}^{(n)},Y^{(n)}\} > \delta   \bigg| \mathbf{X}^{(n)} = \mathbf{x}^{(n)} \bigg \}=0.
\end{align*}
Since $\mathbb{P}_{\mathbf{X}}^{n}(A_{n})\xrightarrow{} 1$, we have
$$
\lim_{n \to \infty}\mathbb{P}_{0}^{n}\bigg \{ \pi_{\xi}\{ \mathcal{F}_{\xi} \backslash \mathcal{F}_{\xi}^{n} | \mathbf{X}^{(n)},Y^{(n)}\} > \delta \bigg \}=0,
$$
for any $\delta > 0$, which completes the proof of (\ref{eq:step4}).
\qed 

\newpage

\numberwithin{equation}{section}
\section{Proof of Theorem \ref{lemma5}}
\label{app:approximation}
\renewcommand{\theequation}{E.\arabic{equation}}
\renewcommand{\thefigure}{E.\arabic{figure}}
The proof consists of three steps.
In the first step, we approximate $f_{0,S}$ by a specially designed multinary-product tree called
{\it the equal probability-product tree}, which is populationally identifiable.
The second step is to approximate the equal probability-product tree
by an identifiable multinary-product tree using Theorem \ref{thm:multi_po_em_approx2}.
The third step is to transfer the identifiable multinary-product tree obtained in the second step into a sum of identifiable binary-product trees whose heights are bounded.

\subsection{Approximation of $f_{0,S}$ by the equal probability-product tree}
\label{app:approximation_eq_p}

We first approximate $f_{0,S}$ by a specially designed populationally identifiable multinary-product tree so called the EP-product (equal probability-product)  tree.
For a positive integer $r,$ let $q_{j,\ell}, \ell=1,\ldots,r$ be the
$\ell \cdot 100/r \%$ quantiles of $\mathbb{P}_{j},$ the distribution of $X_j.$
Let $\mathcal{P}_{r,j}^{\text{EP}}$ be the partition of $\mathcal{X}_j$ consisting of
$(q_{j,(\ell-1)}, q_{j,\ell}], \ell=1,\ldots,r$ with $q_{j,0}=0,$
which we call the equal-probability partition of size $r.$
Then, the EP-product tree with the parameters $S, \mathcal{P}_{r}^{\text{EP}}=\prod_{j\in S} \mathcal{P}^{\text{EP}}_{r,j}$ and
$\gamma=\{\gamma_{\pmb{\ell}}\in \mathbb{R}, \pmb{\ell}\in \{1,\ldots,r\}^{|S|}\}$ is defined as
$$f_{S,\mathcal{P}_{r}^{\text{EP}}, \gamma} (\mathbf{x}_{S})=
\sum_{\pmb{\ell}} \gamma_{\pmb{\ell}} \mathbb{I}(\mathbf{x}_S\in I_{\pmb{\ell}}).$$ 
That is, the EP-product tree is a multinary-product tree defined on the product partition
$\mathcal{P}_{r}^{\text{EP}}.$
The next lemma is about the approximation of a smooth function by an EP-product tree.

\begin{lemma}\label{le:approx}
Define
$$\gamma_{\pmb{\ell}}= 
\frac{1}{\mathbb{P}_{S}^{\text{ind}}(I_{\pmb{\ell}})}
\int_{I_{\pmb{\ell}}} f_{0,S}(\mathbf{x}_S) \mathbb{P}_{S}^{\text{ind}}(d\mathbf{x}_S)$$
for all $\pmb{\ell}\in [r]^{|S|}$. 
Then, $f_{S,\mathcal{P}_{r}^{\text{EP}}, \gamma}$ is populationally identifiable and it satisfies the following error bound:
$$\sup_{\mathbf{x}_{S} \in \mathcal{X}_S} |f_{0,S}(\mathbf{x}_{S}) - f_{S,\mathcal{P}_{r}^{\text{EP}}, \gamma} (\mathbf{x}_{S})|
 \le \|f_{0,S}\|_{\mathcal{H}^{\alpha}} \left(\frac{\sqrt{|S|}}{r p_L}\right)^{\alpha},$$
where $p_L=\inf_{\mathbf{x}\in \mathcal{X}} p_{\mathbf{X}}(\mathbf{x}).$ 
\end{lemma}
\medskip

\begin{proof}
For $\mathbf{x}_{S}\in I_{\pmb{\ell}},$
\begin{align*}
 & |f_{0,S}(\mathbf{x}_{S})-f_{S,\mathcal{P}_{r}^{EP},\gamma}(\mathbf{x}_{S})|\\
=& \left| f_{0,S}(\mathbf{x}_{S}) - \frac{1}{\mathbb{P}_{S}^{\text{ind}}(I_{\pmb{\ell}})} \int_{I_{\pmb{\ell}}} f_{0,S}(\mathbf{x}'_{S}) \mathbb{P}_{S}^{\text{ind}}(d\mathbf{x}'_{S})\right|\\
\leq&   \frac{1}{\mathbb{P}_{S}^{\text{ind}}(I_{\pmb{\ell}})} \int_{I_{\pmb{\ell}}} |f_{0,S}(\mathbf{x}_{S})-f_{0,S}(\mathbf{x}'_{S}) |\mathbb{P}_{S}^{\text{ind}}(d\mathbf{x}'_{S})\\ 
 \leq& \frac{1}{\mathbb{P}_{S}^{\text{ind}}(I_{\pmb{\ell}})} \|f_{0,S}\|_{\mathcal{H}^{\alpha}} \int_{I_{\pmb{\ell}}} \|\mathbf{x}_{S}-\mathbf{x}'_{S}\|_2^{\alpha} \mathbb{P}_{S}^{\text{ind}}(d\mathbf{x}'_{S})\\
\leq& \|f_{0,S}\|_{\mathcal{H}^{\alpha}} \sup_{\mathbf{x}'_{S}\in I_{\pmb{\ell}}} \|\mathbf{x}_{S}-\mathbf{x}'_{S}\|_2^{\alpha},
\end{align*}
where $\Vert \cdot \Vert_{2}$ is the Euclidean norm for a vector, i.e., for a given vector $\mathbf{e}=(e_{1},...,e_{n})$, $\Vert \mathbf{e} \Vert_{2} = \sqrt{\sum_{i=1}^{n}e_{i}^{2}}.$
Since $I_{\pmb{\ell}}$ is a $|S|$-dimensional hyper-cube whose side lengths are less than $1/(r p_L),$ we have
$\|\mathbf{x}_{S}-\mathbf{x}'_{S}\|_2 \le \sqrt{|S|}/(r p_L),$
which completes the proof of the approximation error bound.

For the populational identifiability condition, let $\mathbf{x}_{-j}=(x_\ell, \ell\in S, \ell \ne j)$ for a given $\mathbf{x}_{S}$ and let
$\pmb{\ell}_{-j}$ be the index in $\{1,\ldots,r\}^{|S|-1}$ such that $\mathbf{x}_{-j}\in I_{\pmb{\ell}_{-j}}.$
We will show $$
\int_{\mathcal{X}_{j}} f_{S,\mathcal{P}_{r}^{EP},\gamma}(\mathbf{x}_{S}) \mathbb{P}_{j}(dx_{j})=0
$$
for all $\mathbf{x}_{-j}\in \mathcal{X}_{-j}$ and $j\in S.$
For a given $k\in \{1,\ldots,r\},$ let $\pmb{\ell}(k) \in \{1,\ldots,r\}^{|S|}$ be an index defined as
$\pmb{\ell}(k)_{-j}=\pmb{\ell}_{-j}$ and $\pmb{\ell}(k)_j=k.$ 
Then, 
\begin{align*}
  &   \int_{\mathcal{X}_{j}}f_{S,\mathcal{P}_{r}^{EP},\gamma}(\mathbf{x}_{S}) \mathbb{P}_{j}(dx_{j})\\
=& \sum_{k=1}^r \frac{1}{\mathbb{P}_{S}^{\text{ind}}(I_{\pmb{\ell}(k)})} \int_{I_{\pmb{\ell}(k)}} 
f_{0,S}(\mathbf{x}'_{S}) \mathbb{P}_{S}^{\text{ind}}(d\mathbf{x}'_{S})\mathbb{P}_j(I_{j,k}) \\
= & \frac{1}{\mathbb{P}_{S\backslash \{j\} }^{\text{ind}}(I_{\pmb{\ell}_{-j}})} \sum_{k=1}^r \int_{I_{\pmb{\ell}(k)}} 
f_{0,S}(\mathbf{x}'_{S})\mathbb{P}_S^{\text{ind}}(d\mathbf{x}'_{S})\\
=& \frac{1}{\mathbb{P}_{S\backslash \{j\} }^{\text{ind}}(I_{\pmb{\ell}_{-j}})} \int_{I_{\pmb{\ell}_{-j}}}\int_{\mathcal{X}_j}
f_{0,S}(\mathbf{x}'_{S}) \mathbb{P}_j(dx'_j) \mathbb{P}_{S\backslash \{j\} }^{\text{ind}}(d\mathbf{x}'_{-j})\\
=& 0,
\end{align*}
since $\int_{\mathcal{X}_j}
f_{0,S}(\mathbf{x}'_{S}) \mathbb{P}_j(dx'_j)=0$ for all $\mathbf{x}'_{-j}$ by the populational identifiability of $f_{0,S}.$
\end{proof}

\subsection{Approximation of the EP-product tree by an identifiable multinary-product tree}
\label{app:approximation_pop_EQ}

Since the EP-product tree $f_{S,\mathcal{P}_{r}^{EP},\gamma}$ in Section \ref{app:approximation_eq_p} is a populationally identifiable multinary-product tree, we can apply Theorem \ref{thm:multi_po_em_approx2} to obtain an identifiable multinary-product tree $f_{S,\mathcal{P}_{r}^{EP},\hat{\gamma}}$ such that
\begin{align*}
\sup_{\mathbf{x}_{S}}|f_{S,\mathcal{P}_{r}^{EP},\gamma}(\mathbf{x}_{S}) - f_{S,\mathcal{P}_{r}^{EP},\hat{\gamma}}(\mathbf{x}_{S}) | \leq (24F+1)2^{|S|}\sqrt{r^{|S|}\log n \over n}
\end{align*}
with probability at least $1- 2^{|S|+1}|S|r^{|S|^{2}}\exp\left(-{r^{|S|}\log n \over 4F^{2}}\right)$ under $\mathbb{P}_{\mathbf{X}}^{n}$.

To sum up, we have shown that the modified EP tree $f_{S,\mathcal{P}_{r}^{EP},\hat{\gamma}}$ approximates $f_{0,S}$ well in the sense that
\begin{align*}
    \sup_{\mathbf{x}}|f_{0,S}(\mathbf{x})-f_{S,\mathcal{P}_{r}^{EP},\hat{\gamma}}(\mathbf{x})| &\leq \sup_{\mathbf{x}} |f_{0,S}(\mathbf{x})-f_{S,\mathcal{P}_{r}^{EP},\gamma}(\mathbf{x})| + \sup_{\mathbf{x}} |f_{S,\mathcal{P}_{r}^{EP},\gamma}(\mathbf{x})-f_{S,\mathcal{P}_{r}^{EP},\hat{\gamma}}(\mathbf{x})| \\
    &\leq \Vert f_{0,S}\Vert_{\mathcal{H}^{\alpha}}\bigg({C_{S}\over r}\bigg)^{\alpha} + (24F+1)2^{|S|} \sqrt{r^{|S|}\log n\over n}.
\end{align*}
with probability at least
$1-2^{|S|+1}|S|r^{|S|^{2}}\exp\left(-{r^{|S|}\log n \over 4F^{2}}\right)$ with resepect to $\mathbb{P}^{n}_{\mathbf{X}}$, where $C_{S}=\sqrt{|S|} / p_{L}$.

\subsection{Decomposition of the modified EP-product tree by a sum of
identifiable binary-product trees with bounded heights}
\label{app:approximation_emp_decompose}

The final step is to decompose the modified EP-product tree which satisfies the identifiability condition into the sum of identifiable binary-product trees with bounded heights.

\subsubsection{Notations}

Let $\mathcal{P}_j$ be an interval partition of $\mathcal{X}_j$ with $|\mathcal{P}_j|=\phi_j.$
That is, $\mathcal{P}_j=\{I_{j,k}, k=1,\ldots,\phi_j\},$ where
$I_{j,k}$s are disjoint intervals of $\mathcal{X}_j$ with $\bigcup_k I_{j,k}=\mathcal{X}_j.$
Without loss of generality, we assume that the intervals are ordered such that $I_{j,k}<I_{j,k'}$ whenever $k<k'$.
Here $I_{j,k}<I_{j,k'}$ means that for any $z\in I_{j,k}$ and $z'\in I_{j,k'}$, we have $z<z'$. 
That is, the indices of the intervals in each interval partition are sorted from left to right.

For $S \subseteq [p]$, let $f$ be an identifiable multinary-product tree defined on the partitions $\mathcal{P}_j, j\in S$ given as 
$$f(\mathbf{x}_{S})=\sum_{\pmb{\ell}} \gamma_{\pmb{\ell}} \prod_{j\in S} 
\mathbb{I}(x_j\in I_{j,\ell_j}),$$ where $\gamma_{\pmb{\ell}}$ is the height vector for $\pmb{\ell}\in \prod_{j\in S} [\phi_j].$
We introduce several notations related to $f.$

\begin{itemize}
    \item ${\rm part}(f)_j=\mathcal{P}_j$: interval partition of $\mathcal{X}_j$.
    %\item ${\rm part}(f)_{jk}=I_{j,k}:$ the $k$th interval in ${\rm part}(f)_j.$
    \item ${\rm order}(f)_j= |{\rm part}(f)_j| (=\phi_j),$ ${\rm order}(f)=(\phi_j, j\in S).$
    \item ${\rm index}(f)=\prod_{j\in S} [\phi_j]$: the set of indices for the product partition
    $\prod_{j\in S} \mathcal{P}_j.$
     \item $\pmb{\gamma} = (\gamma_{\pmb{\ell}} , \pmb{\ell} \in \text{index}(f))$.
    \item For given $\pmb{\ell}\in {\rm index}(f)$ and $\ell\in [\phi_j],$
    let $\pmb{\ell}_{+(j,\ell)}$ be the element in ${\rm index}(f)$  obtained by replacing $\ell$
    in the $j$th position of $\pmb{\ell}.$ That is
    $\pmb{\ell}_{+(j,\ell)}=(\ell_1,\ldots,\ell_{j-1}, \ell, \ell_{j+1},\ldots, \ell_{|S|}).$
\end{itemize}

\begin{figure}[h]
    \centering
    \includegraphics[width=0.9\linewidth]{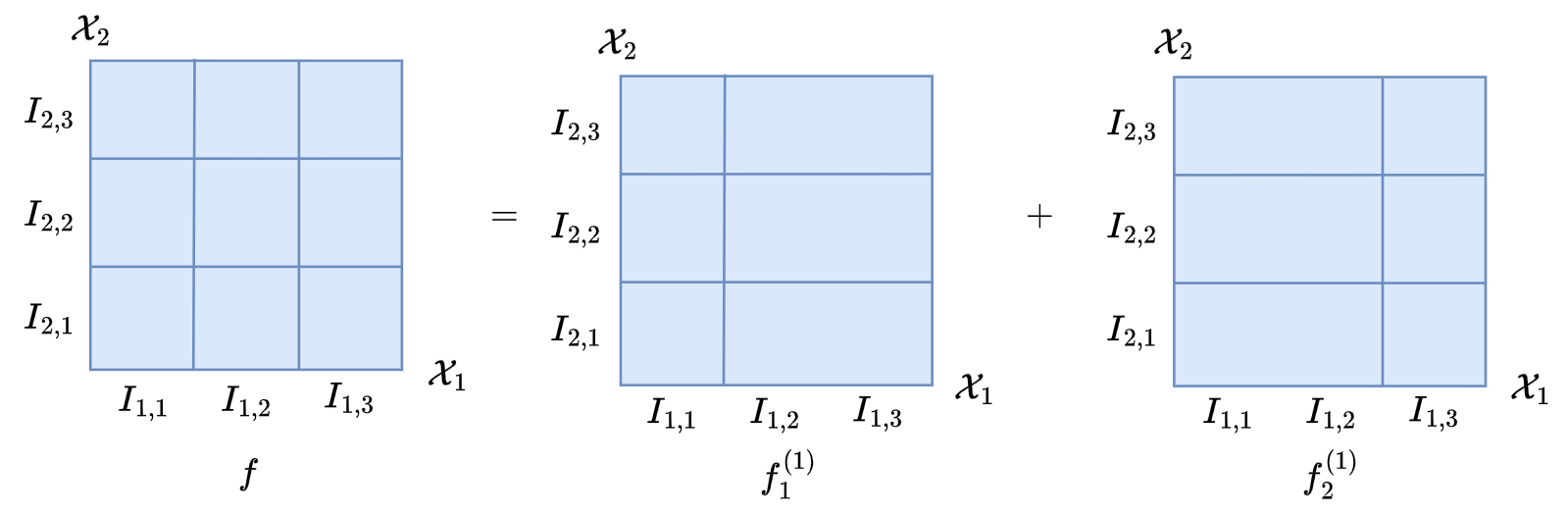}
    \caption{Example of decomposing the partition of $f$ into the partitions of $f_{1}^{(1)}$ and $f_{2}^{(1)}$.}
    \label{fig:decomposing}
\end{figure}

\subsubsection{Proof of the decomposition}

The strategy of the decomposition is to decompose a given identifiable multinary-product tree into the sum of two sibling 
identifiable multinary-product trees whose orders are smaller than their parent identifiable multinary-product tree.
Figure \ref{fig:decomposing} presents an example of decomposing the partition of an identifiable multinary-product tree $f$ for $S=\{1,2\}$ and $\phi_{1}=\phi_{2}=3$ into two partitions of child identifiable multinary-product trees $f_{1}^{(1)}$ and $f_{2}^{(1)}$, where $\text{order}(f_{1})_{1}=2$, $\text{order}(f_{1})_{2}=3$, $\text{order}(f_{2})_{1}=2$, and $\text{order}(f_{2})_{2}=3$.  
We repeat this decomposition until the orders of all identifiable multinary-product trees become 2.
Figure \ref{fig:decomposing_binary} presents an example of decomposing the partition of a identifiable multinary-product tree $f$ into four partitions of identifiable binary-product trees $f_{1,1}^{(1,2)},f_{1,2}^{(1,2)},f_{2,1}^{(1,2)}$, and $f_{2,2}^{(1,2)}$, where $\text{order}(f_{i,j}^{(1,2)})_{k}=2$ for $i=1,2$, $k=1,2$ and $j=1,2$. 
The following lemma is the key tool whose proof is given in Section \ref{proof:decompose_1} and \ref{proof:decompose_S}.

\begin{lemma} \label{le:decom}
Let $f$ be an identifiable multinary-product tree of $\bold{x}_S$.
Suppose that there exists $h\in S$ such that $\text{order}(f)_{h}>2.$
Then, there exist $\text{order}(f)_{h}-1$ many identifiable 
multinary-product trees $f_j, j=1,\ldots,\phi_h-1$ such that
$$f(\cdot)= f_{1}(\cdot)+\cdots+ f_{\phi_h-1}(\cdot),$$
where ${\rm order}(f_{j})_k={\rm order}(f)_k$ for all $j \in [\phi_h-1]$ when $k\ne h$
and ${\rm order}(f_{j})_h=2$ for all $j\in [\phi_h-1].$
Moreover, $\sup_j \|f_j\|_\infty \le 2 \|f\|_\infty.$
\end{lemma}
\medskip

\begin{figure}[t]
    \centering
    \includegraphics[width=0.9\linewidth]{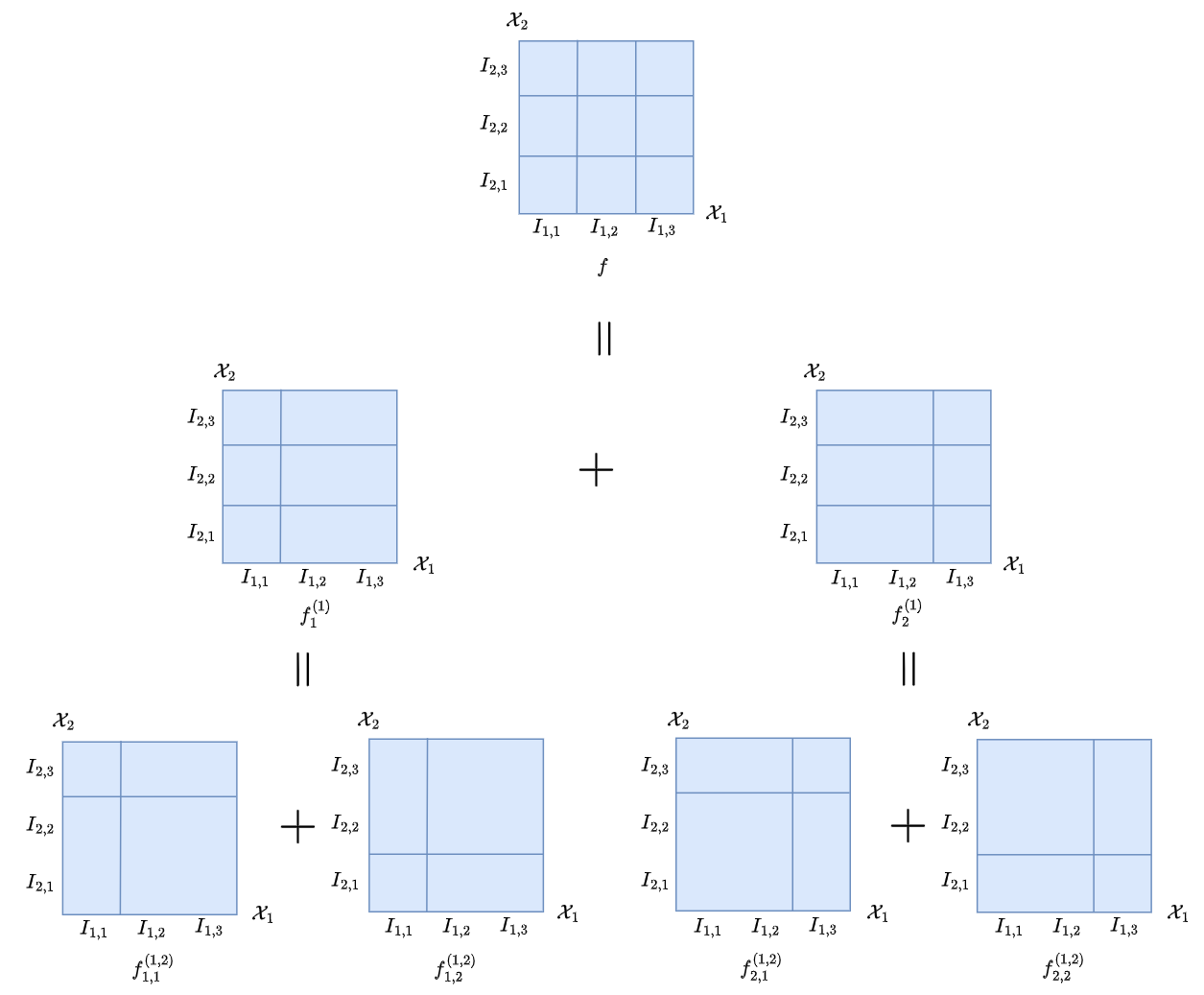}
    \caption{Example of decomposing partition of $f$ into the partitions of identifiable binary-product trees $f_{1,1}^{(1,2)}$, $f_{1,2}^{(1,2)}$, $f_{2,1}^{(1,2)}$, and $f_{2,2}^{(1,2)}$.}
    \label{fig:decomposing_binary}
\end{figure}

We will prove the decomposition by use of Lemma \ref{le:decom}.
Let $f$ be a given EP-product tree of the component $S$
with ${\rm order}(f)_j=r$ for all $j\in S,$  which we are going to decompose.
Without loss of generality, we let $S=\{1,2,\ldots,|S|\}.$
At first, we apply Lemma \ref{le:decom} to decompose
$f(\cdot)=\sum_{k_1=1}^{r-1} f_{k_1}^{(1)}(\cdot),$ where $f_{k_1}^{(1)}$s are identifiable multinary-product trees
such that $\text{order}(f_{k_1}^{(1)})_1 =2$ 
and $\text{order}(f_{k_1}^{(1)})_\ell = r, \ell\ge 2$ for all $k_1=1,\ldots,r-1$ 
and $\sup_{k_1} \|f_{k_1}^{(1)}\|_\infty \le 2 \|f\|_\infty.$ 

In turn, we can decompose each $f_{k_1}^{(1)}$ by a sum of identifiable multinary-product trees such that
$f_{k_1}^{(1)}(\cdot)=\sum_{k_2=1}^{m-1} f_{k_1,k_2}^{(1,2)}(\cdot),$
where $\text{order}(f_{k_1,k_2}^{(1,2)})_\ell=2, \ell=1,2$ and $\text{order}(f_{k_1,k_2}^{(1,2)})_\ell=r, l\ge 3$
and $\sup_{k_1,k_2} \|f_{k_1,k_2}^{(1,2)}\|_\infty \le 2^2 \|f\|_\infty.$ 
We repeat this decomposition to have $f_{k_1,\ldots,k_{|S|}}^{(1,\ldots,|S|)}(\cdot)$ such that
$$
f(\cdot)=\sum_{k_1}^{r-1} \cdots \sum_{k_{|S|}=1}^{r-1} f_{k_1,\ldots,k_{|S|}}^{(1,\ldots,|S|)} (\cdot) \quad \text{with} \quad
\text{order}(f_{k_1,\ldots,k_{|S|}}^{(1,\ldots,|S|)})_\ell=2$$
for all $\ell\in S$ and $\sup_{k_1,\ldots, k_{|S|}} \|f_{k_1,\ldots,k_{|S|}}^{(1,\ldots,|S|)}\|_\infty \le 2^{|S|} \|f\|_\infty,$
which completes the proof.

\qed

\subsubsection{Proof of Lemma \ref{le:decom} for $|S|=1$}

\label{proof:decompose_1}

For $j\in [r],$ let $\bold{1}_j$ be the $r$-dimensional vector such that
the first $j$ many entries are 1 and the others are 0. That is, $\mathbf{1}_j = ( \underbrace{1, 1, \dots, 1}_{j \text{ times}}, \underbrace{0, 0, \dots, 0}_{r-j \text{ times}} )^\top$.
Let $\bold{1}_j^c=\bold{1}_r-\bold{1}_j,$ that is, $\mathbf{1}_{j}^{c} = ( \underbrace{0, 0, \dots, 0}_{j \text{ times}}, \underbrace{1, 1, \dots, 1}_{r-j \text{ times}} )^\top$.

Let $\mathcal{S}^r=\{\bold{w}\in \mathbb{R}^r: w_j \ge 0 ,\forall j\in [r], \sum_{j=1}^r w_j=1\}.$
For a given $\bold{w}\in \mathcal{S}^r,$ let $\mathbb{R}^r_{\bold{w}}=\{\textbf{u} \in \mathbb{R}^r: \bold{w}^\top \textbf{u}=0\}.$
In addition, let $\bold{w}_{s:t}=\sum_{j=s}^t w_j.$
\medskip

\noindent{{\bf Claim:}} Fix $\bold{w}\in \mathcal{S}^r$.
For any $\textbf{u}\in \mathbb{R}_{\bold{w}}^r,$
there exist pairs of real numbers $(a_j,b_j), j=1,\ldots,r-1$ such that
\begin{align}
&\textbf{u}=\sum_{j=1}^{r-1}(a_j \bold{1}_j + b_j \bold{1}_j^c), \label{eq:decom_one_dim} \\ 
&\bold{w}^\top (a_j \bold{1}_j + b_j \bold{1}_j^c) =0
\label{eq:wi}
\end{align}
for all $j\in [r-1]$ and 
$$\max \{|a_j|,|b_j| : j\in [r-1]\} \le 2 \|\textbf{u}\|_\infty.$$
\medskip

\begin{figure}[h]
    \centering
    \includegraphics[width=0.8\linewidth]{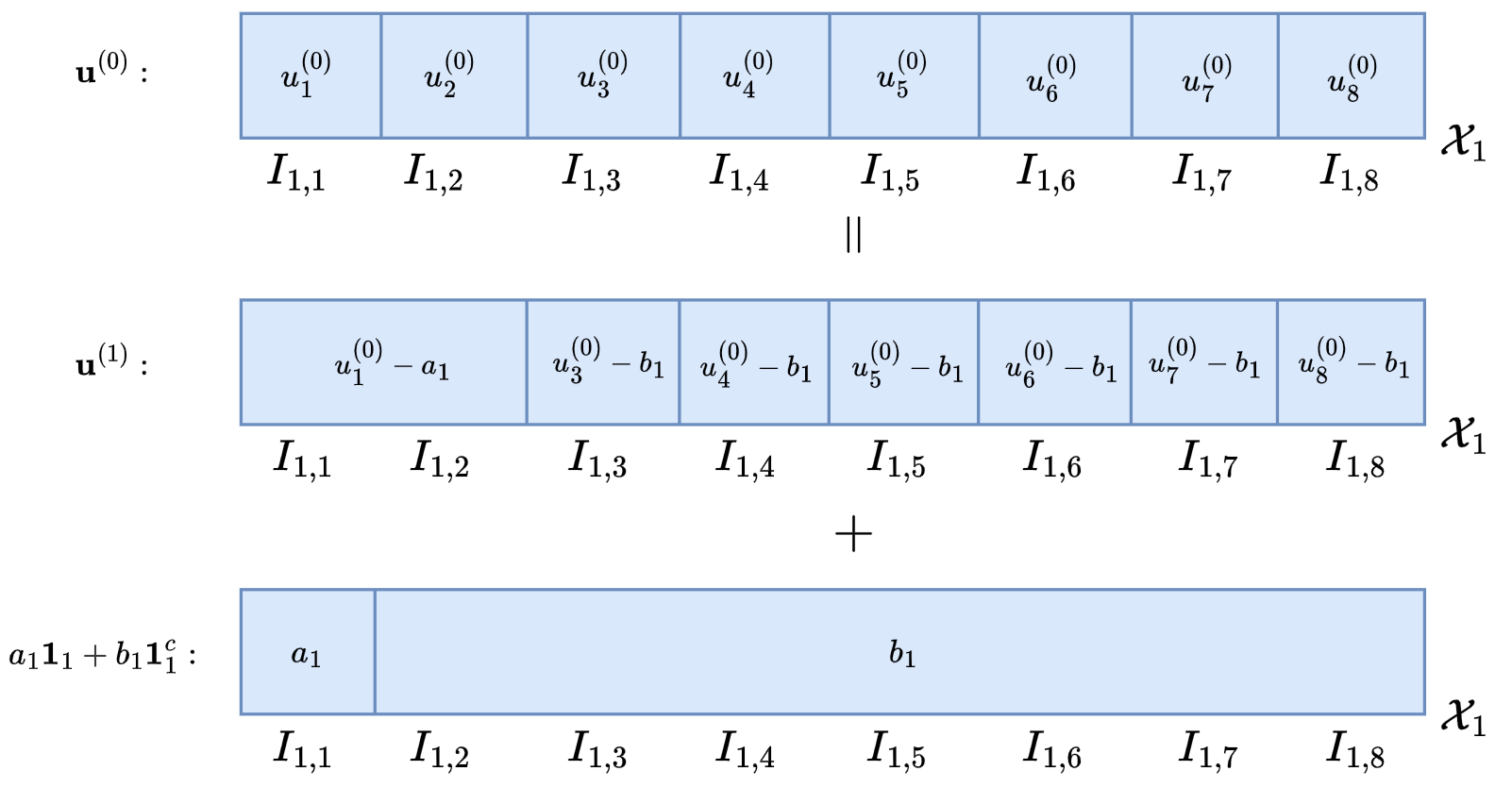}
    \caption{Decomposition of $\mathbf{u}^{(0)}$ into $\mathbf{u}^{(1)}$ and $(a_{1} \bold{1}_{1} + b_{1} \bold{1}_{1}^c)$ in the case of $r=8$.}
    \label{fig_decompse_example2}
\end{figure}

\begin{proof}
Let $\textbf{u}^{(0)}=\textbf{u}$ and we define $\textbf{u}^{(1)}$ as
$$\textbf{u}^{(1)}=\textbf{u}^{(0)}- (a_{1} \bold{1}_{1} + b_{1} \bold{1}_{1}^c),$$
where $a_{1}=(1-w_1)(u_1^{(0)}-u_{2}^{(0)})$ and $b_{1}=w_1 (u_{2}^{(0)}-u_1^{(0)}).$ 
It is easy to see that $\bold{w}^\top (a_{1} \bold{1}_{1} + b_{1} \bold{1}_{1}^c)=0$ and
thus $\bold{w}^\top \textbf{u}^{(1)}=0$. 
Moreover, we have
$u_{j}^{(1)}=u_{j}^{(0)}-b_{1}$ for $j=3,\ldots,r$ and $u_1^{(1)}=u_1^{(0)}-a_{1}.$ 
In addition, it can be shown that $u_1^{(1)}=u_2^{(1)}.$
Figure \ref{fig_decompse_example2} presents an example of this decomposition in the case of $r=8$.
$\newline$
$\newline$
For $k \in \{2,\ldots,r-2\},$ we define $\textbf{u}^{(k)}$ recursively as
\begin{equation}
\label{eq:decom1}
\textbf{u}^{(k)}:=\textbf{u}^{(k-1)}- (a_{k} \bold{1}_{k} + b_{k} \bold{1}_{k}^c),
\end{equation}
where  $a_{k}=(1-\bold{w}_{1:k})(u_{k}^{(k-1)}-u_{k+1}^{(k-1)})$ and
$b_{k}=\bold{w}_{1:k}(u_{k+1}^{(k-1)}-u_{k}^{(k-1)}).$ 
It is not difficult to see that $\bold{w}^\top (a_{k} \bold{1}_{k} + b_{k} \bold{1}_{k}^c) = 0$
(and thus $\bold{w}^\top \textbf{u}^{(k)}=0$) because $\bold{w}^\top \textbf{u}^{(k-1)}=0.$
In addition, we have $u_{1}^{(k-1)}=\cdots= u_{k}^{(k-1)}$ for $k \in [r-2]$. 
$\newline$
$\newline$
To complete the proof, we will show that $\textbf{u}^{(r-2)}$ can be expressed as 
$$
\textbf{u}^{(r-2)}=a_{r-1} \bold{1}_{r-1} + b_{r-1} \bold{1}_{r-1}^c,
$$
where $a_{r-1} = w_{r}(u_{r-1}^{(r-3)} - u_{r}^{(r-3)})$ and $b_{r-1} = (1-w_{r})(u_{r}^{(r-3)} - u_{r-1}^{(r-3)}).$
Note that, since $u_{1}^{(r-3)}=\cdots= u_{r-2}^{(r-3)}$ and $\mathbf{w}^{\top}\mathbf{u}^{(r-3)}=0$, it holds that
\begin{align}
\mathbf{w}_{1:r-2}u_{s}^{(r-3)} + w_{r-1}u_{r-1}^{(r-3)} + w_{r}u_{r}^{(r-3)} = 0 \label{eq:one_dim_decompose_condition}
\end{align}
for $s \in [r-2]$.
For $s \in [r-2]$, we have
\begin{align}
u_{s}^{(r-2)} &= u_{s}^{(r-3)} - a_{r-2} \nonumber\\
&= u_{s}^{(r-3)} - (w_{r-1}+w_{r})(u_{r-2}^{(r-3)}-u_{r-1}^{(r-3)}) \nonumber\\
&= w_{r}(u_{r-1}^{(r-3)} - u_{r}^{(r-3)}). \label{eq:one_dim_decomp_transformation1}
\end{align}
In turn, it follows that
\begin{align}
u_{r-1}^{(r-2)} &= u_{r-1}^{(r-3)} - b_{r-2} \nonumber\\
&= u_{r-1}^{(r-3)} - \mathbf{w}_{1:r-2}(u_{r-1}^{(r-3)} - u_{r-2}^{(r-3)} ) \nonumber\\
&= u_{r-1}^{(r-3)} - \mathbf{w}_{1:r-2}u_{r-1}^{(r-3)} - w_{r-1}u_{r-1}^{(r-3)} - w_{r}u_{r}^{(r-3)} \label{eq:one_dim_decomp_transformation2}\\
&= w_{r}( u_{r-1}^{(r-3)} -  u_{r}^{(r-3)}). \nonumber
\end{align}
and
\begin{align}
u_{r}^{(r-2)} &= u_{r}^{(r-3)} - b_{r-2} \nonumber\\
&= u_{r}^{(r-3)} - \mathbf{w}_{1:r-2}(u_{r-1}^{(r-3)} - u_{r-2}^{(r-3)} ) \nonumber\\
&= u_{r}^{(r-3)} - \mathbf{w}_{1:r-2}u_{r-1}^{(r-3)} - w_{r-1}u_{r-1}^{(r-3)} - w_{r}u_{r}^{(r-3)}  \label{eq:one_dim_decomp_transformation3}\\
&= (1-w_{r})(u_{r}^{(r-3)}-u_{r-1}^{(r-3)}). \nonumber
\end{align}
Here, (\ref{eq:one_dim_decomp_transformation1}), (\ref{eq:one_dim_decomp_transformation2}), and (\ref{eq:one_dim_decomp_transformation3}) are due to $u_{1}^{(r-3)}=\cdots= u_{r-2}^{(r-3)}$ and (\ref{eq:one_dim_decompose_condition}).
% To sum up, we can write
% $\textbf{u}^{(r-2)}=a_{r-1} \bold{1}_{r-1} + b_{r-1} \bold{1}_{r-1}^c,$ where $a_{r-1} = w_{r}(u_{r-1}^{(r-3)} - u_{r}^{(r-3)})$ and $b_{r-1} = (1-w_{r})(u_{r}^{(r-3)} - u_{r-1}^{(r-3)}).$
Thus, (\ref{eq:decom1}) holds with $\textbf{u}^{(r-1)}={\bf 0}.$
Hence, we complete the proof of the existence of the real numbers $(a_{j},b_{j}),j=1,...,r-1$, satisfies (\ref{eq:decom_one_dim}).

Now, we will prove the upper bound of $\max \{|a_j|,|b_j|: j\in [r-1]\}$.
For $k \in [r-2]$, we have
$u_{k+1}^{(k-1)}-u_k^{(k-1)}=u_{k+1}^{(0)}-u_k^{(0)}$ and thus
$|a_k|\le 2\|\textbf{u}\|_\infty$ and $|b_k| \leq 2\|\textbf{u}\|_{\infty}$ which completes the proof of \textbf{Claim}.
\end{proof}

Let $f$ be an identifiable multinary-product tree $f$ for $S=\{h\}$, i.e, $f(x_{h})=\sum_{j=1}^r \gamma_{j} \mathbb{I}(x_{h}\in I_{h,j}),$
where $I_{h,j}$s are an interval partition of $\mathcal{X}_{h}.$
We let $\mathbf{u}=(\gamma_{j}, j\in [r])$ and $\mathbf{w}=(\mu_{n,h}\{I_{h,j}\}, j\in [r])$ and apply \textbf{Claim} to have
$f(x_{h})=\sum_{j=1}^{r-1} f_{j}^{(h)}(x_{h}),$ where
each $f_{j}^{(h)}$ is an identifiable binary-product tree with the binary partition $\{ \cup_{\ell=1}^{j+1} I_{h,\ell}, \cup_{\ell=j+1}^r I_{h,\ell}\}$
and the height vector $(a_j,b_j).$

\subsubsection{Proof of Lemma \ref{le:decom} for general $S$}

\label{proof:decompose_S}
Fix $\pmb{\ell}\in {\rm index}(f)$ and $h \in S$.
Let $\textbf{u}_{\pmb{\ell},h}=(\gamma_{\pmb{\ell}_{+(h,j)}}, j \in [r] ).$
Applying \textbf{Claim} in Section \ref{proof:decompose_1} of Supplementary Material to $\textbf{u}_{\pmb{\ell},h}$ with $\bold{w}=(\mu_{n,h}\{I_{h,j}\}, j\in [r])$,
we have
$$\textbf{u}_{\pmb{\ell},h} = \sum_{j=1}^{r-1} (a_{j}^{(\pmb{\ell},h)} \bold{1}_j + 
b_{j}^{(\pmb{\ell},h)} \bold{1}_j^c).$$
Now, for $j \in [r-1]$, we define $f_{j}^{(h)}$ as the multinary-product tree with the interval partitions $\mathcal{P}_k^{(j)}, k\in S$ and the height vector $\pmb{\gamma}^{(j)}$
such that $\mathcal{P}^{(j)}_k=\mathcal{P}_k$ for $k\ne h$, $\mathcal{P}^{(j)}_h=\{ \cup_{m=1}^j I_{h,m}, \cup_{m=j+1}^{r} I_{h,m}\}$, $\gamma^{(j)}_{\pmb{\ell}_{+(h,1)}}=a_{j}^{(\pmb{\ell},h)}$, and $\gamma^{(j)}_{\pmb{\ell}_{+(h,2)}}=b_{j}^{(\pmb{\ell},h)}$ for $\pmb{\ell} \in \text{index}(f)$.
Since $|a_{j}^{(\pmb{\ell},h)}|\vee |b_{j}^{(\pmb{\ell},h)}| \le 2 \|\textbf{u}_{\pmb{\ell},h}\|_\infty \le 2 \|f\|_\infty,$ 
we have $\|f_{j}^{(h)}\|_\infty \le 2 \|f\|_\infty$
for $j=[r-1]$.

The final mission is to show that $f_{j}^{(h)}$s satisfy the identifiability condition, i.e.,
\begin{align*}
\int_{\mathcal{X}_k} f_{j}^{(h)}(\bold{x}_{S}) \mu_{n,k}(dx_{k}) = 0
\end{align*}
for $k \in S$.
First, $\int_{\mathcal{X}_k} f_{j}^{(h)}(\textbf{x}_{S}) \mu_{n,k}(dx_k)=0$ for $k=h$ by (\ref{eq:wi}).
For $k\ne h,$ the proof of \textbf{Claim} reveals that
there exists $r$-dimensional vectors $\mathbf{v}_1^{(j)}$ and $\mathbf{v}_2^{(j)}$  such that
$\gamma^{(j)}_{\pmb{\ell}_{+(h,m)}}=\mathbf{v}_m^{(j)\top} \gamma_{\pmb{\ell}_{+(h,\cdot)}},$
for $m=1,2,$ where $\gamma_{\pmb{\ell}_{+(h,\cdot)}}=(\gamma_{\pmb{\ell}_{+(h,j)}}, j\in [r]).$
Thus, for $m=1,2$, we have
\begin{align}
\mathbb{E}_{\mu_{n},k}(\gamma_{\pmb{\ell}_{+(h,m)}}^{(j)})=
\bold{v}_m^{(j)\top} \mathbb{E}_{\mu_{n},k}(\gamma_{\pmb{\ell}_{+(h,\cdot)}})=0    
\end{align}
since $\pmb{\gamma}$ satisfies the identifiability condition, where 
$$
\mathbb{E}_{\mu_{n},k}(\gamma_{\pmb{\ell}_{+(h,\cdot)}}) = \big(\mathbb{E}_{\mu_{n},k}(\gamma_{\pmb{\ell}_{+(h,1)}}),...,\mathbb{E}_{\mu_{n},k}(\gamma_{\pmb{\ell}_{+(h,r)}})\big)^{\top}$$
and $\mathbb{E}_{\mu_{n},k}(\cdot)$ is defined in (\ref{eq:marignal_expect_def}). 

In conclusion, the function $f$ can be decomposed as
\begin{align*}
f(\cdot) = \sum_{j=1}^{r-1}f_{j}^{(h)}(\cdot),
\end{align*}
where $f_{j}^{(h)}$ satisfies the identifiability condition and $\Vert f_{j}^{(h)} \Vert_{\infty} \leq 2\Vert f \Vert_{\infty}$ for $j=[r-1]$.
By applying \textbf{Claim} sequentially to $f_{j}^{(h)}$s for other variables $S\backslash \{h\}$, the proof is done. 

\qed

\newpage

\section{Discussion about the posterior concentration rate in the Gaussian regression model with unknown nuisance parameter $\sigma^{2}$}
\label{sec:discuss_unknown_sigma}
In this section, we discuss the posterior concentration rate of ANOVA-BART when the data distribution follows the Gaussian regression model with unknown variance $\sigma^{2}$.
That is, we consider the nonparametric regression model defined as
\begin{align}
Y = f(\mathbf{x}) + \epsilon,
\end{align}
where $\epsilon \sim N(0,\sigma^{2})$ and $f$ is a regression function.
Among the regularity Conditions \ref{eq:Assumption_1}-\ref{eq:Assumption_4}, we replace
Condition \ref{eq:Assumption_5} by
 Condition \ref{eq:assumption_gaussian} in place of Condition \ref{eq:Assumption_5} given as
\begin{enumerate}[label=(K.\hspace{-0.3em} \arabic*)]
    \item There exist positive constants $\sigma_{\text{min}}^{2}$ and $\sigma_{\text{max}}^{2}$ such that $\sigma_{\text{min}}^{2} < \sigma^{2} < \sigma_{\text{max}}^{2}$. \label{eq:assumption_gaussian}
\end{enumerate}
For $\xi > \max\{2^{p}F,\sigma_{\text{min}}^{-2},\sigma_{\text{max}}^{2}\}$, we let $\pi_{\xi}\{\cdot\} \propto \pi\{ \}\mathbb{I}(\Vert f \Vert_{\infty}\leq \xi, 1/\xi \leq \sigma^{2} \leq \xi)$.
Using Lemma 1 in \cite{lim2023synergizing} and Theorem 4 of \cite{nonp1}, to show the posterior concentration rate of ANOVA-BART in the Gaussian regression model, it suffices to verify the three conditions in (\ref{eq1-1}), (\ref{eq2-1}) and (\ref{eq3-1}), as in the proof of Theorem 2 in \cite{artbart}.
Since the only part that differs from the existing proof for exponential families
is Condition~(\ref{eq2-1}), we provide a brief explanation of the proof of
Condition~(\ref{eq2-1}) only.
$\newline$
Specifically, for Kullback-Leibler ball $\mathbb{B}_{n}$ with parameter $\theta = (f,\sigma^{2})$ and $\theta_{0} = (f_{0},\sigma_{0}^{2})$, direct calculation yields,
\begin{align}
\mathbb{B}_{n} \supseteq \{ \theta \in \mathcal{F}_{\xi} \times (1/\xi,\xi) : \Vert f - f_{0} \Vert_{2,n} \leq C_{\mathbb{B}^{*}}\epsilon_{n}, |\sigma^{2}-\sigma_{0}^{2}| \leq C_{\mathbb{B}^{*}}\epsilon_{n} \} 
\end{align}
for some positive constant $C_{\mathbb{B}^{*}}$.
Since
\begin{align}
\pi_{\xi}\{\mathbb{B}_{n}\} &\geq \pi_{\xi}\{ \theta \in \mathcal{F}_{\xi} \times (1/\xi,\xi) : \Vert f - f_{0} \Vert_{2,n} \leq C_{\mathbb{B}^{*}}\epsilon_{n}, |\sigma^{2}-\sigma_{0}^{2}| \leq C_{\mathbb{B}^{*}}\epsilon_{n} \} \\
&\geq \pi\{f : \Vert f - f_{0} \Vert_{\infty} \leq C_{\mathbb{B}^{*}}\epsilon_{n} \}\pi \{\sigma^{2} : |\sigma^{2} - \sigma^{2}_{0}| \leq C_{\mathbb{B}^{*}}\epsilon_{n}\}
\end{align}
and
\[
\pi\bigl\{\sigma^{2} : \lvert \sigma^{2} - \sigma_{0}^{2} \rvert \le \epsilon_{n} \bigr\}
\gtrsim n^{-1},
\]
Condition~(\ref{eq2-1}) is satisfied for any
$\xi > \max\bigl\{ 2^{p}F,\ \sigma_{\min}^{-2},\ \sigma_{\max}^{2} \bigr\}$.

\newpage

\section{Discussion about the posterior concentration rate for the $X$-fixed Design}
\renewcommand{\theequation}{F.\arabic{equation}}
\label{sec:fixed-x}

A key issue of ANOVA-BART for the $X$-fixed design is how to define $f_{0,S}$.
For a given $f_{0}$, Theorem \ref{thm:fANOVA_decomp} yields a unique ANOVA decomposition with respect to $\mu_{n}$, whose interactions are identifiable.
Recall that the proof of the posterior convergence rate for the $X$-random design consists of the two main parts:
(1) for $\mathbf{x}^{(n)}\in A_n,$ where $A_n$ is defined in (\ref{eq:An}),
the posterior convergence rate is $\epsilon_n$
and (2) the probability of $A_n$ converges to 1.
Thus, for the $X$-fixed design, the mission is to figure out sufficient conditions 
for $\mathbf{x}^{(n)}\in A_n$ with $f_{0,S}$ defined in Theorem \ref{thm:fANOVA_decomp} with respect to $\mu_{n}$. 

In order that $\mathbf{x}^{(n)}\in A_n,$ there exists a sum of identifiable binary-product trees
that approximates $f_{0,S}$ closely for each $S \in \mathbb{S}.$
In Section \ref{app:approximation_eq_p}, we show that the EP-tree, which satisfies the populational identifiability condition, approximates $f_{0,S}$ closely for the $X$-random design.
We can modify this proof to show that the empirical EP-tree, which satisfies the identifability condition, approximates $f_{0,S}$ closely
for the $X$-fixed design when there exists a positive constant $C_{2}$ such that, for all $r\in [n],$
$$
\max_{A \in \mathcal{P}_{r}^{nEP}}\text{Diam}(A) \leq C_{2}/ r,
$$ 
where ${\rm Diam}(A) := \max_{\mathbf{v},\mathbf{w} \in A} \|\mathbf{v}-\mathbf{w}\|_{2}$ and 
$\mathcal{P}_{r}^{nEP}$ denotes the equal-probability partition induced by the quantiles of the empirical distribution.
Then, by letting the modified EP-tree defined in Section \ref{app:approximation_pop_EQ} to be equal to the empirical EP-tree,
all the results in Section \ref{app:approximation_pop_EQ} to Section \ref{app:approximation_emp_decompose} are satisfied.
Thus, we conclude that
the posterior convergence rate is $\epsilon_n$ as long as 
$\max_{A \in \mathcal{P}_{r}^{nEP}}\text{Diam}(A) \leq C_{2}/ r$
for all $r\in [n].$

\newpage

\section{Details of the experiments} \label{app:experiment}
In this section, we provide details of the experiments including hyperparameter selections.

\subsection{Data standardization}
The minimax scaling is applied to NAM, while the mean-variance standardization is used for ANOVA-BART,
BART, SSANOVA and MARS. 
In addition, the one-hot encoding is applied to 
categorical covariates.

\subsection{Hyperparameter selection}

\subsubsection{Experiments for prediction performance}

For each method, we select the hyperparameters among the candidates based on 5-fold cross-validation.
The candidate hyperparameters for each method are given as follows.

\begin{itemize}
  \item ANOVA-BART
    \begin{itemize}
      \item $\alpha_{\text{split}}$ and $\gamma_{\text{split}}$ in (\ref{split_prior})  are set to 0.95 and 2, respectively, following \cite{BART}.
      \vskip 0.2cm
      \item $v =\{1,3,5,7,9\}$
      \vskip 0.2cm
      \item For $\lambda,$ we reparameterize it to $q_\lambda,$ where
      $q_\lambda=\pi(\sigma^{2} \leq \hat{\sigma}^{2})$ and  $\hat{\sigma}^{2}$ is
      the variance of the residuals from a linear regression model estimated by the least square method.
      For the candidates of $q_\lambda,$ we set $q_\lambda = \{0.90, 0.95, 0.99\}.$ This approach is used in \cite{BART}.
      \vskip 0.2cm
        \item For \textsc{Madelon} data set, we set the burn-in iterations to 2,000 and the sampling iterations to 2,000.
        For all other data sets, we set the burn-in iterations to 1,000 and the sampling iterations to 1,000.
      \vskip 0.2cm
      \item We set the step size for height proposal distribution as $0.01$. 
      \vskip 0.2cm
          \item $T_{\max} = \{50, 100, 200, 300\}$
          \vskip 0.2cm
          \item $C_{*} = \{10^{-6}, 10^{-4}, 10^{-2}, 1\}$
          \vskip 0.2cm
           \item $\sigma_{\beta}^{2}= \{10^{-3}, 10^{-2}, 10^{-1}\} $  
           \item $M = \{1,5\}$
    \end{itemize}
\vskip 0.5cm
\item BART (\cite{BART})
\begin{itemize}
    \item We consider candidate hyperparameters similar to those explored in \cite{BART}.
    \item $\alpha_{\text{split}},\gamma_{\text{split}},v,\lambda$, the number of iterations: set to the same values as for ANOVA-BART
    \vskip 0.2cm
    \item $T = \{50,100,200 \}$
    \vskip 0.2cm
    \item $\sigma_{\beta}^{2}=\{9/T, 9 /(4T), 1/T, 9/(25T)\}$
\end{itemize}
\vskip 0.5cm
\item SSANOVA
\begin{itemize}
    \item The smoothing parameter $\phi = \{10^{-4}, 10^{-2}, 1, 10^{2}, 10^{4}\}$
    \vskip 0.2cm
    \item The number of knots $=\{10,30,50\}$
    \vskip 0.2cm
    \item The maximum order of interaction $=\{1,2\}$
    \vskip 0.2cm
    \item Further details on the hyperparameters of SSANOVA can be found in \cite{gu2014smoothing}.
\end{itemize}
\vskip 0.5cm
\item MARS
\begin{itemize}
    \item Maximum number of terms in the pruned model $= \{10, 30, 50, 100\}$
    \vskip 0.2cm
    \item The maximum order of interaction $=\{1,2,3\}$
    \vskip 0.2cm
    \item The smoothing parameter $=\{1, 2, 3, 4\}$
    \vskip 0.2cm
    \item Further details on the hyperparameters of MARS can be found in \cite{milborrow2017earth}.
\end{itemize}
\vskip 0.5cm
\item NAM
  \begin{itemize}
   \item The number of hidden layers $=3$ and the numbers of hidden nodes $=(64,32,16)$ (the architecture
    used in \cite{nam, nbm})
    \vskip 0.2cm
   \item  Adam optimizer with learning rate 1e-4 and weight decay 7.483e-9
   \vskip 0.2cm
   \item Batch size $=256$
   \vskip 0.2cm
   \item Epoch $= \{500 , 1000\}$
   \vskip 0.2cm
   \item The maximum order of interaction $=\{1,2\}.$
  \end{itemize}
\end{itemize}

\subsubsection{Experiments for component Selection}

For the experiment of component selection on synthetic data, the hyperparameters for ANOVA-BART are set to $T_{\max} = 300, v = 3, q_{\lambda}=0.95$, and $C_{*} = 10^{-2}$.

\subsubsection{Experiments for uncertainty quantification}

The candidates of the hyperparameters of ANOVA-BART, BART and MARS are the same as those used in the prediction performance experiments.

\subsubsection{Experiments for Stable interpretation}

The candidates of the hyperparameters of ANOVA-BART and NAM are the same as those used in the prediction performance experiments.

\bibliographystyle{plain}
\bibliography{references}

\end{document}